\documentclass[twoside,11pt]{article}

\usepackage{blindtext}

\usepackage[preprint]{jmlr2e}
\usepackage{amsmath}
\newtheorem{assumption}[theorem]{Assumption}
\usepackage{mathtools}

\newcommand{\br}{\mathbb{R}}

\newcommand{\bn}{\mathbb{N}}

\newcommand{\bs}{\mathbb{S}}
\newcommand{\bE}{\mathbb{E}}

\newcommand{\ct}{\mathcal{T}}
\newcommand{\cd}{\mathcal{D}}

\newcommand{\cx}{\mathcal{X}}

\newcommand{\ch}{\mathcal{H}}
\newcommand{\cp}{\mathcal{P}}

\newcommand{\lam}{\lambda}

\newcommand{\set}[1]{\left\{#1\right\}}

\newcommand{\cl}{\mathrm{cl}}

\newcommand{\sgn}{\mathrm{sgn}}

\newcommand{\supp}{\operatorname{supp}}
\newcommand{\curv}[1]{\left\{#1\right\}}

\newcommand{\unif}{\operatorname{Unif}}
\newcommand{\vol}{\operatorname{Vol}}

\newcommand{\mt}{m,\tau}
\newcommand{\bl}{\beta,\lambda}

\usepackage{xcolor}

\usepackage{lastpage}
\jmlrheading{volume}{2026}{1-\pageref{LastPage}}{month/year}{month/year}{21-0000}{Shuang Liang, Tom Jacobs and Guido Mont\'ufar}

\ShortHeadings{Implicit Bias of SGD in Multivariate ReLU Networks}{Liang, Jacobs, Mont\'ufar} 
\firstpageno{1}

\begin{document}

\title{Implicit Bias of SGD in Multivariate ReLU Networks: Effective Width Collapse}

\author{\name Shuang Liang \email liangshuang@g.ucla.edu \\
       \addr Department of Statistics \& Data Science\\
       University of California, Los Angeles  \\
       Los Angeles, CA 90095, USA
       \AND
       \name Tom Jacobs \email tom.jacobs@cispa.de \\
       \addr CISPA Helmholtz Center for Information Security 
       \AND 
       \name Guido Mont\'ufar \email montufar@math.ucla.edu \\
       \addr Departments of Mathematics and Statistics \& Data Science\\
       University of California, Los Angeles; and\\
       Max Planck Institute for Mathematics in the Sciences}

\editor{My editor}

\maketitle

\begin{abstract}%   <- trailing '%' for backward compatibility of .sty file
We study the implicit bias of noisy stochastic gradient descent in training wide two-layer ReLU networks for multivariate regression. In a mean-field regime, the training dynamics are approximated by a Wasserstein gradient flow that converges to a unique stationary measure. We characterize the structure of this stationary measure and the predictor it represents. We show that, despite the network being infinitely overparameterized, the learned predictor admits an effectively finite representation: the input weights and biases align along finitely many directions, leading to an effective width collapse. In particular, the solution function is continuous piecewise affine, with affine regions determined by the cells of a finite hyperplane arrangement. The number of learned directions, and hence hyperplanes, is bounded above by $2\cp-1$, where $\cp$ denotes the number of linear dichotomies realizable on the training inputs. We further establish a non-redundancy property of the learned representation by proving that each learned direction induces a unique ternary activation pattern on the training data. 
Consequently, the complexity of the learned predictor is governed by the combinatorial geometry of the training data. 
\end{abstract}


\section{Introduction}

Overparameterized neural networks generally admit infinitely many interpolating solutions for a given training data set. 
Yet optimization algorithms such as stochastic gradient descent (SGD) consistently select solutions that generalize remarkably well. 
Understanding the principles governing this selection mechanism is a central question in modern learning theory. 
One approach to this question is through the notion of \emph{implicit bias}, which posits that optimization algorithms not only minimize the training error but also systematically favor certain minimizers over others. 
For example, in linear models, gradient-based optimization is known to select minimum-norm solutions. 
More generally, a large body of work has characterized the implicit bias of gradient descent and SGD through variational principles involving margin maximization, reproducing kernel Hilbert space norms, total variation of curvatures, and other notions of complexity \citep[e.g.,][]{zhang2019type, chizat2020implicit, Lyu2020Gradient, ji2020directional}. 
Despite this substantial progress, our understanding of the implicit bias of nonlinear neural networks remains incomplete, particularly in feature-learning regimes.

A natural framework for studying implicit bias in feature-learning dynamics is the mean-field regime \citep{mei2018mean,chizat2018global,rotskoff2018trainability,sirignano2020mean}. 
In this regime, a wide two-layer network is represented by a probability measure over neurons, and the training dynamics are described by a Wasserstein gradient flow on the space of probability measures. 
Moreover, when stochastic noise is incorporated into the dynamics, the resulting evolution admits a unique stationary measure that characterizes the long-time behavior of noisy SGD. 
This raises the following fundamental question: 
\emph{What is the structure of the stationary measure and of the solution function represented by this stationary measure?}

\begin{figure}[t]
    \centering
    \includegraphics[width=1\linewidth]{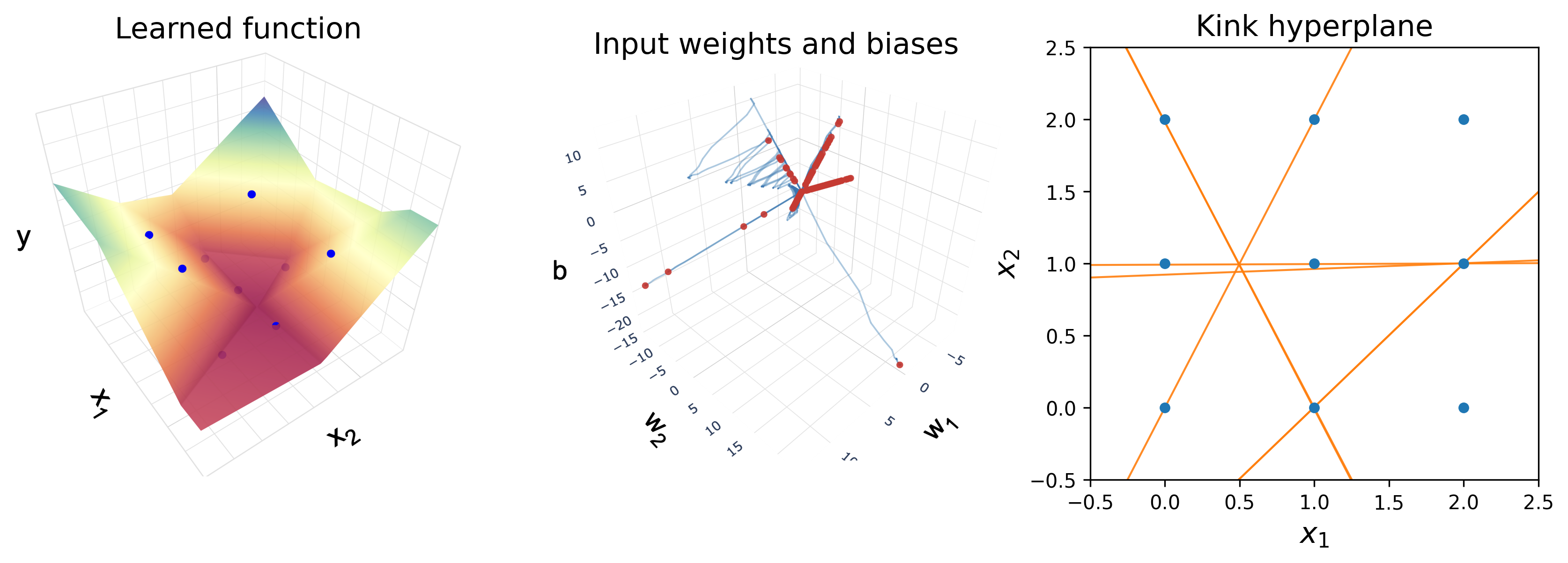}
    \caption{A two-layer ReLU network with $N=1000$ hidden neurons and two input dimensions, trained via SGD with weight decay. 
    Left: the learned function is piecewise affine and approximately fits the training data (blue dots). 
    Middle: the training trajectories of the input weights and biases 
    $(w^i_1,w^i_2,b^i)\in \br^3$, $i=1,\ldots, N$ align along a small number of directions, and the effective width of the network collapses. 
    Right: the kink hyperplanes of the trained function (orange lines) form a non-redundant arrangement in relation to the input data (blue dots).}
    \label{fig:intro}
\end{figure}

\paragraph{Main Contributions} 
We study multivariate regression with wide two-layer ReLU networks trained by noisy stochastic gradient descent with weight decay in the mean-field regime. 
Our main results, illustrated in Figure~\ref{fig:intro}, are as follows: 
\begin{itemize} 

\item 
In Theorem~\ref{thm:pwl}, we show that, as the noise level tends to zero, the solution selected by noisy SGD 
converges, irrespective of the strength of weight decay, to a continuous piecewise affine function. 
More precisely, this function is affine on each connected component of the complement of a finite hyperplane arrangement. 

\item 
We further show that the number of hyperplanes is bounded above by $2\mathcal P-1$, where $\mathcal P$ denotes the number of linear dichotomies realizable on the training inputs. 
Thus, the linear-region complexity of the learned solution is controlled by the combinatorial geometry of the training data set. 

\item 
Next, in Theorem~\ref{thm:feature}, we show that the input weights and biases of the trained network collapse onto at most $2\cp - 1$ distinct directions. Consequently, the trained network has effective width at most $2\cp - 1$. 

\item 
Finally, in Theorem~\ref{thm:nonredun}, we establish a strong non-redundancy property of these effective neurons. 
Specifically, distinct effective neurons induce distinct ternary activation patterns on the training inputs.
\end{itemize}

Our work builds on the analysis of \citet{shevchenko2022mean}, who characterized the solution functions selected by noisy SGD in the mean-field regime for univariate ReLU networks. 
We extend this analysis to the multivariate setting and, in addition, characterize the structure of the stationary measure, thereby describing the network parameters at convergence. 
In Section~\ref{sec:related-works} we place our results in the context of recent advances on mean-field dynamics, noisy SGD, implicit bias, and sparse or finite-support solutions. 
To our knowledge, our results provide the first derivation of effective width collapse as an implicit bias of noisy SGD for multivariate regression with general training data, together with a problem-dependent characterization of the learned solution.

\paragraph{Organization} 
The remainder of the article is organized as follows. 
Section~\ref{sec:related-works} gives an overview of related work. 
Section~\ref{sec:preliminaries} provides the necessary preliminaries, introducing the mean-field analysis of two-layer networks and discussing the implementation of ReLU activation functions within this framework. 
Section~\ref{sec:main-results} presents our main results, and Section~\ref{sec:discussion} discusses their interpretation, technical aspects, and limitations. 
Section~\ref{sec:proof} outlines the proof strategy and key components. 
Section~\ref{sec:experiments} presents numerical experiments illustrating the main results. 
Supporting results and technical proofs are collected in the appendices.

\section{Related Work}  
\label{sec:related-works}

%This section briefly reviews prior work on the implicit bias of wide neural networks, sparsity induced by optimization dynamics, and sparsity induced by explicit regularization. 

\paragraph{Implicit Bias for Wide Neural Networks} 
The implicit bias of gradient-based parameter optimization in wide neural network depends significantly on the training regime, in particular, on the scale of network parametrization and initialization relative to the network width 
\citep[see, e.g.,][]{JMLR:v22:20-1123,yang2021tensor}. 
A significant line of work has studied the implicit bias in the \emph{lazy} or \emph{kernel} regime \citep{jacot2018neural,du2018gradient, arora2019fine,lee2019wide}, where the network parameters remain close to their initialization throughout training and no feature learning takes place. 
In this regime, the implicit bias is characterized as minimization of a reproducing kernel Hilbert space norm \citep{zhang2019type}, 
whose explicit form has been investigated in series of works \citep{williams2019gradient, sahs2022shallow, heiss2023implicit,jin2023implicit,liang2025implicit}.

A distinct and important training regime is the mean-field regime, where the network parameters are modeled as an interacting particle system whose evolution is described by a nonlinear PDE over probability measures \citep{mei2018mean, chizat2018global, rotskoff2018trainability, sirignano2020mean}. 
This setting is of particular interest because it captures many nontrivial dynamical phenomena, including feature learning. 
However, compared to the kernel regime, much less is known about the convergence of training and the resulting implicit bias. 
The challenges in analyzing the mean-field dynamics are discussed by \citet{ma2020towards}. 
Nevertheless, several important advances have been made in characterizing the implicit bias in the mean-field regime. 
In a classification setting, \citet{chizat2020implicit} showed that if the network parameters converge in direction, the converged direction is a maximizer for an $\mathcal{F}_1$ norm maximum margin problem. 
\citet{pellegrini2020analytic} provided an analysis for hinge loss and linearly separable data with a spherically symmetric input distribution. Assuming infinitely many training samples, they showed that the ReLU network implements a linear classifier throughout training. 

The work most closely related to ours is that of \citet{shevchenko2022mean}, whose analysis provides the main foundation for the present work. 
They studied noisy SGD in a setting where the global convergence of the corresponding mean-field PDE to a stationary measure had previously been established by \citet{mei2018mean}. 
For univariate ReLU networks, they showed that, in the vanishing-noise limit, the stationary measure corresponds to a piecewise linear function with at most three knots between any two consecutive training inputs. 
A detailed comparison between their results and ours is provided in Section~\ref{sec:discussion}.

\paragraph{Explaining Sparsity from Training Dynamics} 
The phenomenon that training dynamics lead to sparse features has been widely observed empirically \citep{maennel2018gradient,JMLR:v22:20-1123, zhou2022empirical}. 
One attempt to explain this phenomenon is through \emph{early alignment} \citep{maennel2018gradient}, whereby, under infinitesimal initialization, the network parameters align along a small number of directions during the initial phase of training. 
Motivated by this phenomenon, a line of work has investigated the full training dynamics and 
established effective width collapse \citep{phuong2021the, lyu2021gradient, wang2022the, boursier2022gradient, chistikov2023learning, min2024early, boursier2025early}. 
Building on the connection between gradient flow and maximum-margin optimization in homogeneous networks for classification \citep{Lyu2020Gradient,ji2020directional}, 
\citet{frei2023implicit,kou2024implicit} demonstrated an implicit bias toward low-rank input weights. 
Moreover, \citet{safran2022effective} showed that, for univariate classification problems whose labels are generated by a teacher network of width $r$, the learned function has effective width $O(r)$. 
For Gaussian cluster data and linear classification loss, \citet{vasudeva2026rich} showed that, with infinitely many training samples, the trained ReLU network yields a linear decision boundary.

Besides \citet{shevchenko2022mean}, existing works establishing effective width collapse as a consequence of the training dynamics rely on restrictive assumptions on the training data. 
In classification, these include 
orthogonally separable \citep{phuong2021the,wang2022the},
symmetric separable \citep{lyu2021gradient}, 
nearly orthogonal \citep{frei2023implicit, kou2024implicit}, and positively correlated  \citep{min2024early} training data. 
In multivariate regression, 
\citet{boursier2022gradient} assume that the training inputs are mutually orthogonal, 
\citet{chistikov2023learning} consider labels generated by a single neuron, and 
\citet{boursier2025early} analyze a training set consisting of three data points. 
In addition, these works typically require the parameter initialization to have sufficiently small scale or to satisfy a balancedness condition. 
% %
By contrast, we establish effective width collapse under minimal assumptions on the training data and parameter initialization. 
The tradeoff is that our analysis relies on the network being sufficiently wide, allowing us to exploit the mean-field framework.

\paragraph{Explaining Sparsity from Optimal Solutions}
Our results complement prior work on sparse or finite-support solutions of static optimization problems that capture the learning objective. 
The setting closest to ours is that of shallow ReLU networks with unbounded width and Euclidean norm regularization of the network parameters. 
By the homogeneity of the ReLU activation, this is equivalent to regularizing the $\mathcal{F}_1$ functional norm \citep[see, e.g.,][]{bach2017breaking}. 
For univariate ReLU networks with a free skip connection, \citet{boursier2023penalising} showed that, the minimum-$\mathcal{F}_1$-norm interpolator is unique and can be expressed as the sum of an affine function and a ReLU network of width at most $M-1$, where $M$ is the number of training samples. 
For multivariate data, \citet{de2020sparsity, del2026dual} 
showed that every minimum-$\mathcal{F}_1$-norm interpolator admits a representation as a ReLU network with effective width bounded by a constant $K$ that depends on the input dimension and sample size. 
By contrast, we derive an effective-width bound for the solution selected by noisy SGD. 
Our bound is strictly sharper than those of \citet{de2020sparsity, del2026dual}; see a detailed comparison in Section~\ref{sec:discussion}.

A different regularizer, which penalizes only the Euclidean or $\ell_p$ norm of the network weights but not the biases, has been studied both for unbounded-width networks \citep{savarese2019infinite, ongie2019function, parhi2021banach,debarre2022sparsest} and finite-width networks \citep{JMLR:v22:20-1447, nakhleh2026global}. 
In both settings, it has been shown that there exist optimal solutions that can be represented as ReLU networks with width at most $K$, where $K$ is a data-dependent constant. 

In classification, another widely studied optimization problem is the maximum-margin problem. 
This line of work is motivated by the observation of \citet{Lyu2020Gradient, ji2020directional} that, under suitable conditions, training homogeneous neural networks converges in direction to a KKT point of a maximum-margin problem. 
Subsequent work has investigated the structure of the maximum-margin solutions and their connection to  sparse feature representations \citep{frei2023benign,timor2023implicit, tsilivis2025flavors, jacobs2026implicitbiasmirrorflow, jacobs2026never}. 

While these works provide a detailed understanding of the solutions of the associated optimization problem, it remains largely open under what conditions a training algorithm converges to such solutions. 
Our results bridge this gap by showing that effective width collapse arises not only as a property of static variational problems, but also as a consequence of noisy SGD training dynamics.

\section{Preliminaries} 
\label{sec:preliminaries}

In this section, we recall preliminary results from the mean-field theory of wide neural networks and introduce a bounded, smoothed approximation of the ReLU neuron, which is needed to apply these mean-field results.

\subsection{Mean-Field Analysis for Wide Two-Layer Networks} 
\label{sec:mean field two layer wide network}

Consider fully connected networks with $d$ input units, $N$ hidden units, and one output unit, defined by 
\begin{equation}\label{eq:model}
    h_\sigma(x,\Theta)= \frac{1}{N}\sum_{i=1}^N\sigma(x,\theta_i) , 
\end{equation}
where $\sigma\colon \br^d \times \br^D\to \br$ is the unit (or neuron) activation, $\theta_i\in \br^D$ denotes the parameter vector of the $i$th hidden unit, and $\Theta=(\theta_i)_{i=1}^N \in \br^{ND}$ denotes the parameter of the entire network. 
To recover a ReLU network, we set $\sigma(x,\theta)=a (w^\top x+b)_+$, where $\theta=(a,w,b)\in \br\times \br^{d}\times\br = \br^{d+2}$ consists of the output weight $a$, the input weights $w$, and the bias $b$, and where $(u)_+=\max(u,0)$. 
Assume that the parameters of the hidden units are initialized as i.i.d.\ samples from a fixed measure $\rho_0\in \cp_2^a(\br^D)$. 
Here and throughout, $\cp_2^a(\br^D)$ denotes the space of absolutely continuous probability measures on $\br^D$ with finite second moments.

Consider a finite training data set $\set{(x_j,y_j)}_{j=1}^M\subset \br^d \times \br$. 
We consider the mean squared error with $\ell_2$-regularization: 
$$
L(\Theta)=\frac{1}{M}\sum_{j=1}^M (h_\sigma(x_j,\Theta)-y_j)^2 + \frac{\lambda}{N} \|\Theta\|^2. 
$$
To minimize $L$, we train the neural network using noisy stochastic gradient descent (noisy SGD), with parameter updates of the form 
\begin{equation}
\label{eq:SGD}
    \theta_{i}^{k+1} = \theta_i^k - 2 s_{k,N} \big(h_\sigma(\tilde{x}_k,\Theta^k)-\tilde{y}_k\big) \nabla_{\theta_i} \sigma(\tilde{x}_k, \theta_i^k) - 2s_{k,N} \lambda \theta_i^k + \sqrt{\frac{2s_{k,N}}{\beta}}g_i^k, 
\end{equation} 
where $s_{k,N}$ is the step size, $(\tilde{x}_k,\tilde{y}_k)_{k\in {\bn_0}}$ are i.i.d.\ samples drawn from the empirical distribution $\frac{1}{M}\sum_{j=1}^M\delta_{(x_j,y_j)}$ induced by the training data, and $(g_i^k)_{i\in \set{1,\ldots,N},k\in \bn}$ are i.i.d.\ Gaussian noise drawn from $N(0,I_D)$. The parameter $\beta^{-1}$ controls the noise level and is commonly referred to as the \emph{temperature}. In this work, we focus on the vanishing-temperature regime, $\beta^{-1} \to 0$, in which case \eqref{eq:SGD} recovers the standard SGD algorithm \citep{robbins1951stochastic}.

The network model~\eqref{eq:model} depends on the parameter $\Theta$ only through the empirical measure 
$\tfrac{1}{N}\sum_{i=1}^N\delta_{\theta_i}$. 
In particular, the ordering of the hidden units does not affect the network output. 
This motivates a reformulation of the model in terms of probability measures $\rho\in \cp(\br^D)$: 
$$
h_\sigma(x,\rho) = \int \sigma(x,\theta) \, d\rho.
$$ 
When $\rho = \tfrac{1}{N}\sum_{i=1}^N \delta_{\theta_i}$, this representation recovers the finite-width network~\eqref{eq:model}.

Let $\hat\rho_k^{(N)} = \tfrac{1}{N} \sum_{i=1}^N \delta_{\theta_i^k}$ denote the empirical distribution of the network parameters after $k$ steps of noisy SGD \eqref{eq:SGD}. 
Consider a step size scheme $s_{k,N} = \varepsilon \xi(k\varepsilon)$, 
where 
$\varepsilon=\varepsilon_N>0$ depends on the width $N$ 
and $\xi\colon [0,\infty) \to [0,\infty)$ is a sufficiently regular function, e.g., $\varepsilon_N=N^{-1}$ and $\xi\equiv 1$, so $s_{k,N}= N^{-1}$. 
It is known that, for any fixed training time $t\geq0$, 
the empirical parameter measure $\hat\rho_{\lfloor t/\varepsilon \rfloor}^{(N)}$ converges weakly to $\rho_t$ as $\varepsilon \to 0$ and $N\to\infty$, 
where $\rho_t$ is the solution of the following partial differential equation with initial condition $\rho_0$ (the parameter initialization distribution): 
\begin{equation}\label{eq:pde}
    \begin{aligned}
        \partial_t \rho_t &= 2\xi(t)\nabla_\theta \cdot \Big(\rho_t \nabla_\theta \Psi^\lambda_\sigma(\theta;\rho_t) \Big) + 2\xi(t) \beta^{-1} \Delta_\theta\rho_t\\
        \Psi^\lambda_\sigma(\theta;\rho) &= \frac{1}{M}\sum_{j=1}^M\Big(h_\sigma(x_j,\rho)-y_j\Big) \sigma(x_j,\theta) +\frac{\lambda}{2}\|\theta\|^2.
    \end{aligned}
\end{equation}
Here $\nabla_\theta \, \cdot$ denotes the divergence operator and $\Delta_\theta$ the Laplacian. 
The mean-field potential $\Psi^\lambda_\sigma(\theta;\rho)$ is interpreted as the loss experienced by an individual neuron when all other neurons are represented collectively by the distribution $\rho$. 
Thus, $\hat\rho_{\lfloor t/\varepsilon \rfloor}^{(N)}$ may be viewed as a discrete-time, finite-particle approximation of $\rho_t$. 
Note that in the step size scheme $s_{k,N}=\varepsilon \xi(k\varepsilon)$, $\varepsilon$ specifies the time discretization scale and $\xi$ controls the speed of evolution of the PDE~\eqref{eq:pde}.

Let $\mathcal{H}$ denote the set of admissible densities of probability measures in $\mathcal{P}_2^a(\mathbb{R}^D)$:
$$
\ch = \set{\rho \colon \br^{D}\to [0,+\infty) \colon \rho \text{ is measurable}, \int \rho(\theta)d\theta =1, M(\rho)<+\infty}, 
$$
where $M(\rho)=\int \|\theta\|^2\rho(\theta)d\theta$ is the second moment. 
For notational simplicity, we use $\rho$ to denote both a measure and its density when the latter exists.
Then PDE~\eqref{eq:pde} can be interpreted as a Wasserstein gradient flow on $\ch$, minimizing the following free energy: 
\begin{equation}\label{eq:energy}
    F_\sigma^{\bl}(\rho) = \frac{1}{2M}\sum_{j=1}^M\Big(h_\sigma(x_j,\rho)-y_j\Big)^2 +  \frac{\lambda}{2}M(\rho)-
\frac{1}{\beta}H(\rho) , 
\end{equation}
where $H(\rho)=-\int\rho(\theta)\log(\rho(\theta))d\theta$ is the entropy. 
Under suitable regularity assumptions on the activation function $\sigma$ and the initialization measure $\rho_0$, 
it has been shown that the free energy~\eqref{eq:energy} admits a unique minimizer $\rho^\ast$ in $\ch$, and that PDE~\eqref{eq:pde} converges weakly to this minimizer. 
Moreover, the minimizer satisfies the following condition: 
\begin{equation}\label{eq:minimizer}
    \rho^\ast(\theta) =Z_{\beta,\lambda,\rho^\ast}^{-1}  \exp(-\beta \Psi^\lambda_\sigma(\theta;\rho^\ast)), 
\end{equation}
where $Z_{\beta,\lambda, \rho}=\int \exp(-\beta \Psi^\lambda_\sigma(\theta;\rho)) d\theta $ is the partition function. 
This condition follows from the Euler-Lagrange equation characterizing minimizers of the free energy $F^{\bl}_\sigma$ \citep[see][Lemma 10.3]{mei2018mean}. 
Notably, since $\rho^\ast$ appears on both sides of~\eqref{eq:minimizer}, 
the equation constitutes a self-consistency condition for the stationary measure. 
As the right-hand side has the form of a Boltzmann distribution, \eqref{eq:minimizer} is referred to as a \emph{Boltzmann fixed point condition}.

The aforementioned results are summarized in the following theorem.

\begin{theorem}[\citealt{mei2018mean}]\label{thm:mei}
    Consider a neural network~\eqref{eq:model} with activation $\sigma$, trained using noisy SGD~\eqref{eq:SGD} with step size scheme $s_{k,N}=\varepsilon \xi(k\varepsilon)$, initialization distribution $\rho_0$, and $\beta\ge 1$, $\lambda>0$. 
    Assume that: 
    (i) $\xi\colon [0,\infty)\to [0,\infty)$ is bounded and Lipschitz continuous, with $\int_{0}^\infty \xi(t)dt=+\infty$;
    (ii) $\rho_0\in \cp_2^a(\br^D)$ and is sub-Gaussian; 
    (iii) 
    $\sigma$ is bounded and for any fixed $x\in \br^d$, $\sigma(x,\cdot)\in C^4(\br^D,\br)$ and $\nabla^k_\theta \sigma(x,\cdot)$ is bounded for $k=0,\ldots,4$. 
    Then, the free energy $F_\sigma^{\bl}$ admits a unique minimizer $\rho^\ast$ in $\ch$, which satisfies~\eqref{eq:minimizer}. 
    Furthermore, for any fixed $x\in \br^d$, the following holds almost surely with respect to the random initialization, the data and noise sampling in 
    the noisy SGD algorithm: 
    \begin{equation*}
        \lim_{t\to+\infty} \lim_{\substack{ N \to +\infty\\ \varepsilon \to 0}}
        \left|\frac{1}{N}\sum_{i=1}^N\sigma(x, \theta_i^{\lfloor t/\varepsilon\rfloor})
        -h_\sigma(x,\rho^\ast)\right|=0,
    \end{equation*}
    where the limit in $(N,\varepsilon)$ is taken along any sequence $(N,\varepsilon_N)$ satisfying that $(N,\varepsilon_N)\to(+\infty,0)$, 
    $\varepsilon_N \log (N/\varepsilon_N) \to 0$, and 
    $N/ \log (N/\varepsilon_N) \to \infty$.
\end{theorem}

Thus, the function $h_\sigma(\cdot,\rho^\ast)$ characterizes the implicit bias of noisy SGD training, 
and a fundamental goal of the present work is to 
understand its structure. 
This is challenging because the stationary measure $\rho^\ast$ is known only implicitly through the Boltzmann fixed point condition~\eqref{eq:minimizer}, a self-consistency equation that does not explicitly characterize either $\rho^\ast$ or the function it represents.

\subsection{Bounded and Smoothed Approximation of ReLU Neurons} 
\label{sec:relu}

Theorem~\ref{thm:mei} requires that the neuron map $\sigma(x,\theta)$ is bounded and smooth in the parameter $\theta$, assumptions that are not satisfied by the ReLU neuron 
\begin{equation}\label{eq:relu-neuron}
    \sigma(x,\theta)=a(w^\top x+b)_+. 
\end{equation}
In this subsection, we introduce a bounded and smooth approximation of the ReLU neuron that enables the application of Theorem~\ref{thm:mei}. 
Our construction follows the approximation introduced by \citet{shevchenko2022mean}. 

For $m,\tau>0$, let $(\cdot)^{m,\tau}_+$ denote the $C^4$ smooth, bounded approximation of the univariate ReLU activation, defined as follows: 
$$
(u)_+^{\mt} = \begin{cases}
    S_\tau(u)=\log(1+e^{\tau u})/\tau,& u < S_\tau^{-1}(m^2)\\
    \phi_{\mt}(u), & u\ge S_\tau^{-1}(m^2),
\end{cases}\quad u\in \br. 
$$
Thus, $(u)_+^{m,\tau}$ agrees with the softplus $S_\tau(u)=\tau^{-1}\log(1+e^{\tau u})$ until it reaches the value $m^2$, after which it is smoothly capped by the tail $\phi_{\mt}$. 
The tail is chosen such that $\|(\cdot)_+^{m,\tau}\|_\infty\leq 2m^2$ and $(\cdot)_+^{m,\tau}\in C^4(\mathbb R)$ with bounded derivatives up to the fourth order.

We also introduce a smooth, bounded approximation of the identity map in order to make the neuron output~\eqref{eq:relu-neuron} bounded in output weights $a$. 
For $m,\tau>0$ and $n\in \bn$, let 
$$
v^{\mt} = \begin{cases}
    v,& \|v\|< m-\tau^{-1}\\
    \psi_{\mt}(v), &\|v\|\ge m-\tau^{-1}, 
\end{cases}\quad v\in \br^n. 
$$
The tail $\psi_{\mt}$ is chosen so that
$\|(\cdot)^{\mt}\| \leq m$ 
and 
$(\cdot)^{\mt}\in C^4(\br^n,\br^n)$ with bounded derivatives up to the fourth order. 
In addition, we assume that the tail functions $\phi_{\mt}$ and $\psi_{\mt}$ satisfy suitable regularity properties; in particular, for every $m>0$ the limits $\lim_{\tau\to+\infty}(u)_+^{\mt}$ and $\lim_{\tau\to+\infty}v^{\mt}$ exist. 
For brevity, we defer the detailed construction to Appendix~\ref{app:relu}.

Define the \emph{$m$-truncated, $\tau$-smoothed ReLU neuron} with $m,\tau>0$: 
$$
\sigma^{\mt}(x,\theta)= a^{\mt} ( x^\top w^{\mt}+b )_+^{\mt},\quad x\in \br^d,\theta=(a,w,b)\in \br^{d+2}. 
$$
The truncation of the input weights $w$ is introduced only for technical reasons. 
By construction, the approximated ReLU neuron converges pointwise to the ReLU neuron as $\mt\to+\infty$, i.e., $\lim_{\mt\to+\infty} \sigma^{\mt}(x,\theta)=a(w^\top x+b)_+$ for every $(x,\theta)$. 
Moreover, $\sigma^{\mt}$ satisfies the assumptions of Theorem~\ref{thm:mei}. 
Let $\rho_{\bl,\mt}^\ast$ denote the stationary measure of the mean-field PDE~\eqref{eq:pde} associated with the neuron $\sigma^{\mt}$, 
regularization parameter $\lambda$, and temperature $\beta^{-1}$. 
With a slight abuse of notation, let 
$$
h^\ast_{\beta,\lambda,\mt}(\cdot)=h_{\sigma^{\mt}}(\cdot, \rho_{\bl,\mt}^\ast)
$$
denote the corresponding network function.

The goal of this work is to study the stationary measure $\rho^\ast_{\bl,\mt}$ and the corresponding solution function $h^\ast_{\bl,\mt}$ 
in order to characterize the implicit bias of noisy SGD in both parameter space and function space. 
We impose the following assumptions on $(\beta,\lambda,\mt)$.

\begin{assumption}\label{assum}
    Suppose $(\beta,\lambda,m,\tau)$ satisfies $\lambda,\tau>0$, and 
    \begin{equation*}
        \beta>\max\left\{1,\frac{1}{\lambda}, \frac{1}{\lambda} \log \frac{1}{\lambda}\right\}, \  m > \max\left\{C_0,\ \frac{1}{4}\left( \frac{C_1}{\sqrt{\lambda}}+ \sqrt{\frac{C_1^2}{\lambda}+4} \right)^2,\exp(2C_2\beta +2)\right\} , 
    \end{equation*}
    where $C_0,C_1,C_2$ are constants independent of $(\beta,\lambda, m,\tau)$. 
\end{assumption}

The next result removes the smoothing parameter $\tau$ by showing that both the stationary measure and the corresponding solution function have pointwise limits as $\tau\to+\infty$. 
Its proof is inspired by \citet[][Lemma A.4]{shevchenko2022mean} and is deferred to Appendix~\ref{app:tau}. 

\begin{lemma}[Zero-smoothing limit]
\label{lem:zerosmooth-h}
\label{prop:tau} 
    Under Assumption~\ref{assum}, 
    the following 
    pointwise limits exist: 
    $$
    \rho^\ast_{\bl,m}(\cdot)=\lim_{\tau\to+\infty} \rho^\ast_{\bl,\mt}(\cdot),\qquad
    h^\ast_{\bl,m}(\cdot)=\lim_{\tau\to+\infty}h^\ast_{\bl,\mt}(\cdot).
    $$
    Moreover, $h^\ast_{\bl,m}(\cdot)\in C^2(\br^d)$. 
\end{lemma}

We are interested in the limits of 
$\rho^\ast_{\bl,m}$ and $h^\ast_{\bl,m}$ as $m,\beta\to+\infty$, where $m\to+\infty$ recovers the ReLU activation and $\beta\to+\infty$ corresponds to the vanishing-temperature limit of noisy SGD. 
The existence of these limits 
\begin{equation}\label{eq:mblimit}
    \lim_{m,\beta\to+\infty}\rho^\ast_{\bl,m}(\cdot),\quad 
    \lim_{m,\beta \to+\infty}h^\ast_{\bl,m}(\cdot) 
\end{equation}
remains open even in the univariate setting \citep{shevchenko2022mean}. 
Several obstacles arise. As $\beta\to+\infty$, the free energy~\eqref{eq:energy} generally loses strict convexity and may therefore have multiple minimizers. Moreover, as $m\to+\infty$, the activation function becomes unbounded, causing the free energy functional to potentially lose lower semicontinuity. 
Consequently, the standard arguments used to establish the existence of minimizers no longer apply \citep[see, e.g.,][]{jordan1998variational}. 
These difficulties are intrinsic to the free energy and the unboundedness of the ReLU activation, rather than artifacts of the particular approximation introduced above.

Motivated by these obstacles, we do not assume that the limits in \eqref{eq:mblimit} exist. 
Instead, we analyze all accumulation points arising in the joint limit $\beta,m\to+\infty$. 
Specifically, we study the structural properties of any $\rho^\ast$ and $h^\ast$ for which there exists a sequence $\set{(\beta_n,\lambda_n,m_n)}$ satisfying $\beta_n,m_n\to+\infty$ and 
$h^\ast =\lim_{n\to+\infty}h^\ast_{\beta_n,\lambda_n,m_n}$ and $\rho^\ast =\lim_{n\to+\infty}\rho^\ast_{\beta_n,\lambda_n,m_n}$. 
The precise notion of these limits will be specified further below.

\section{Main Results}
\label{sec:main-results}

In this section, we present our main results on the implicit bias of SGD when training wide two-layer ReLU networks for multivariate regression. 

\subsection{Implicit Bias Toward Piecewise Affine Functions} 
\label{sec:pwl}

For a set of input data points $\cx=\set{x_j}_{j=1}^M \subset \br^d$, a linear \emph{dichotomy} is an unordered partition $\cx=\cx_+\sqcup \cx_-$ such that 
$\cx_+$ and $\cx_-$ are strictly separable by an affine hyperplane; that is, there exist $w\in \br^d$ and $b\in \br$ satisfying $x_j^\top w+b>0$ for $x_j \in \cx_+$ and $x_j^\top w+b<0$ for $x_j \in \cx_-$. 
Note that a dichotomy is strict and unordered. It should be distinguished from a binary sign pattern on $\cx$, where the two classes carry specified labels and interchanging the labels gives a different sign pattern, and from a linear trichotomy of $\cx$, where the separation hyperplane is allowed to pass through some of the input points. 
Let $\cp(\cx)$ denote the total number of realizable linear dichotomies of $\cx$.

The following theorem shows that the function obtained by training a ReLU network with noisy SGD in the vanishing-temperature regime is piecewise affine. 

\begin{theorem}[SGD Selects a Piecewise Affine Solution]\label{thm:pwl}
    Consider a training data set $\{(x_j,y_j)\}_{j=1}^M \subset \br^d\times \br$, with $d\geq 1$ and $x_i\neq x_j$ whenever $i\neq j$. 
    Let $\cx=\set{x_j}_{j=1}^M$ denote the set of training inputs. 
    For each $(\bl,\mt)$, let $h^\ast_{\bl,m,\tau}(x)$ denote the network function associated with the $m$-truncated, $\tau$-smoothed ReLU neuron, evaluated at the stationary measure $\rho^\ast_{\bl,\mt}$ of the mean-field PDE~\eqref{eq:pde}. 
    Let $h^\ast_{\bl,m}=\lim_{\tau\to+\infty} h^\ast_{\bl,\mt}$, whose existence follows from Lemma~\ref{lem:zerosmooth-h}. 
    Consider any sequence $(\beta_n,\lambda_n,m_n)$ satisfying Assumption~\ref{assum} for $n\geq 1$, with $m_n,\beta_n\to+\infty$ and $\lambda_n\to \bar \lambda$ for an arbitrary $\bar \lambda>0$ as $n\to+\infty$, 
    and such that $\lim_{n\to+\infty} h^\ast_{\beta_n,\lambda_n,m_n}(x)$ exists for every $x\in \br^d$. 
    Then there exist hyperplanes 
    $\pi_k = \set{x\in \br^d\colon x^\top w_k+b_k=0}$, $k=1,\ldots,K$, 
    with 
    \begin{equation}\label{eq:pwl-num}
    K\leq 2\cp(\cx)-1, 
    \end{equation}
    such that the limit function $h^\ast(\cdot)=\lim_{n\to+\infty} h^\ast_{\beta_n,\lambda_n,m_n}(\cdot)$ is affine on every connected component of $\br^d \setminus \bigcup_{k=1}^K \pi_k$. 
    Moreover, $h^\ast$ is Lipschitz continuous and satisfies 
    \begin{equation}\label{eq:lip}
        \|h^\ast\|_{\operatorname{Lip}} \leq 2(d+2)+\frac{2}{\bar \lambda M}\sum_{j=1}^M y_j^2.  
    \end{equation}
\end{theorem}

Theorem~\ref{thm:pwl}, together with Theorem~\ref{thm:mei}, shows that
training an infinitely wide ReLU network with noisy SGD on training inputs $\cx$ results in a continuous piecewise affine function with at most $2\cp(\cx)-1$ kink hyperplanes, as the network width $N$ tends to infinity and the temperature $\beta^{-1}$ vanishes. 
Note that $\cp(\cx)$ depends only on the affine configuration of the input points $\cx$, and is independent of both the network width $N$ and the ambient dimension $d$. 
Hence, the bound on the linear-region complexity of the learned function depends only on the geometry of the training data.

By a classical result of \citet{cover1965geometrical}, the bound~\eqref{eq:pwl-num} can be further estimated as follows:
\begin{equation}\label{eq:pwl-num-cover}
    2\cp(\cx)-1\leq 2 \sum_{i=0}^{d} \binom{M-1}{i}-1,
\end{equation}
where equality holds if and only if the training inputs $\cx$ are in general position, i.e., every subset of $d$ or fewer points is linearly independent. 
In the non-general position case, $\cp(\cx)$ can be computed by via Zaslavsky's Theorem \citep{zaslavsky1975facing} from the theory of hyperplane arrangements \citep[see, e.g.,][Theorem 2.5]{stanley2007introduction}.

We emphasize that the regularization strength $\bar \lambda = \lim_{n\to+\infty}\lambda_n$ can be an arbitrary positive real number. 
Thus, the strength of weight decay does not fundamentally affect the piecewise affine nature of the learned function. 
However, as shown in~\eqref{eq:lip}, it controls the slope of the learned function: stronger regularization leads to flatter solutions and a larger tradeoff between data fitting and regularization. 
A numerical illustration of the role of the regularization strength is provided in Figure~\ref{fig:wd}.

We reiterate that Theorem~\ref{thm:pwl} applies to any function $h^\ast$ 
that arises as the pointwise limit of a sequence $h^\ast_{\beta_n,\lambda_n,m_n}$. 
Proposition~\ref{prop:accumu} shows that this class of accumulation points is non-empty and arises naturally: every sequence $(\beta_n,\lambda_n,m_n)$ satisfying Assumption~\ref{assum}, with $m_n,\beta_n\to+\infty$ and $\lambda_n\to \bar \lambda$, has a subsequence $(\beta_{n_k},\lambda_{n_k},m_{n_k})$ such that $h^\ast_{\beta_{n_k},\lambda_{n_k},m_{n_k}}$ converges pointwise.

\subsection{Feature Learning} 
\label{sec:feature}

Theorem~\ref{thm:pwl} characterizes the implicit bias of noisy SGD in function space by showing that the algorithm is biased toward piecewise affine solution functions. 
We now examine the implicit bias in parameter space by studying the structure of the trained network parameters. 
This elucidates the feature-learning mechanism underlying the learning process.

Before stating our results, we emphasize that the piecewise affine structure of a solution function does not, in general, provide a tight characterization of the structure of the network parameters. 
For ReLU networks, the parameter fiber associated with a given function, i.e., the set of parameters realizing that function, may exhibit intricate geometric structure \citep[see, e.g.,][]{phuong2020functional, PetzkaTS20,grigsby2023hiddensymmetriesrelunetworks,grillo2026relunetworksadmitidentifiable,ramakrishnan2026completesymmetryclassificationshallow,gegenfurtner2026symmetriesthreelayerrelunetworks}. 
%}
In particular, a piecewise affine solution function could arise from sparse output weights while leaving the input weights essentially unchanged, corresponding to feature selection rather than feature learning. 
Moreover, there exist continuous piecewise affine functions that cannot be represented \emph{exactly} by any finite-width shallow ReLU network, 
such as $(u,v)\in \br^2 \mapsto \max(0,u,v)$. 
Therefore, to establish feature learning and effective width collapse, we investigate the stationary measure directly.

Consider the projection map $\omega(a,w,b)=(w,b)$, and let $\omega_\#\rho$ denote the push-forward on $\mathbb{R}^{d+1}$ of a measure $\rho \in \mathcal{P}(\mathbb{R}^{d+2})$. 
When $\rho^\ast$ is the stationary measure of the mean-field PDE, the measure $\omega_\#\rho^\ast$ is precisely the marginal measure of the input weights and biases of the trained network.

We now present our second main result, which establishes feature learning by showing that, after training, the input weights and biases align along finitely many data-dependent directions. 

\begin{theorem}[Feature Learning of SGD]\label{thm:feature}
    Consider the same assumptions and notations as in Theorem~\ref{thm:pwl}.  
    Assume additionally that $\rho_{\beta_n,\lambda_n,m_n}^\ast$ converges to some $\rho^\ast\in \mathcal{P}(\br^d)$ in $2$-Wasserstein distance as $n\to+\infty$. 
    Let 
    $
    q^\ast = \omega_\# \rho^\ast 
    $ 
    denote the marginal measure of the input weights and biases. 
    Then there exist $K \leq 2\cp(\cx) -1$ directions, represented by unit vectors $(w_k,b_k)\in\mathbb{R}^{d+1}$ with $\|(w_k,b_k)\|=1$ and associated rays 
    $$
    r_k = \set{(w,b)\in \br^{d+1}\colon (w,b)=\gamma (w_k,b_k) 
    \text{ for some } \gamma\geq0 
    },\quad k=1,\ldots, K,
    $$
    such that 
    \begin{equation}\label{eq:supp}
        \supp(q^\ast) \subset \bigcup_{k=1}^K r_k. 
    \end{equation}
\end{theorem}

Theorem~\ref{thm:feature} shows that the input weights and biases of the 
trained network align along at most $K \leq 2\cp(\cx)-1$ directions. 
In particular, the learned function is representable by a shallow ReLU network with $K$ hidden neurons: 
\begin{equation}\label{eq:small-network}
    \int a(x^\top w+b)_+d\rho^\ast\   = \  
    \sum_{k=1}^K a_k (x^\top w_k +b_k)_+\ \text{with } a_k = \int_{\set{(a,w,b)\colon (w,b)\in r_k}} a \|(w,b)\| d\rho^\ast. 
\end{equation}
Thus, the effective width of the network collapses to at most $2\cp(\cx)-1$.

We remark that \eqref{eq:supp} includes the degenerate case where the support is exactly the origin, i.e., $\supp(q^\ast)=\set{0}$. In this case, the represented function is identically zero, and the only learned feature is the zero feature. 
This occurs, for example, when $y_j=0$ for all $j$, in which case the stationary measure is the Gaussian $\rho^\ast(\theta)\propto \exp(-\frac{\beta \lambda}{2}\|\theta\|^2)$, which satisfies the Boltzmann fixed point condition~\eqref{eq:minimizer}. 
Outside this degenerate case, whenever $\supp(q^\ast) \neq \set{0}$, there is at least one aligned direction, and the network learns sparse, nonzero features. 
Empirically, we observe that after training, the input weights and biases often separate into two groups: one concentrated near the origin, whose contribution to the represented function is negligible, and another consisting of parameters with non-negligible norms that align along 
a few directions and constitute the learned nonzero features.

Theorem~\ref{thm:feature} applies to any measure $\rho^\ast$ that arises as the $2$-Wasserstein limit of a sequence $\rho^\ast_{\beta_n,\lambda_n,m_n}$. 
The assumption of $2$-Wasserstein convergence is used in the proof to obtain uniform integrability of the second moments (see Appendix~\ref{app:feature}). 
Establishing the existence of such a convergent sequence appears to be challenging in general. 
Nevertheless, Lemma~\ref{lem:Mbound} shows that any sequence $(\beta_n,\lambda_n,m_n)$ satisfying Assumption~\ref{assum}, with $m_n,\beta_n\to+\infty$ and $\lambda_n\to \bar\lambda>0$, admits a subsequence $(\beta_{n_k},\lambda_{n_k},m_{n_k})$ for which $\rho^\ast_{\beta_{n_k},\lambda_{n_k},m_{n_k}}$ converges weakly.

\subsection{Non-Redundancy of the Learned Features}
\label{sec:nonredun}

We have shown that the input weights and biases of the trained network align along $K\le 2\cp(\cx)-1$ directions, each corresponding to an effective hidden neuron as in~\eqref{eq:small-network}. 
We now characterize these directions and, consequently, the structure of the learned features.

For the given set $\cx=\set{x_j}_{j=1}^M$, 
define a linear \emph{ternary activation pattern} on $\cx$ as a map $A\colon \cx\to \set{-1,0,1}$ for which there exist $w\in \br^d$ and $b\in \br$ satisfying $\sgn(x_j^\top w+b)=A(x_j)$ for $x_j\in \cx$. 
We distinguish two subclasses of ternary activation patterns: 
if $A(x_j)\in\{\pm 1\}$ for every $x_j$, 
we call $A$ a \emph{binary activation pattern}. 
Otherwise, if $A(x_j)=0$ for at least one $x_j$, we call $A$ a \emph{strict ternary activation pattern}. 
We further introduce the following partial order on activation patterns: 
\begin{equation}\label{eq:partial-order}
    A' \preceq A \Longleftrightarrow A(x_j)\in \set{0, A'(x_j)},\,\forall j. 
\end{equation}
Thus, $A' \preceq A$ means that $A$ is obtained from $A'$ by replacing some (possibly none) of the $\pm1$ labels of $A'$ with zeros, while leaving all remaining labels unchanged. 
In this case, we say that $A$ is a \emph{degeneration} of $A'$. 
Any ternary activation pattern is a degeneration of itself: $A\preceq A$. 

We present our third main result, which describes the structure of the learned features of the ReLU networks after SGD training.

\begin{figure}[t]
    \centering
    \includegraphics[width=.9\linewidth]{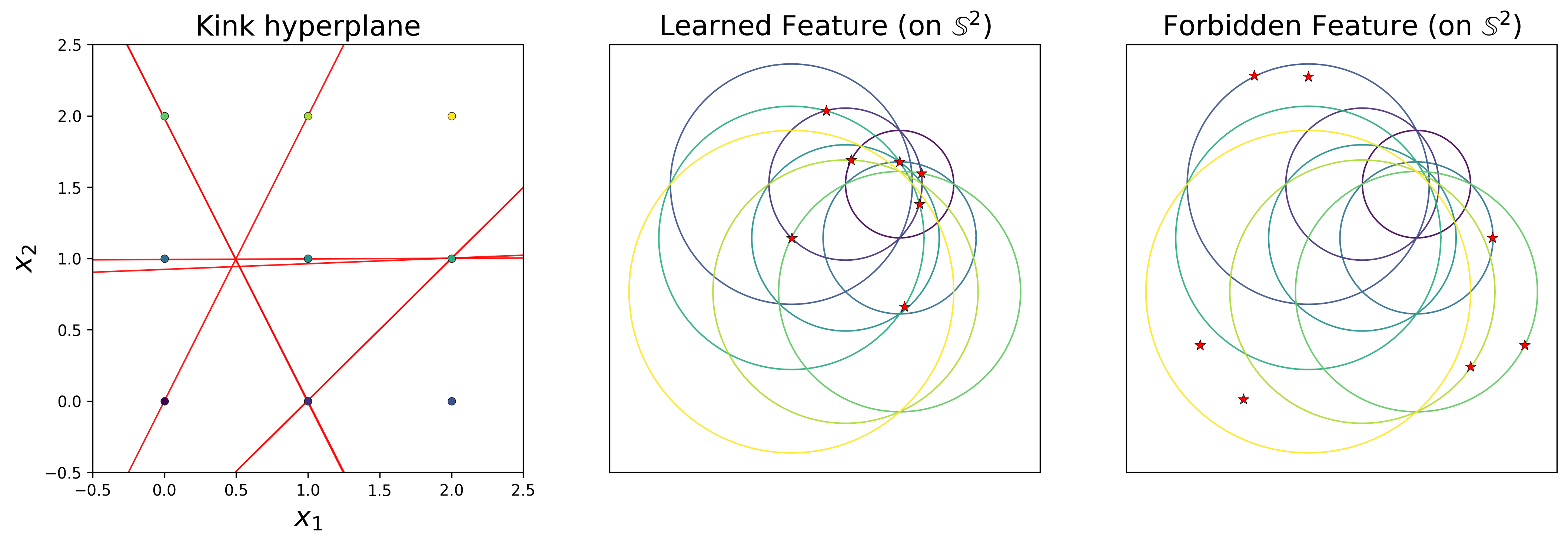}
    \caption{
    Non-redundancy of the learned features. Left: input points and kink hyperplanes of the trained network under the same setting of Figure~\ref{fig:intro}. 
    Middle: 
    the space $\bs^2$ of normalized input weights and biases after stereographic projection. 
    The red stars are the learned features, i.e., the directions that the input weights and biases align with. 
    Each direction induces a kink hyperplane on the left; 
    these kink hyperplanes may overlap.
    The colored circles represent the  
    hyperplanes 
    $\set{(w,b)\colon w^\top x_j+b=0}$, each corresponding to a training point $x_j$ on the left. 
    Right: examples of feature arrangements forbidden by Theorem~\ref{thm:nonredun}. 
    The bottom pair have the same sign pattern; 
    the top pair is comparable under the partial order; 
    the triplet on the right are degenerations of the same binary sign pattern, exceeding the allowed maximum of two.  
    }
    \label{fig:nonredun}
\end{figure}

\begin{theorem}[Non-redundancy of Learned Features]\label{thm:nonredun}
    Under the same notations and assumptions as in Theorem~\ref{thm:feature},
    let $\set{(w_k,b_k)}_{k=1}^K$ be the collection of distinct directions that the input weights and biases align with. 
    For $k=1,\ldots,K$, let $A_k$ denote the ternary sign pattern on the training inputs $\cx$ induced by the direction $(w_k,b_k)$: 
    $$
    A_k(x_j)= \sgn(x_j^\top w_k+b_k), \quad \forall j.
    $$
    Then the following holds: 
    \begin{enumerate}
        \item[(i)] 
        For any distinct $k, k'$, the two ternary sign patterns $A_k$ and $A_{k'}$ are incomparable: 
        neither $A_k\preceq A_{k'}$ nor $A_{k'}\preceq A_k$ holds. 
        In particular, $A_k\neq A_{k'}$ whenever $k\neq k'$. 

        \item[(ii)] 
        For any $k\in \set{1,\ldots,K}$, there exists at least one $j\in \set{1,\ldots,M}$ such that $A(x_j)= 1$.

        \item[(iii)] 
        Given an arbitrary linear binary activation pattern $\hat A$ of $\cx$, 
        the set
        \begin{equation}\label{eq:da}
            \cd(\hat A)=\set{A_k \colon \hat A \preceq A_k,\, k=1,\ldots,K}
        \end{equation}
        contains at most two elements. 
        Moreover, if $\cd(\hat A)$ has two elements, 
        then both of these elements are strict ternary activation patterns. 
    \end{enumerate}
\end{theorem}

Theorem~\ref{thm:nonredun} shows a strong non-redundancy of the learned features, illustrated in Figure~\ref{fig:nonredun}. 
Item $(i)$ shows that each effective neuron of the trained network induces a distinct activation pattern on the training inputs. 
Moreover, the activation patterns induced by different effective neurons are incomparable with respect to the partial order $\preceq$. 
Item $(ii)$ shows that every effective neuron activates at least one training point. 
Item~$(iii)$ states that, 
from any binary activation pattern $\hat A$ on $\mathcal X$, 
one can obtain at most two learned ternary patterns $A_k$ 
by turning some of the $\pm1$ labels of $\hat A$ into zeros.
Theorem~\ref{thm:nonredun} can be directly translated into a description of the locations of the kink hyperplanes of the learned functions; the formal statement is provided in Corollary~\ref{coro:kink}.

Identifying the exact aligned directions, or the exact location of the kink hyperplanes,  
is challenging for general training data. 
We obtain the following result for symmetric training data sets of size two.

\begin{proposition}[Kink Hyperplanes for Two-Point Data]
\label{prop:toy}
    Under the same notations and assumptions as in Theorem \ref{thm:feature}, 
    let $\cx=\{(x_1,y_1),(x_2,y_2)\} \subset \br^d\times \br$ 
    be a training data set of size two, with $x_1=-x_2$ and $|y_1|=|y_2|$. 
    Let $\Pi$ be the collection of kink hyperplanes $\pi_k=\{x\in \br^{d}\colon x^\top w_k+b_k =0\}$ for $k=1,\ldots, K$. 
    The following holds: 
    
    \begin{itemize}    
        \item \textbf{Even labels $y_1=y_2$:} 
        If $\|x_1\|<1$, then $\Pi=\varnothing$ and $h^\ast$ is globally affine. 
        If $\|x_1\|\geq 1$, then $\Pi$ consists of at most two hyperplanes: 
        $\Pi \subset \set{\pi_1',\pi_2'}$ where
        $\pi'_1=\{x\colon x_1^\top x+1=0\}$ and 
        $\pi'_2=\{x\colon x_1^\top x-1=0\}$. 
        
        \item \textbf{Odd labels $y_1=-y_2$:} 
        $\Pi$ consists of at most two hyperplanes: 
        $\Pi \subset \set{\pi_1',\pi_2'}$ where $\pi_1'=\{x\colon x_1^\top x+\bar b=0\}$, $\pi_2'=\{x\colon x_1^\top x-\bar b=0\}$, and $\bar b=\min\{1,\|x_1\|^2\}$. 
        
    \end{itemize}
\end{proposition}

\begin{figure}[t]
    \centering
    \includegraphics[width=1.\linewidth]{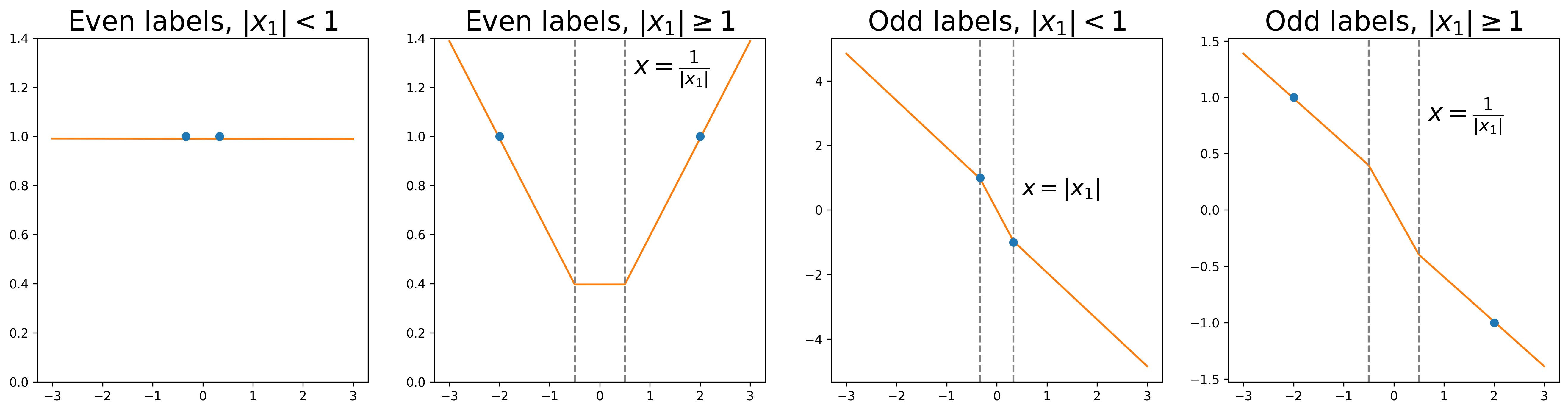}
    \caption{A two-layer univariate ReLU network with $N=1000$ hidden neurons, trained via SGD with weight decay on a symmetric data set of size two. 
    Blue dots represent the training data and orange curves the learned functions. 
    The figure illustrates how the location of the kink hyperplanes (knots) is determined by the training data, as predicted in Proposition~\ref{prop:toy}.
    }
    \label{fig:1D}
\end{figure}

Proposition~\ref{prop:toy} characterizes exactly the features extracted from simple symmetric data sets; see Figure~\ref{fig:1D} for an illustration. 
Interestingly, the solution function may or may not be globally linear depending on the norm of the training inputs. 
Another interesting phenomenon is that, the kink hyperplanes may or may not pass through the training points, which is in contrast to the linear spline interpolation solutions
observed in different training regimes~\citep{williams2019gradient, lai2023generalization} and in the optimal solutions \citep{savarese2019infinite, JMLR:v22:20-1447}. 
A similar observation was made by \citet{shevchenko2022mean} in the univariate setting: 
for the specific training set $\{(-10,2),(10,2)\}$, they showed that no function that is linear on the interval $[-10,10]$ can minimize the free energy. Consequently, the solution function must have at least one knot in the open interval $(-10,10)$. 
Here, we consider a more general setting and explicitly characterize the location of the kink hyperplanes.

We remark that characterizing the exact location of the kink hyperplanes for more general data requires a fine-grained understanding of the residuals of the learned predictor at the training points. We discuss this direction further in Section~\ref{sec:location}. A complete characterization remains an interesting and important problem for future work.

\section{Discussion of the Main Results}
\label{sec:discussion}

\paragraph{Generalization} 
The Lipschitz bound~\eqref{eq:lip} provides a standard route toward a generalization bound for the limiting solution function $h^\ast$ in Theorem~\ref{thm:pwl}. 
Specifically, under common boundedness assumptions 
of the input and output domains, 
the class of functions satisfying~\eqref{eq:lip} has controlled Rademacher complexity \citep[see, e.g.,][]{wainwright2019high}. 
Standard Rademacher-complexity arguments then yield a generalization bound \citep{bartlett2002rademacher}. 
We remark that the bound~\eqref{eq:lip} follows from the uniform second-moment bound on the stationary measures \(\rho^\ast_{\beta,\lambda,m}\) in Lemma~\ref{lem:Mbound}, which is a multivariate adaptation of Lemma A.2 in the work of \citet{shevchenko2022mean}. 

The bound on the effective width in Theorem~\ref{thm:feature} also allows one to control the network complexity, e.g., by bounding the VC-dimension via results in the work of \citet{bartlett2019nearly}. 
However, such a bound would be insufficient for generalization. 
A first issue is that the bound $2\cp(\cx)-1$ grows with the cardinality of the training set $M$. Moreover, controlling only the effective width (or the number of linear regions) is often too weak for obtaining non-vacuous generalization bounds, as piecewise linear functions with a bounded number of regions still constitute a large function class. 
A similar dependence on the sample size also appears in the bound of \citet{shevchenko2022mean} for univariate networks; we note that no generalization result is provided there either.

\citet{safran2022effective} obtained a generalization guarantee for univariate shallow ReLU networks based solely on a bound on the number of linear regions of the learned function. 
This is possible for two reasons: first, they considered a classification problem with the $0$-$1$ loss, where the prediction error depends only on the sign of the network output, not on its magnitude. In this setting, bounding the number of linear regions, and hence the number of sign changes of the prediction function, is sufficient to control the complexity of the function class relevant to the $0$-$1$ loss; 
second, they assumed that the labels are generated by a teacher network of width $r$, and by analyzing a KKT condition satisfied by the trained student network, 
they derived a bound on the number of linear regions that depends only on $r$, not on the sample size. 
The student-teacher setting is important, since fixed-size ReLU networks can have poor empirical performance on random labels even under the $0$-$1$ loss \citep{telgarsky2016benefits}.

\paragraph{The Role of Weight Decay} 
We considered a sequence of regularization parameters $\lambda_n>0$ converging to an arbitrary limit $\bar \lambda>0$. 
The positivity of each $\lambda_n$ is needed to ensure the
applicability of Theorem~\ref{thm:mei}, 
in particular the validity of the Boltzmann fixed point condition~\eqref{eq:minimizer}. As observed by \citet{mei2018mean}, when $\lambda=0$ and the neuron map $\sigma$ is bounded, the right-hand side of~\eqref{eq:minimizer} fails to be normalizable. 
The positivity assumption on $\bar \lambda$ is needed in the analysis of the limiting support of the measure $\rho^\ast_{\bl,m}$ (see Proposition~\ref{prop:shrink}).

We note that the particular value of $\bar \lambda$ has no bearing on the piecewise-affine nature of the solution function, 
the alignment of the learned input weights and biases, 
or the non-redundancy property of the learned features. 
Nevertheless, it controls the Lipschitz constant of the solution function; see~\eqref{eq:lip} and Figure~\ref{fig:wd}.

\paragraph{The Role of Stochasticity in SGD} 
The stochasticity in SGD, including both data sampling and injected Gaussian noise, is needed in the mean-field analysis in Theorem~\ref{thm:mei}. 
Several alternative training procedures have also been studied in the mean-field regime, such as mini-batch and full-batch training, training without noise, training with Gaussian noise under different scalings, and training with heavy-tailed noise \citep{mei2018mean,chizat2018global,mei2019mean, sirignano2020mean,wan2024implicit,descours2024law}. 
In these settings, however, either convergence of the mean-field PDE dynamics in time to a stationary measure has not been established, or the stationary measures lack a sufficiently explicit characterization. As a result, analyzing the corresponding implicit bias remains challenging. 
Nevertheless, neither the specific data sampling scheme nor the precise noise model enters our analysis. The proofs rely only on the Boltzmann fixed point condition~\eqref{eq:minimizer} and on structural properties of the ReLU neuron.

\paragraph{The Role of Initialization and Scaling} 
Notice that Theorem~\ref{thm:mei} imposes only mild assumptions on the initialization distribution $\rho_0$: namely, that $\rho_0\in \mathcal{P}_2^a$ and is sub-Gaussian. 
Consequently, our results apply to any initialization satisfying these conditions. 
This broad applicability stems from the injected Gaussian noise in SGD, or equivalently from the entropy regularization appearing in the free energy. The entropy term makes the free energy strictly convex, thereby ensuring the existence of a unique stationary measure. 

We consider wide neural networks in the mean-field regime, 
using the 
parametrization with one-over-width scaling in \eqref{eq:model} and an initialization distribution that is independent of the network width. 
Alternative choices of the parametrization and initialization scaling give rise to different training regimes, with different training dynamics and implicit biases under gradient-based optimization \citep[see, e.g.,][]{williams2019gradient, JMLR:v22:20-1123, yang2021tensor}.

\paragraph{Comparison with \citet{shevchenko2022mean}}
\citet{shevchenko2022mean} studied univariate network with $m$-truncated, $\tau$-smoothed ReLU neurons, trained by noisy SGD with temperature $\beta^{-1}$ and weight decay $\lambda$ on $M$ data points. 
They showed that, as $\beta,\tau,m\to+\infty$, the learned function is approximated by piecewise linear functions with $O(M)$ knots. 
Theorem~\ref{thm:pwl} extends this result to the multivariate case $d\geq1$, where the bound on the number of kink hyperplanes has order $O(\sum_{i=0}^{d} \binom{M-1}{i})$. 
Regarding the configuration of the knots, \citet{shevchenko2022mean} showed that there are at most three knots between any two consecutive training points and, when three knots occur, two must coincide with the endpoints of the interval, while the third lies strictly in its interior.
This appears as a special case of Corollary~\ref{coro:kink}, which is a direct consequence of Theorem~\ref{thm:nonredun} (see Appendix~\ref{app:kink} for details). 
Moreover, Corollary~\ref{coro:kink} yields an additional consequence in the univariate setting: there can be at most two knots in $(-\infty, x_1]\cup [x_M,+\infty)$, where $x_1=\min_j x_j$ and $x_M=\max_j x_j$. If two knots are present, they must occur precisely at the boundary data points $x_1$ and $x_M$. To our knowledge, this observation is not contained in the analysis of \citet{shevchenko2022mean}.

Beyond extending the above results to general input dimensions, a key advance over \citet{shevchenko2022mean} is that we investigate the implicit bias in parameter space. 
Specifically, Theorem~\ref{thm:feature} shows that the learned input weights and biases align along finitely many directions, leading to an effective collapse of the network width. 
Theorem~\ref{thm:nonredun} further characterizes the structure of these directions and the resulting learned features.  
As discussed in Section~\ref{sec:feature}, these results are nontrivial and do not follow from the piecewise-affine structure of the represented function alone.

Compared with \cite{shevchenko2022mean}, the present work has the following limitations. 
The piecewise-linearity result of \cite{shevchenko2022mean} applies to $\lim_n \lambda_n\ge 0$ including the regime of vanishing regularization, whereas our analysis requires $\bar \lambda=\lim_n \lambda_n>0$. 
As discussed earlier, the positivity of $\bar \lambda$ is needed to analyze the limiting support of the stationary measure $\rho^\ast_{\bl,m}$. 
The support analysis is a key ingredient in our study of the parameter-space implicit bias and, to our knowledge, has no counterpart in the work of \citet{shevchenko2022mean}, where the support of the limiting measure is not investigated.

Additionally, to establish the piecewise-affine structure of the solution function, both \citet{shevchenko2022mean} and our work rely on a bound on the second derivative of the solution function for $x$ outside of a failure set $U$. 
This bound vanishes as $\beta,m\to+\infty$ and the set $U$ shrinks to a union of points in the univariate case or hyperplanes in the multivariate case. 
In the univariate setting, \citet[][Theorem 1]{shevchenko2022mean} further derives a non-asymptotic bound on the Lebesgue measure of $U$. 
We do not obtain an analogous bound for $d>1$. 
In our analysis, the failure set $U$ consists of $x$ that violates the Quadratic Positivity Condition. This condition requires that, in the space of hyperplanes, the set $\hat G_x$ of all hyperplanes passing through $x$ be disjoint from another set $\hat \Omega^{\beta,\lambda,m}$ (see Sections~\ref{sec:qpc}--\ref{sec:shrink} for details). 
Although we obtained a non-asymptotic bound on the size of $\hat \Omega^{\beta,\lambda,m}$ (see Appendix~\ref{app:shrink}), 
converting this information into a quantitative estimate for the measure of the failure set in the input domain appears challenging.

\paragraph{Comparison with \citet{de2020sparsity,del2026dual}}
\citet{de2020sparsity,del2026dual} studied the minimal total variation norm (TV norm) interpolation problem: 
\begin{equation}\label{eq:del}
    \inf_{\mu\in M(\bs^{d})} 
    \|\mu \|_{\mathrm{TV}} \quad \mathrm{s.t. }\ \int (w^\top x_j+b)_+   d\mu(w,b) = y_j,\quad j=1,\ldots,M,
\end{equation} 
where $M(\bs^{d})$ is the space of signed Radon measures on $\bs^{d}$. 
They proved that all minimizers of~\eqref{eq:del} are atomic measures, of the form $\rho^\ast = \sum_{k=1}^{K'}c_i \delta_{w_i}$,
and thus correspond to ReLU networks with at most $K'$ hidden neurons.
As discussed in Section~\ref{sec:related-works}, the key difference between problem~\eqref{eq:del} and the free-energy-minimization problem considered in this work is their connection to training dynamics:
the latter precisely characterizes the implicit bias of noisy SGD, 
whereas there is no general guarantee that the minimizer of problem~\eqref{eq:del} is reached by the training procedure. 
Moreover, due to the entropy regularizer, the minimizer of the free energy at any finite temperature $\beta^{-1}>0$ must be absolutely continuous. 
By contrast, atomic measures belong to the feasible set of~\eqref{eq:del} and are naturally promoted by the TV regularizer.

\citet{de2020sparsity} showed that $K'\leq 2\cp(\cx)$ (see the proof of Theorem 4.2 therein), which exceeds our bound~\eqref{eq:pwl-num} by one. 
\citet[][Lemma 3.11]{del2026dual} showed that 
$K' \leq \max_{k=0,\ldots,M} \binom{M}{k}\Xi(M-k,d+1-k)$,
where $\Xi(n,d)$ is the number of regions in $\br^d$
produced by an arrangement of $n$ central hyperplanes. 
In Appendix~\ref{app:bound} we show that, for generic data, this bound is strictly larger than ours, and that in the regime $2d+1 \leq M$, our bound improves upon it by a factor of at least $2^d/(d+1)^2$.

\section{Proof of the Main Results} 
\label{sec:proof}

In this section, we outline the proofs of our main results. 
In our analysis, we mostly focus on the input weights and biases of the ReLU network, as they define the hidden features and the kink hyperplanes of the output function. 
Thus, throughout this section, by parameter space we mean the space of input weights and biases $\br^{d+1}$.

In Section~\ref{sec:param}, we introduce the notation and a decomposition of the parameter space induced by the training data. 
In Section~\ref{sec:zero-smooth}, we derive a bound on the limit of the Hessian of the solution function $h^\ast_{\bl,\mt}(x)$ as $\tau\to+\infty$. 
In Section~\ref{sec:qpc}, we introduce the Quadratic Positivity Condition at a test location $x$. This condition requires, in the parameter space, that the set of all hyperplanes passing through $x$ be disjoint from a cluster set $\Omega^{\beta,\lambda,m}$, defined by the non-positive sets of certain quadratic functions appearing in the Hessian bound. 
For every $x$ satisfying this condition, we obtain a further bound on the limiting Hessian that vanishes as $\beta,m\to+\infty$. 
In Section~\ref{sec:shrink}, we show that the cluster set $\Omega^{\beta,\lambda,m}$ shrinks to a limiting set $\Omega^\ast$, given by a union of $K\le 2\cp(\cx)-1$ rays in $\br^{d+1}$. 
These rays define $K$ hyperplanes in input space. 
Based on this, we then show that  
(i) for every $x$ at positive distance from these $K$ hyperplanes, the Hessian at $x$ vanishes as $\beta,m\to+\infty$; and 
(ii) the marginal measure of input weights and biases has vanishing mass outside these $K$ rays. 
Finally, in Section~\ref{sec:location}, we investigate the geometry of the limiting set $\Omega^\ast$.

Our proof strategy builds on the analysis of \citet{shevchenko2022mean} for univariate networks. 
Their argument has two main components. 
First, they derive a bound on the Hessian at a test location $x$ that vanishes as $\beta,m\to+\infty$, provided that an associated quadratic function $f_{\beta,m}$ is positive at $x$. 
Second, they show that the non-positive set of $f_{\beta,m}$ has vanishing Lebesgue measure as $\beta,m\to+\infty$, and hence shrinks to a finite set of points. 
In their analysis, the quadratic function $f_{\beta,m}$ is analyzed on the input domain and plays a central role. 
As noted by the authors, extending these arguments beyond one dimension presents two main obstacles. 
First, there is no simple analogue of the interval decomposition of the input domain: 
in $\br^1$, the training inputs partition the input domain into intervals, on each of which the activation pattern of every ReLU neuron is fixed. 
Second, whereas in one dimension the non-positive set of a quadratic function is the union of at most two intervals, in higher dimensions the corresponding quadric sublevel sets can exhibit considerably more intricate geometry, making their concentration difficult to analyze. 
We overcome both obstacles by shifting the analysis from the input space to the parameter space, where the training data induce a natural decomposition. 
We identify the relevant quadratic functions on the parameter space and analyze the shrinkage of their non-positive sets. 

In the sequel, we let $\alpha=(\bl,m)$ 
and, for two quantities $A,B$, we write $A\lesssim B$ if there exists a constant $C$ independent of $(\bl,m,\tau,x)$, such that $A \leq C B$.

\subsection{Subdivision of Parameter Space}
\label{sec:param}

We begin by introducing a natural decomposition of the parameter space induced by the training data, together with the notation needed for our analysis.

For a single ReLU neuron $a(w^\top x + b)_+$, the non-linearity occurs along the hyperplane 
$\{x \in \mathbb{R}^d \colon w^\top x + b = 0\}$. 
This defines a natural mapping from the parameter space to the set of affine hyperplanes in the input space. 
Conversely, for any $x\in \br^d$, let 
$G_{x}= \set{(w,b)\in \br^{d+1}\colon w^\top x + b=0}$, which is the set of parameters $(w,b)$ whose associated input-space hyperplane passes through $x$. 
Given a training set $\cx=\{x_j\}_{j=1}^M$, consider the arrangement of hyperplanes $\{G_{x_j}\}_{j=1}^M\subset \br^{d+1}$. 
Let $\set{U_k}$ denote the sectors of the arrangement, that is, the open connected components, of $\br^{d+1}\setminus \cup_j  G_{x_j}$. 
Each sector $U_k$ corresponds to a unique binary sign pattern on $\cx$ and determines an index set 
$$
J_k=\set{j \in \set{1,\ldots,M} \colon \sgn(w^\top x_j + b)=1,\, \forall (w,b)\in U_k}. 
$$
Thus, the sectors are in one-to-one correspondence with the binary sign patterns on $\cx$ realizable by affine functions. 
The number of such binary sign patterns is exactly twice the number of linear dichotomies. 
Hence, there are precisely $2\cp(\cx)$ sectors in total. 
For analytical convenience, we work with their closures and write $O_k = \cl(U_k)$ for $k=1,\ldots,2\cp(\cx)$.

Fix a point $x \in \mathbb{R}^d$ at which the Hessian of the solution function is to be evaluated.
The hyperplanes $\{G_{x_j}\}_{j=1}^M$ induce a decomposition of $G_x$, $G_x =  \bigcup_k (G_x \cap 
O_k
)$. 
Each open sector $G_x \cap U_k$ consists of those parameters that define a hyperplane passing through the evaluation point $x$ while inducing the binary sign pattern with index set $J_k$ on the training set. 
The set $G_x$ can be naturally identified with $\br^d$ via the map $w\mapsto (w,-w^\top x)$. 
Then, for $k=1,\ldots,2\cp(\cx)$, the closed sector $G_x\cap O_k$ can be identified with 
$$
O_k^w=\set{w\in\br^d \colon (w,-w^\top x)\in O_k}. 
$$

\subsection{Zero-Smoothing Limit}
\label{sec:zero-smooth}

Next, we control the Hessian of the solution function in the zero-smoothing limit $\tau\to+\infty$. 
By the construction of $\sigma^{\mt}$ in Appendix~\ref{app:relu}, 
the limits $(u)_+^m=\lim_{\tau \to+\infty} (u)_+^{\mt}$ and $v^m=\lim_{\tau \to+\infty} v^{\mt}$ exist for every $m>0$; in particular, 
\begin{equation}\label{eq:vm}
v^m=\lim_{\tau \to+\infty} v^{\mt}= \begin{cases}
    v, & \|v\|\leq m\\
    \frac{m}{\|v\|} v, &\|v\|>m.
\end{cases}
\end{equation}
Let $\sigma^m$ denote the \emph{$m$-truncated} ReLU neuron in the zero-smoothing limit: 
$$
\sigma^m(x,\theta)=\lim_{\tau\to+\infty}\sigma^{\mt}(x,\theta)=a^m(x^\top w^m +b)_+^m . 
$$
Since the maps $u\mapsto (u)_+^{\mt}$ and $v\mapsto v^{\mt}$ are bounded and have bounded first derivatives with the bounds independent of $\tau$, 
 $\sigma^m$ is bounded and Lipschitz continuous for every $m>0$.

By results of \citet[Lemma 10.2--10.4]{mei2018mean}, the free energy $F_{\sigma^{m}}^{\bl}$ associated with the $m$-truncated ReLU neuron $\sigma^m$,
with regularization $\lambda>0$ and temperature $\beta^{-1}>0$, admits a unique minimizer, which we denote by $\rho^\ast_\alpha$. 
Define
$$
h^\ast_{\alpha}(\cdot)=h_{\sigma^m}(\cdot,\rho^\ast_{\alpha}), \qquad R_j^\alpha = \frac{1}{M} \Big( h^\ast_{\alpha}(x_j)- y_j \Big),\, j=1,\ldots,M. 
$$
Thus, $h_\alpha^\ast$ is the solution function associated with the $m$-truncated neuron $\sigma^{m}$ and $R_j^\alpha$ is the residual of $h^\ast_{\alpha}$ at the $j$th data point, divided by $M$.

A priori, the limit $\lim_{\tau\to+\infty} h^\ast_{\bl,\mt}$ need not coincide with $h^\ast_{\alpha}$, 
since $h^\ast_{\bl,\mt}$ is obtained by first minimizing the free energy associated with the smoothed and truncated neuron $\sigma^{\mt}$ and then taking $\tau\to+\infty$, 
whereas $h^\ast_\alpha$ is obtained by minimizing the free energy associated directly with the limiting neuron $\sigma^m$.
The following result shows that, as $\tau\to+\infty$, not only $h^\ast_{\bl,\mt}$ but also its Hessian converge to $h^\ast_\alpha$ and its associated Hessian, respectively. This proves Lemma~\ref{prop:tau}. We also bound the norm of the Hessian in the zero-smoothing limit. The proof is provided in Appendix~\ref{app:tau}. 

\begin{proposition}[Zero-Smoothing Limit]\label{prop:tau-limit}
    For any $\alpha=(\bl,m)$ that satisfies Assumption~\ref{assum},
    we have that 
    $h^\ast_{\alpha}\in C^2(\br^d)$ and 
    $$
    h^\ast_{\alpha}(x)= \lim_{\tau\to+\infty}h^\ast_{\bl,\mt}(x),
    \quad \nabla_x^2 h^\ast_{\alpha}(x)=\lim_{\tau\to+\infty}\nabla_x^2 h^\ast_{\bl,\mt}(x), \quad \text{for every $x\in \br^d$}.
    $$
    Furthermore, 
    \begin{equation}\label{eq:hessian-quantity}
        \|\nabla_x^2 h^\ast_{\alpha}(x)\| 
        \lesssim 
        m^{-1}(1+\lambda^{-1})
        +\beta^{-1}\lambda^{-\frac{d+6}{4}}(1+\sqrt{\lambda})^{\frac{d+2}{2}} 
        + 
        \beta^{\frac{d+1}{2}}\lambda^{\frac{d}{4}}(1+\sqrt{\lambda})^{\frac{d+2}{2}}\mathcal{I}^\alpha(x),
    \end{equation}
    where 
    \begin{align*}
        \mathcal{I}^\alpha(x)&=\sum_{k=1}^{2\cp(\cx)} \int_{O_k^w} \|w\|^2 |(Q_k^\alpha)^\top w^m| \exp\curv{-\frac{\lambda\beta}{2}\Big(\|w\|^2 +|x^\top w^m|^2 - ((Q_k^\alpha)^\top w^m)^2\Big)} dw\\
        Q_k^\alpha & = \frac{1}{\lambda}\sum_{j\in J_k} R^\alpha_j (x_j-x)\in \br^d, \qquad k=1,\ldots,2\cp(\cx), 
    \end{align*}
    with the convention that $Q_k^\alpha=0$ if $J_k=\varnothing$. 
\end{proposition}

Thus, the Hessian is upper bounded by three terms. The first two vanish as $m,\beta\to+\infty$. The third is more complicated and depends on $\mathcal{I}^\alpha(x)$.

\subsection{Quadratic Positivity Condition}
\label{sec:qpc}

We introduce a condition for controlling the term $\mathcal{I}^\alpha(x)$ in the right-hand side of~\eqref{eq:hessian-quantity} and thus controlling the Hessian. 
To motivate the condition, recall from~\eqref{eq:vm} that $w^m=\gamma_w w$ for some $\gamma_w\in [0,1]$. Observe that if
\begin{equation}\label{eq:hessian-condition-2}
\|w\|^2+|x^\top w|^2-|(Q_k^\alpha)^\top w|^2 \geq \delta \|w\|^2, \quad \forall w \in O_k^w, \ \forall k, 
\end{equation}
for some $\delta \in (0,1]$, 
then, 
\begin{equation}\label{eq:okw}
\begin{aligned}
\|w\|^2+|x^\top w^m|^2 -|(Q_k^\alpha)^\top w^m|^2 
&=(1-\gamma_w^2)\|w\|^2 + \gamma_w^2(\|w\|^2+|x^\top w|^2-|(Q_k^\alpha)^\top w|^2)\\
& \geq (1-\gamma^2_w(1-\delta))\|w\|^2\\
& \geq \delta\|w\|^2. 
\end{aligned}
\end{equation}
Under this condition, each integral in $\mathcal{I}^\alpha(x)$ can be controlled by the third moment of a Gaussian distribution as follows: 
\begin{align*}
&\int_{O_k^w} \|w\|^2 |(Q^\alpha_k)^\top w^m| \exp\curv{-\frac{\lambda\beta}{2}\Big( \|w\|^2+|x^\top w^m|^2-|(Q^\alpha_k)^\top w^m|^2\Big)} dw    \\
\leq & \sqrt{1+\|x\|^2}\int_{\br^d} \|w\|^3 \exp\curv{-\frac{\lambda\beta \delta}{2}\|w\|^2} dw\\
\lesssim& \sqrt{1+\|x\|^2} (\lambda\beta\delta)^{-\frac{d+3}{2}}.
\end{align*}
Then the Hessian can be bounded as 
$$
\|\nabla_x^2 h^\ast_{\alpha}(x)\| 
\lesssim m^{-1}(1+\lambda^{-1})
+\beta^{-1}\lambda^{-\frac{d+6}{4}}(1+\sqrt{\lambda})^{\frac{d+2}{2}} 
\Big( 1 + 
\delta^{-\frac{d+3}{2}} \sqrt{1+\|x\|^2}\Big),
$$
where the right-hand side vanishes as $m,\beta\to+\infty$, provided that $\delta$ is lower bounded.

We reformulate condition~\eqref{eq:hessian-condition-2} in a more convenient form. 
For $k=1,\ldots,2\cp(\cx)$, define a quadratic function 
\begin{equation}\label{eq:f}
f^\alpha_k(w,b)
= \|w\|^2+b^2 - \Big( \frac{1}{\lambda}\sum_{j\in J_k} R^\alpha_j (x_j^\top w +b) \Big)^2, 
\end{equation}
with the convention that $f_k^\alpha(w,b)=\|w\|^2+b^2$ if $J_k=\varnothing$. 
Note that 
$f_k^\alpha$ is 2-homogeneous and that 
the left-hand side of~\eqref{eq:hessian-condition-2} is precisely $f^\alpha_k(w,-w^\top x)$. 
Thus, \eqref{eq:hessian-condition-2} is satisfied if 
\begin{equation}\label{eq:qpc}
\min_{k\in \set{1,\ldots,2\cp(\cx)}} \inf_{\substack{(w,b)\in G_x\cap O_k\\ \|w\|^2+b^2=1}} f_k^\alpha(w,b)>0,
\end{equation}
with the convention that $\inf_{\substack{(w,b)\in S}} f(w,b)=+\infty$ if $S=\varnothing$.
We refer to~\eqref{eq:qpc} as the \emph{Quadratic Positivity Condition}, 
as it requires the positivity of $f_k^\alpha$ on the intersection of the hyperplane $G_x$ and the closed sector $O_k$ for every $k$. 

We have the following result 
bounding the integral in the right-hand side of~\eqref{eq:hessian-quantity} and thus the Hessian of the solution function at $x$, 
provided that $x$ satisfies the Quadratic Positivity Condition. 
The proof is provided in Appendix~\ref{app:qpc}.

\begin{proposition}[Quadratic Positivity Condition]
\label{prop:qpc}
    Assume that $\alpha=(\bl,m)$ satisfies Assumption~\ref{assum}.
    Assume that $x\in \br^d$ satisfies the Quadratic Positivity Condition~\eqref{eq:qpc}. 
    Let 
    \begin{equation*}\label{eq:margin}
        \delta = \min \Big\{ \min_{k} \inf_{\substack{(w,b)\in G_x\cap O_k\\ \|w\|^2+b^2=1}} f_k^\alpha(w,b),\ \frac{1}{2} \Big\}\in \Big(0,\frac{1}{2}\Big]. 
    \end{equation*}
    Then,
    \begin{equation}\label{eq:hessian-bound}
        \|\nabla_x^2 h^\ast_{\alpha}(x)\|
    \lesssim  
    m^{-1}(1+\lambda^{-1})
        +\beta^{-1}\lambda^{-\frac{d+6}{4}}(1+\sqrt{\lambda})^{\frac{d+2}{2}}
        \left(1
        + 
        \sqrt{1+\|x\|^2}\cdot \delta^{-\frac{d+3}{2}}
        \right). 
    \end{equation} 
\end{proposition}

\subsection{Shrinkage of Cluster Set}
\label{sec:shrink}

According to the bound~\eqref{eq:hessian-bound}, 
the Hessian vanishes as $m,\beta\to+\infty$ and $\lambda\to\bar \lambda>0$, 
provided that the margin $\delta$ remains uniformly bounded away from zero throughout this limiting process. 
We now examine when $\delta$ admits such a lower bound.

\begin{figure}[t]
    \centering
    \includegraphics[width=0.25\linewidth]{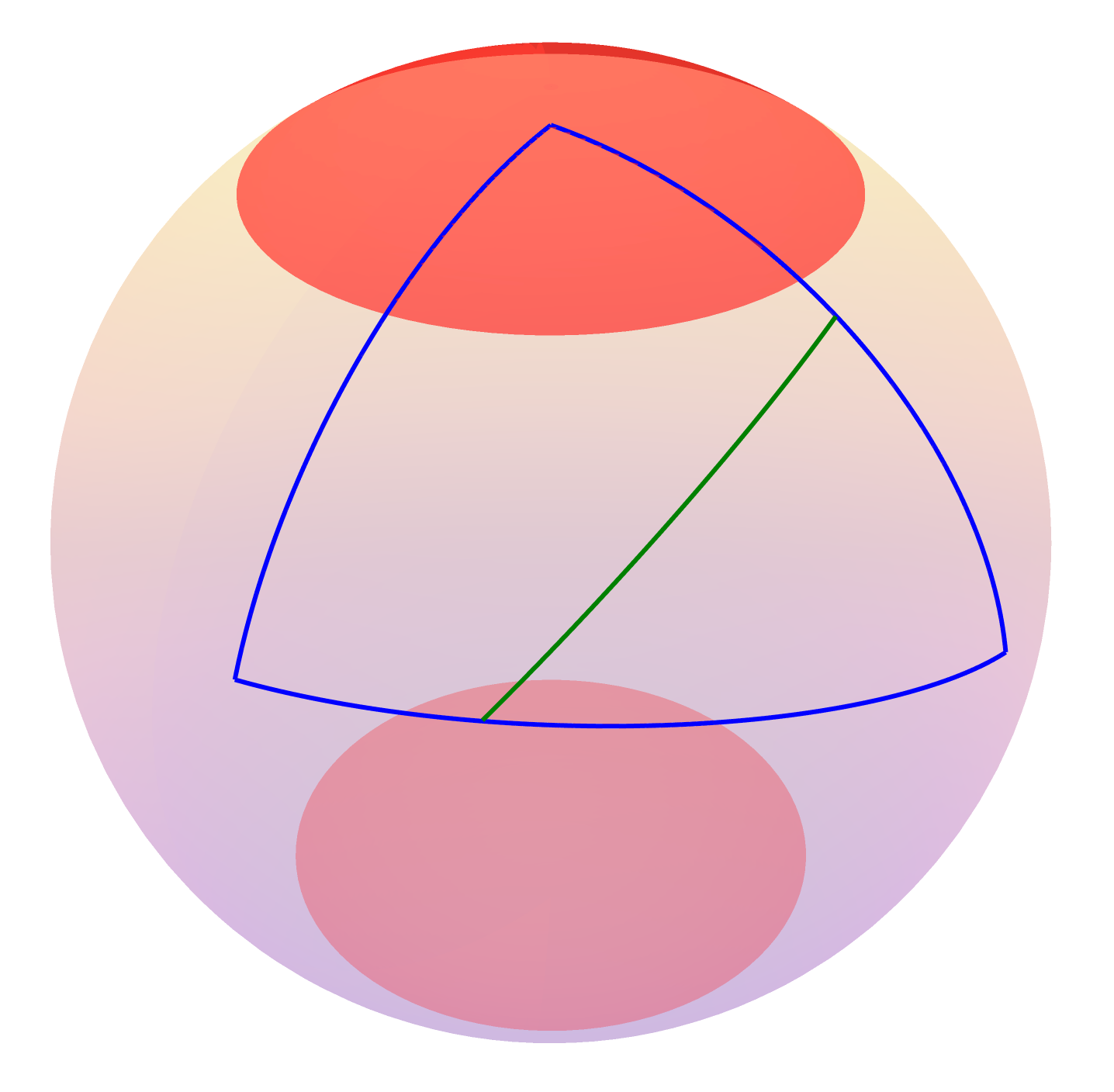}
    \caption{
    Illustration of the Quadratic Positivity Condition for $d=2$. 
    For a fixed $k$, 
    the red antipodal disks represent $\hat \Omega_k^\alpha$, the non-positive set of the quadratic function $f_k^\alpha$; 
    the region bounded by the blue arcs is $\hat O_k$, the sector corresponding to a particular binary pattern on the training inputs; 
    and the green arc is $\hat G_x\cap \hat O_k$, the set of parameters defining a hyperplane that passes through $x$ while inducing that binary pattern on the training set. 
    } 
    \label{fig:qpc}
\end{figure}

To this end, we derive a geometric description of the Quadratic Positivity Condition. 
Let 
\begin{equation*}
    P_k^\alpha = \frac{1}{\lambda}\sum_{j\in J_k}  R^\alpha_j \begin{pmatrix}
    x_j\\1
    \end{pmatrix} \in \br^{d+1},
\end{equation*}
with the convention that $P_k^\alpha=0$ if $J_k=\varnothing$. 
Writing $z=(w,b)$, we can express $f_k^\alpha$ as the quadratic form
$$
f_k^\alpha(z)= z^\top \left(I_{d+1} -P_k^\alpha (P_k^\alpha)^\top \right) z = z^\top \Lambda_k^\alpha z.
$$
The coefficient matrix $\Lambda_k^\alpha$ has eigenvalues $1$ with multiplicity $d$ and $1-\|P_k^\alpha\|^2$. 
Thus, if we define the non-positive set of $f_k^\alpha$ by
\begin{equation*}
    \Omega_k^\alpha = \set{(w,b)\in\br^{d+1} \colon f_k^\alpha (w,b)\leq 0},
\end{equation*}
then $\Omega_k^\alpha$ is a quadratic double cone when $\|P_k^\alpha\|^2>1$, a line when $\|P_k^\alpha\|^2=1$, and the singleton $\set{(0,0)}$ when $\|P_k^\alpha\|^2<1$.
Here and throughout, we call a set $S\subset \br^{d+1}$ a cone if $v\in S$ implies $c v\in S$ for all $c \geq 0$. 
Notably, \eqref{eq:qpc} is equivalent to requiring that the cones $G_x\cap O_k$ and $\Omega_k^\alpha$ intersect only at the origin. 
This condition can be expressed on the unit sphere. 
For any cone $S\subset \br^{d+1}$, let 
$\hat S = S\cap \bs^d$ denote its spherical cross-section. 
Since a cone is fully determined by its spherical cross-section, \eqref{eq:qpc} is equivalent to 
$$
(\hat G_x \cap \hat O_k ) \cap \hat \Omega_k^\alpha = \varnothing, \quad \forall k, 
$$
illustrated in Figure~\ref{fig:qpc}, 
or, equivalently, 
$$
\hat G_x \cap \hat \Omega^\alpha =  \varnothing, \quad \text{where} \quad \hat \Omega^\alpha = \bigcup_{k=1}^{2\cp(\cx)} \hat O_k \cap \hat \Omega^\alpha_k. 
$$

We refer to the set $\hat \Omega^\alpha$ as the \emph{cluster set}, 
in analogy to the terminology of \citet{shevchenko2022mean}. 
Next we show that, 
under a suitable sequence $\alpha_n=(\beta_n,\lambda_n,m_n)$, the cluster set $\hat \Omega^{\alpha_n}$ shrinks to at most $2\cp(\cx)-1$ fixed points on the sphere $\bs^d$, as $n\to+\infty$. 
Consequently, for any $x\in\br^d$ such that $\hat G_x$ avoids these limiting points,
we can obtain a uniform positive lower bound on the margin $\delta$.
The following result formalizes this shrinkage of the cluster set.
Its proof is provided in Appendix~\ref{app:shrink}.

\begin{proposition}[Shrinkage of Cluster Set]\label{prop:shrink}
    Consider any sequence $\alpha_n=(\beta_n,\lambda_n,m_n)$ 
    that satisfies Assumption~\ref{assum} for $n\geq 1$, 
    with $m_n\to+\infty$ and $\lambda_n\to \bar \lambda$ for an arbitrary $\bar \lambda>0$ as $n\to+\infty$, 
    and such that $\lim_{n\to+\infty} h^\ast_{{\alpha_n}}(x)$ exists for every $x\in \br^d$. 
    For $k=1,\ldots,2\cp(\cx)$, let
    $P^\ast_{k}= \lim_{n\to+\infty} P_k^{\alpha_n}$, 
    $f^\ast_k(w,b) = z^\top(I_{d+1}- P^\ast_{k} (P^\ast_{k})^\top)z$ with $z=(w,b)$,  
    and $\hat \Omega^\ast_k= \{(w,b)\in \bs^{d} \colon f_k^\ast(w,b)\leq 0\}$. 
    Define
    \begin{equation}\label{eq:omegaast}
        \hat \Omega^\ast = \bigcup_{k=1}^{2\cp(\cx)} \hat O_k \cap \hat \Omega_k^\ast. 
    \end{equation}
    Then, $\hat \Omega^\ast$ consists of at most $2\cp(\cx)-1$ points on $\bs^d$ and, furthermore, 
    for any $\varepsilon>0$ there exists $N>0$ such that for all $n>N$,
    $$
    \hat \Omega^{\alpha_n} \subset \set{z\in \bs^d\colon d_{\angle}(z, \hat \Omega^\ast)\leq \varepsilon},
    $$
    where $d_\angle(z,S)=\inf_{z'\in S}\arccos(z^\top z')$ denotes the spherical point-to-set distance for $z\in \bs^d$ and $S\subset \bs^d$, with the convention that $d_\angle(z,S)=+\infty$ if $S=\varnothing$. 
\end{proposition}

Notice that the limit $\lim_{n\to +\infty}P_k^\alpha$ exists for every $k$, as both limits, $\lim_{n\to+\infty} \lambda_n$ and $\lim_{n\to+\infty}R^{\alpha_n}_j$, exist by our assumptions.

Based on Proposition~\ref{prop:qpc} and Proposition~\ref{prop:shrink}, we obtain the following result, thereby proving Theorem~\ref{thm:pwl}.
The proof is given in Appendix~\ref{app:pwl2}.

\begin{theorem}\label{thm:pwl-detail}
    Consider any sequence $\alpha_n=(\beta_n,\lambda_n,m_n)$ 
    that satisfies Assumption~\ref{assum} for $n\geq 1$, 
    with $m_n,\beta_n\to+\infty$ and $\lambda_n\to \bar \lambda$ for an arbitrary $\bar \lambda>0$ as $n\to+\infty$, 
    and such that $\lim_{n\to+\infty} h^\ast_{\alpha}(x) $ exists for every $x\in \br^d$. 
    Let $\hat \Omega^\ast = \set{(w_k, b_k)}_{k=1}^K$ be defined as in~\eqref{eq:omegaast}, with $K\leq 2\cp(\cx)-1$.  
    Let $\pi_k = \set{x\in \br^d \colon x^\top w_k +b_k=0}$ for $k=1,\ldots,K$. 
    Then for any compact set 
    $\Gamma \subset \br^d \setminus \cup_k \pi_k$, we have that 
    $$
    \lim_{n\to+\infty} \sup_{x\in \Gamma} \|\nabla_x^2 h^\ast_{\alpha_n}(x)\|=0. 
    $$
    Consequently, the function $h^\ast(\cdot)=\lim_{n\to+\infty} h^\ast_{\alpha_n}(\cdot)$ is affine on each connected component of $\br^d \setminus \cup_k \pi_k$. 
    Moreover, $h^\ast$ is Lipschitz continuous and satisfies 
    \begin{equation*}
        \|h^\ast\|_{\operatorname{Lip}} \leq 2(d+2)+\frac{2}{\bar \lambda M}\sum_{j=1}^M y_j^2.  
    \end{equation*}
\end{theorem}

The shrinkage of the cluster set also allows us to analyze the support of the stationary measure $\rho^\ast_{\alpha}$. 
Specifically, for any set $B\subset \br^{d+1}$ such that the projection of $\cl(B)$ onto the sphere is disjoint from $\hat \Omega^\ast$, 
we show that
$$
\lim_{n\to+\infty}\mathbb{E}_{\theta\sim\rho^\ast_{\alpha_n}}\left[\|(w,b)\|^2 \cdot \mathbb{I}_{(w,b)\in B}\right]=0. 
$$
That is, the $\|(w,b)\|^2$-weighted mass of $\rho^\ast_{\alpha_n}$ over $\br\times B$ vanishes as $n\to+\infty$. 
This directly gives the desired results in Theorem~\ref{thm:feature}. 
The formal proof is given in Appendix~\ref{app:feature}.

\subsection{Position of Limiting Cluster Set Relative to Sphere Stratification} 
\label{sec:location}

In view of the previous subsection, the limiting set $\hat \Omega^\ast$ determines 
the aligned directions in parameter space and the kink hyperplanes in input space. 
We now study the location of the points in $\hat \Omega^\ast=\set{(w_k,b_k)}$ relative to a stratification of the sphere $\bs^d$ induced by ternary activation patterns on the training inputs. 

Recall from Section~\ref{sec:nonredun} the definition of linear ternary activation patterns on the training set. 
Each ternary activation pattern $A\colon \cx\to\set{-1,0,1}$ defines a stratum on the sphere $\bs^d$: 
$$
V_A = \set{(w,b)\in \bs^d \colon \sgn(w^\top x_j +b)=A(x_j),\ \forall j }. 
$$
Let $\mathcal{V}=\{V_A\}_A$, where $A$ ranges over all linear ternary activation patterns realizable on $\cx$. 
The following basic result shows that $\mathcal{V}$ forms a stratification of the sphere $\bs^d$ and the strata have simple geometry. 
Its proof is given in Appendix~\ref{app:nonredun}. 

\begin{lemma}\label{lem:strat} 
    The strata form a disjoint decomposition of the sphere: 
    $\bs^d = \sqcup_{V\in \mathcal{V}} V$. 
    Moreover, every stratum $V$ other than $V_0=\{(w,b)\in\bs^d\colon \sgn(w^\top x_j +b)=0\,\forall j\}$ is simply connected and contained in an open hemisphere of $\bs^{d}$. 
\end{lemma}

Notably, all the notions related to ternary activation patterns introduced in Section~\ref{sec:nonredun} can be translated into geometric language. 
To know the activation pattern induced by the hyperplane 
$\pi_k=\{x\colon w_k^\top x+b_k=0\}$, it suffices to know which stratum of $\mathcal{V}$ contains the point $(w_k,b_k)$. 
A ternary activation pattern $A$ is a binary sign pattern if and only if $V_A$ is a $d$-dimensional stratum, i.e., has the same dimension as the sphere $\mathbb{S}^d$, 
and is a strict ternary sign pattern if and only if $V_{A}$ has positive codimension on $\bs^d$. 
For the partial order on ternary sign patterns defined in~\eqref{eq:partial-order}, 
it is straightforward to verify that
\begin{equation*}
    A' \preceq A \Longleftrightarrow  \cl(V_{A'}) \supset V_A. 
\end{equation*}
In addition, for a given binary sign pattern $\hat A$, 
the set $\mathcal{D}(\hat A)$ defined in~\eqref{eq:da} can be rewritten as
$$
\mathcal{D}(\hat A) = \set{A_k \colon (w_k,b_k)\in \cl(V_{\hat A}),k=1,\ldots, K}. 
$$

Exploiting the geometry of the limiting set $\hat \Omega^\ast$, together with the structure of the strata in Lemma~\ref{lem:strat}, we obtain the following result, 
which provides a geometric proof of Theorem~\ref{thm:nonredun}. 
Its proof is provided in Appendix~\ref{app:nonredun}. 
We write $\partial_{\cl(S)}S$ for the relative boundary of $S$ in its closure for any $S\subset \bs^d$.

\begin{proposition}[Geometry of Limiting Cluster Set]\label{prop:strat}
    Consider the stratification $\mathcal{V}$ of the sphere induced by the training inputs $\cx$.  Let $\hat \Omega^\ast$ be defined as in Theorem~\ref{thm:pwl-detail}. 
    The following holds:
    \begin{itemize}
        \item[(i)] 
        Every stratum $V\in \mathcal{V}$ contains at most one point in $\hat \Omega^\ast$. 
        Moreover, $\hat \Omega^\ast \cap V \neq \varnothing$ implies that $\hat \Omega^\ast \cap \partial_{\cl(V)} V=\varnothing$. 

        \item[(ii)]
        $\hat \Omega^\ast \cap \cl(V_{-1})=\varnothing$, where $V_{-1}=\{(w,b)\in\bs^d\colon \sgn(w^\top x_j +b)=-1\,\forall j\}$. 

        \item[(iii)] 
        For an arbitrary binary sign pattern $\hat A$ on $\cx$, 
        $\cl(V_{\hat A})\cap \hat \Omega^\ast$ contains at most two points. 
        If $\cl(V_{\hat A}) \cap \hat \Omega^\ast$ contains two points, then both points must lie in the boundary $\partial_{\cl(V_{\hat A})}V_{\hat A}$. 
    \end{itemize}
\end{proposition}

Finally, we briefly outline the proof of Proposition~\ref{prop:toy}. 
As indicated by Theorem~\ref{thm:pwl-detail}, 
the vectors
$$
P^\ast_k = \frac{1}{\bar \lambda M} \sum_{j\in J_k} \left(h^\ast(x_j)-y_j\right)\begin{pmatrix}
    x_j\\1
\end{pmatrix},\quad k=1,\ldots,2\cp(\cx), 
$$ 
characterize the set $\hat \Omega^\ast$ and hence the locations of the kink hyperplanes $\{\pi_k\}_{k=1}^K$.
For the specific data set
$\cx=\{(x_1,y_1),(x_2=-x_1,y_2=y_1)\}$, 
we have 
(i) $h^\ast$ is an even function and $h^\ast(x_1)-y_1=h^\ast(x_2)-y_2$;
and
(ii) there are only four possible binary sign patterns: $J_1=\set{1,2},J_2=\set{1},J_3=\set{2},J_4=\varnothing$. 
These observations allow us to explicitly compute $P^\ast_k$ for $k=1,\ldots,4$, thereby identifying the locations of the hyperplanes. 
The case of odd labels follows a similar strategy. 
The formal proof is provided in Appendix~\ref{app:toy}. 
For general training data, determining $P^\ast_k$ remains challenging, 
in particular if one relies solely on the Boltzmann fixed point condition. 
We expect that additional information about the free-energy minimizer is needed.

\section{Experiments}
\label{sec:experiments}

In Figure~\ref{fig:intro} and Figure~\ref{fig:nonredun}, we consider a two-layer ReLU network of width $N=1000$ and initialize the neuron parameters $(a,w,b)$ as i.i.d.\ samples from the Gaussian $\mathcal{N}(0, \nu^2 I_{d+2})$ with $\nu=10^{-3}$. 
On a training data set of size $M=9$ with two-dimensional inputs, 
we run the SGD algorithm~\eqref{eq:SGD} with step size $s_{N,k} \equiv 10^{-3}$, 
weight decay regularization strength $\lambda=10^{-2}$ and no Gaussian noise, $\beta^{-1}=0$, for $2\times 10^6$ iterations. 
The learned features are identified as follows: 
(i) among all trained input weights and biases ${(w^i,b^i)}_{i=1}^N$, keep only those with norm larger than $10^{-8}$ and normalize them onto the unit sphere; 
and 
(ii) apply DBSCAN \citep{ester1996density}, a standard clustering method, to the normalized points under the spherical distance: 
$d_\angle(z,z')=\arccos(z^\top z')$ for $z,z'\in \bs^d$. 
In our experiment, the resulting clusters had maximal within-cluster spherical distance $7.9\times 10^{-4}$, approximately $0.045^\circ$. 
These results are consistent with Theorem~\ref{thm:feature} and Theorem~\ref{thm:nonredun}.

In Figure~\ref{fig:1D}, we report the result of training a two-layer ReLU network on a univariate data set of size $M=2$ via SGD for $10^6$ iterations. 
The network width, initialization, step size, regularization, and noise level are the same as in the previous experiment. 
The result is consistent with Proposition~\ref{prop:toy}.

Figure~\ref{fig:wd} and Figure~\ref{fig:noise}
illustrate the roles of weight decay and Gaussian noise in SGD, respectively. 
In these figures, 
a two-layer ReLU network is trained on a two-dimensional data set via SGD for $10^6$ iterations. 
The network width, initialization, and baseline SGD parameters are the same as before. 
The results are consistent with Theorem~\ref{thm:pwl}.

\begin{figure}[t]
    \centering
    \includegraphics[width=1\linewidth]{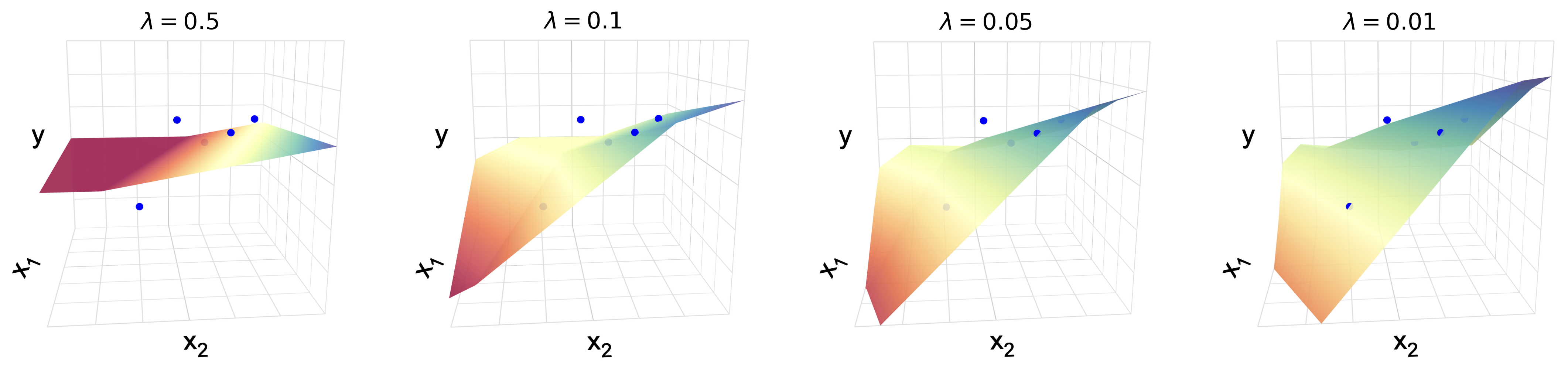}
    \caption{
    A two-layer ReLU network is trained via SGD without noise on two-dimensional input data, under various levels of weight decay $\lambda$. The figure shows that the trained function is piecewise affine for all values of $\lambda$, whereby larger $\lambda$ promotes a smaller Lipschitz constant. 
    }
    \label{fig:wd}
\end{figure}

\begin{figure}[t]
    \centering
    \includegraphics[width=1\linewidth]{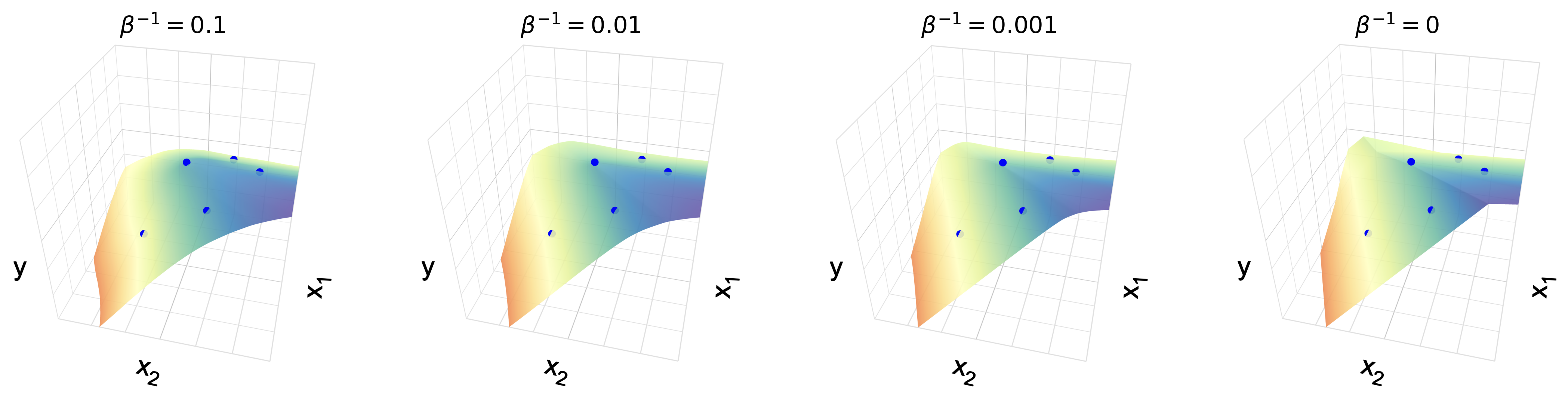}
    \caption{
    A two-layer ReLU network is trained via noisy SGD with weight decay on two-dimensional input data, under various noise levels $\beta^{-1}$. 
    The figure shows that, as the noise level tends to zero, the trained function converges to a piecewise affine function, while for larger noise levels, the function is increasingly smooth. 
    }
    \label{fig:noise}
\end{figure}

\section{Conclusion}

We presented an analysis of SGD training of wide shallow ReLU networks.
We showed in the mean-field limit that the learned function is piecewise linear. 
At convergence, the input weights and biases of the network align along a finite number of directions depending on the training data, so the network learns a finite collection of features and its effective width collapses. 
These directions satisfy a non-redundancy property: distinct directions induce distinct ternary linear sign patterns on the training inputs.

These results contribute to the theoretical understanding of the implicit bias and feature-learning mechanisms of gradient-based training in nonlinear neural networks. 
While it is widely believed that gradient-based optimization can discover useful feature representations, a rigorous mathematical understanding of this phenomenon has remained limited. 
Previous theoretical results establishing the emergence of sparse or structured features have either relied on specialized assumptions on the training data or initialization, or have characterized optimal solutions of associated variational problems. In contrast, our analysis shows that an effectively finite width representation emerges directly from the training dynamics in a broad mean-field setting, and relates its structure to the combinatorial geometry of the training data.

\paragraph{Limitations} 
Our analysis is formulated within the mean-field framework of \citet{mei2018mean} and therefore considers two-layer networks in the infinite-width limit trained by noisy SGD with weight decay. 
Rather than analyzing ReLU neurons directly, we study bounded, smooth approximations and subsequently pass to the ReLU limit. 
This approximation is motivated by the current state of the theory. 
To our knowledge, a direct initialization-to-convergence analysis for mean-field training of ReLU networks has not yet been established.

\paragraph{Future work}
Several interesting directions remain open. A natural next step is to obtain finer-grained descriptions of the learned features, including characterizations of the precise locations of the kink hyperplanes beyond the special classes of training data considered in  Proposition~\ref{prop:toy}. 
Another important question is whether the bound on the number of learned features is tight. 
Our numerical experiments, together with those of \citet{shevchenko2022mean}, suggest that this bound may admit further improvement, possibly depending on the specific structure of the data. 
It would also be of interest to extend the analysis to other optimization procedures (e.g., gradient descent, mini-batch SGD, and preconditioned SGD), classification losses, different activation functions, and more complex architectures such as deep networks or attention-based models.

\section*{Reproducibility Statement}
Code to reproduce our experiments is available at \url{https://github.com/shuangliang15/effective-width-collapse}.

% Acknowledgements and Disclosure of Funding should go at the end, before appendices and references

\acks{This project has been supported in part by NSF grants DMS-2145630 and CCF-2212520. TJ was supported by the European Research Council (ERC) under the Horizon Europe Framework Programme (HORIZON) through grant agreement number 101116395 (SPARSE-ML). 
GM was supported in part by DARPA grant  HR00112520014 through the Artificial Intelligence Quantified (AIQ) program, NSF grant DMS-2522495, DFG project 464109215 within the Priority Programme SPP 2298 ``Theoretical Foundations of Deep Learning'', and the BMFTR through DAAD project 57616814 (SECAI).}

%All acknowledgements go at the end of the paper before appendices and references. Moreover, you are required to declare funding (financial activities supporting the submitted work) and competing interests (related financial activities outside the submitted work). More information about this disclosure can be found on the JMLR website.


\newpage

\appendix

\section*{Appendix}

The appendix is organized into the following sections.

\begin{itemize}
    \item Appendix~\ref{app:relu}: Approximation of the ReLU Neuron
    
    \item Appendix~\ref{app:pwl}: Proof of Theorem~\ref{thm:pwl}

    \item Appendix~\ref{app:feature}: Proof of Theorem~\ref{thm:feature}
    
    \item Appendix~\ref{app:nonredun}: Proof of Theorem~\ref{thm:nonredun}
    
    \item Appendix~\ref{app:toy}: Proof of Proposition~\ref{prop:toy}

    \item Appendix~\ref{app:bound}: Comparison of the Bounds on Effective Width

    \item Appendix~\ref{app:accumu}: Existence of Accumulation Points

    \item Appendix~\ref{app:kink}: Non-redundancy of Kink Hyperplane Arrangement
\end{itemize}

\section{Approximation of the ReLU Neuron} 
\label{app:relu}

In this section, we provide a detailed construction of the $m$-truncated, $\tau$-smoothed ReLU neuron $\sigma^{\mt}$.

\subsection{Truncated and Smoothed ReLU Neuron}

For approximating the ReLU activation function, let 
$$
(u)_+^{\mt} = \begin{cases}
    S_\tau(u)=\log(1+e^{\tau u})/\tau,& u < S_\tau^{-1}(m^2)\\
    \phi_{\mt}(u), & u\ge S_\tau^{-1}(m^2),
\end{cases}\quad u\in \br. 
$$
The tail is chosen so that $(\cdot)_+^{m,\tau}\in C^4(\mathbb R)$ and satisfies
$$
\|\phi_{m,\tau}\|_\infty\le 2m^2,\qquad
\|\phi_{m,\tau}'\|_\infty \le 1,\qquad
\|\phi_{m,\tau}''\|_\infty \le m^{-2},
$$
with bounded third and fourth derivatives. 
Moreover, as $\tau\to\infty$, $(\cdot)_+^{m,\tau}$ converges to a truncated ReLU $(\cdot)_+^m$, defined as follows:  
$$
(u)_+^m=\begin{cases}
    (u)_+,& u< m^2\\ \phi_m(u),& u\geq m^2,
\end{cases}\quad u\in \br,
$$ 
which satisfies that $(\cdot )_+^m\in C^4((0,+\infty))$ and $\phi_{\mt}''$ converges to $\phi_m''$ uniformly on any compact subset of $(m^2,+\infty)$. 
An explicit construction of $\{(\cdot)_+^{\mt}\}$ is given in the next subsection.

For approximating the identity map, let 
$$
v^{\mt} = \begin{cases}
    v,& \|v\|< m-\tau^{-1}\\
    \psi_{\mt}(v), &\|v\|\ge m-\tau^{-1}, 
\end{cases}\quad v\in \br^n. 
$$
The tail function $\psi_{\mt}$ is selected such that $(\cdot)^{\mt}\in C^4(\br^n,\br^n)$, is an odd function, and satisfies that $\|\psi_{\mt}(u)\|\leq m, \|\nabla \psi_{\mt}(u)\|\leq 1$, 
and $\|\nabla^{(k)} \psi_{\mt}\|$ are bounded for $k=2,3,4$.
Additionally, we require that as $\tau\to+\infty$, 
$(\cdot)^{\mt}$ converges to the truncated identity map $(\cdot)^m$, defined as follows
$$
v^m=\begin{cases}
    v,&\|v\|<m\\ m\cdot v/\|v\|, &\|v\|\geq m. 
\end{cases}\quad v\in \br^n. 
$$ 
The existence of the family $\set{(\cdot)^{\mt}}$ is clear. 
The $m$-truncated, $\tau$-smoothed ReLU neuron is defined as:
$$
\sigma^{\mt}(x,\theta)=a^{\mt}(x^\top w^{\mt}+b)_+^{\mt}, \qquad x\in \br^{d},\theta=(a,w,b)\in \br^{d+2}. 
$$

\subsection{
  \texorpdfstring{An Explicit Construction of the Family $\set{(u)_+^{\mt}}$}{An Explicit Construction}
}

We construct a family $\{(u)_+^{m,\tau}\}$ that approximates the univariate ReLU function and satisfies the requirements in previous subsection. 

Select a univariate function
$g\colon [1,+\infty)\to \br$ that satisfies
\begin{itemize}
    \item $g(u)\in[1,2],\ \forall u\geq 1$,
    \item  there exists $\delta>0$ such that $g(u)=u$ when $u\in [1,1+\delta]$, 
    \item $g\in C^4$; for all $u\geq 1$, 
    $g'(u)\in [0,1], g''(u)\in[-1,0], |g^{(3)}(u)|\leq C, |g^{(4)}(u)|\leq C$. 
\end{itemize}
Such a function $g$ implements a truncation. 
The existence of such a function $g$ is clear. 
Let 
$$
S_\tau (u)=\frac{\log(1+e^{\tau u})}{\tau}. 
$$
This is the $\tau$-scaled softplus function, which is a smooth approximation of $(u)_+=\max\{0,u\}$.
Define the tail function
$$
\phi_{\mt}(u)= m^2 g\Big( \frac{S_\tau(u)}{m^2} \Big).
$$
Let
$$
(u)_+^{\mt}=\mathbb{I}_{u< c_{\mt}} S_\tau(u) + \mathbb{I}_{u \ge c_{\mt}} \phi_{\mt}(u). 
$$
Here $c_{m,\tau}=S_\tau^{-1}(m^2)$.

As $S_\tau(c_{\mt})=m^2$ and $S_\tau$ is monotonically increasing, 
there exists $\delta'$ such that
$$
\frac{S_\tau(u)}{m^2}\in[1,1+\delta],\quad \forall u\in[c_{\mt},c_{\mt}+\delta']. 
$$
Then,
$$
\phi_{\mt}(u)=S_\tau(u),\quad \forall u\in[c_{\mt},c_{\mt}+\delta'].
$$
Therefore, $(u)_+^{\mt}$ is $C^4$. 
Since $1\leq  g\le 2$, we have that $|\phi_{\mt}(u)|\leq 2m^2$. 
For the first derivative,
$$
\phi'_{\mt}(u)=g'\Big( \frac{S_\tau(u)}{m^2} \Big) S'_\tau(u).
$$
Since $0\le g'\le 1$ and $0\le S_\tau' \le 1$, we have that $|\phi'_{\mt}(u)|\le 1$. 
For the second derivative, 
$$
\phi''_{\mt}(u)=S''_\tau(u)g'\Big( \frac{S_\tau(u)}{m^2} \Big)+ \frac{1}{m^2}g''\Big( \frac{S_\tau(u)}{m^2} \Big)(S'_\tau(u))^2. 
$$
Since $S''_\tau < 1/S_\tau$, we have that
$$
S''_\tau(u) \leq \frac{1}{m^2}, \quad \forall u\geq c_{\mt}. 
$$
Since $-1\le g'' \le 0$, we have that
$$
\phi''_{\mt}(u) \le S''_\tau(u)g'\Big( \frac{S_\tau(u)}{m^2} \Big)\leq \frac{1}{m^2}. 
$$
Also, as $S''_\tau \geq 0$, 
$$
\phi''_{\mt}(u) \ge -\frac{1}{m^2}(S'_\tau(u))\geq -\frac{1}{m^2}.
$$
Since $g^{(3)},g^{(4)},S_\tau^{(3)},S_\tau^{(4)}$ are bounded, $\phi_{\mt}^{(3)},\phi_{\mt}^{(4)}$ are bounded.

Now define
$$
\phi_m (u) = m^2 g(\frac{u}{m^2}), \quad u\geq m^2,
$$
and let 
$$
(u)_+^m= \mathbb{I}_{u<m^2}\cdot (u)_+ +\mathbb{I}_{u\geq m^2} \phi_m(u). 
$$
Since $S_\tau(u)\to (u)_+$ pointwise and $c_{\mt}\to m^2$ as $\tau\to+\infty$, 
we have $(u)_+^{\mt}\to (u)_+^m$ as $\tau\to+\infty$.  
By similar arguments as above, $(u)_+^m\in C^4((0,+\infty))$. 
Consider any compact set $K\subset (m^2,+\infty)$. 
For all large $\tau$, $K \subset (c_{\mt},+\infty)$. 
Note that 
$$
S_\tau(u)\to u,\quad S'_\tau(u)\to 1,\quad S''_\tau(u)\to 0
$$
uniformly on $K$ as $\tau\to+\infty$. 
Hence, 
$$
\phi''_{\mt}(u) \to \frac{1}{m^2}g''(\frac{u}{m^2})=\phi_m''(u)
$$
uniformly on $K$ as $\tau\to+\infty$. 
By above, the family $\set{(u)^{\mt}_+}_{\mt}$ satisfies all the previous requirements.

\section{Proof of Theorem~\ref{thm:pwl}}
\label{app:pwl}

In this section, we present the proof of Theorem~\ref{thm:pwl}. 
In Appendix~\ref{app:tau}, Proposition~\ref{prop:tau-limit} controls the Hessian of the solution function in the zero-smoothing limit; 
In Appendix~\ref{app:qpc}, Proposition~\ref{prop:qpc} further bounds the limiting Hessian under the Quadratic Positivity Condition;
In Appendix~\ref{app:shrink}, Proposition~\ref{prop:shrink} shows the shrinkage of the cluster set, which in turn determines where the Quadratic Positivity Condition holds. 
Combining these ingredients, in Appendix~\ref{app:pwl2}, we prove Theorem~\ref{thm:pwl-detail}, thereby proving Theorem~\ref{thm:pwl}. 
Technical lemmas are provided in Appendix~\ref{app:tech-tau}.

\subsection{Proof of Proposition~\ref{prop:tau-limit}}
\label{app:tau}

In this section, we prove Proposition~\ref{prop:tau-limit} which bounds the Hessian of the solution depending on $m,\beta,\lambda$ and a term that depends on the training data and residuals. 
Recall that we let $\alpha=(\bl,m)$. We also let $\hat \alpha=(\bl,\mt)$.

\begin{proof}\textit{(Proof of Proposition~\ref{prop:tau-limit})}
By Lemma~\ref{lem:tau-pointwise}, we have that for every $(x,\theta)$, 
$$
\lim_{\tau\to+\infty} \sigma^{m,\tau}(x,\theta) \rho^\ast_{\hat \alpha}(\theta)=  \sigma^m(x,\theta)\rho^\ast_\alpha(\theta). 
$$
With Lemma~\ref{lem:tau-gauss-tail}, there exists $C>0$ independent of $(\tau,x)$ such that
$$
|\sigma^{m,\tau}(x,\theta) \rho^\ast_{\hat \alpha}(\theta)|\leq C \exp(-\frac{\beta\lambda }{2}\|\theta\|^2),\quad \forall \theta. 
$$
Thus, the Dominated Convergence Theorem gives that
$$
\lim_{\tau\to+\infty} h^\ast_{\hat \alpha}(x)=\lim_{\tau\to+\infty}\int \sigma^{m,\tau}(x,\theta) \rho^\ast_{\hat \alpha}(\theta) d\theta =
\int \sigma^m(x,\theta)\rho^\ast_\alpha(\theta) d\theta
=h^\ast_{\alpha}(x), \quad \forall x. 
$$

Let $u=x^\top w^{m,\tau}+b$ and $c_{\mt}=S_\tau^{-1}(m^2)$. 
Notice that 
$$
a^{m,\tau}\nabla_x^2\left( ( x^\top w^{m,\tau}+b )_+^{m,\tau} \right)  = \begin{dcases}
    a^{m,\tau} w^{m,\tau}(w^{m,\tau})^\top \frac{\tau e^{\tau u}}{(1+e^{\tau u})^2} ,& u<c_{\mt}\\
    a^{\mt} w^{\mt} (w^{\mt})^\top  \phi''_{\mt}(u) , & u\geq c_{\mt}. 
\end{dcases}
$$
Recall from Appendix~\ref{app:relu} that $\|\phi''_{\mt}\|_\infty\leq m^{-2}$. 
Meanwhile, $\|S''_\tau\|_\infty \leq \tau/4$. 
By Lemma~\ref{lem:tau-gauss-tail}, there exists $C>0$ such that for every $(x,\theta)$: 
$$
\left\|a^{m,\tau} \nabla_x^2\left( (u)_+^{m,\tau} \right) \rho^\ast_{\hat \alpha}(\theta) \right\| \leq C |a|\|w\|^2  \exp(-\frac{\beta \lambda}{2}\|\theta\|^2).
$$
Note that the right-hand side above is absolutely integrable in $\theta$.
Since $\sigma^{\mt}(x,\theta) \rho^\ast_{\hat \alpha}(\theta)$ is integrable in $\theta$ for every $x$ by Lemma~\ref{lem:tau-gauss-tail} and is $C^2$ differentiable in $x$ for every $\theta$, 
we have that
{\small
\begin{equation}\label{eq:tau-limit}
\begin{aligned}
    & \nabla_x^2 h^\ast_{\hat \alpha}(x)\\
    =&\nabla_x^2 \int \sigma^{\mt}(x,\theta) \rho^\ast_{\hat \alpha}(\theta)d\theta \\
    =&\int \mathbb{I}_{u<c_{\mt}} a^{m,\tau} w^{m,\tau}(w^{m,\tau})^\top \frac{\tau e^{\tau u}}{(1+e^{\tau u})^2} \rho^\ast_{\hat \alpha}(\theta) d\theta + \int \mathbb{I}_{u\geq c_{m,\tau}} a^{\mt} w^{\mt} (w^{\mt})^\top  \phi''_{\mt}(u) \rho^\ast_{\hat \alpha}(\theta) d\theta. 
\end{aligned}
\end{equation} 
}

For the first term of the right-hand side of~\eqref{eq:tau-limit}, consider the change of variable $\theta =(a,w,b) \mapsto (a,w,u')$ with $u'=\tau u = \tau(x^\top w^{m,\tau}+b)$ and we have that
\begin{align*}
    & \int \mathbb{I}_{u<c_{\mt}} a^{m,\tau} w^{m,\tau}(w^{m,\tau})^\top \frac{\tau e^{\tau u}}{(1+e^{\tau u})^2} \rho^\ast_{\hat \alpha}(\theta) \, d\theta\\
    =& \int \mathbb{I}_{u'<\tau c_{\mt}}\  a^{m,\tau} w^{m,\tau}(w^{m,\tau})^\top \frac{e^{u'}}{(1+e^{u'})^2} \rho_{\hat \alpha}^\ast(a, w, \frac{u'}{\tau}-x^\top w^{m,\tau}) \,da \, dw \, du'.
\end{align*}
Notice that for a fixed $(a,w,u')$, 
\begin{align*}
    &|\rho_{\hat \alpha}^\ast(a, w, \frac{u'}{\tau}-x^\top w^{m,\tau})  - \rho_\alpha^\ast (a,w,-x^\top w^m)| \\
    \leq & |\rho_{\hat \alpha}^\ast(a, w, \frac{u'}{\tau}-x^\top w^{m,\tau})- \rho_{\hat \alpha}^\ast(a, w,-x^\top w^m)|+|\rho_{\hat \alpha}^\ast(a, w,-x^\top w^m)-\rho_\alpha^\ast (a,w,-x^\top w^m)|. 
\end{align*}
By Lemma~\ref{lem:tau-gauss-tail}, $\rho_{\hat \alpha}^\ast$ is Lipschitz continuous with a Lipschitz constant independent of $\tau$ near $(a,w,-x^\top w^m)$. 
This implies the first term vanishes as $\tau \to \infty$. 
The second term also vanishes according to Lemma~\ref{lem:tau-pointwise}. 
Therefore, we have that for any fixed $(a,w,u',m,x)$, 
\begin{equation*}
\begin{aligned}
&\lim_{\tau\to+\infty} \mathbb{I}_{u'<\tau c_{\mt}}\  a^{m,\tau} w^{m,\tau}(w^{m,\tau})^\top \frac{e^{u'}}{(1+e^{u'})^2} \rho_{\hat \alpha}^\ast(a, w, \frac{u'}{\tau}-x^\top w^{m,\tau}) 
\\= &a^m w^m (w^m)^\top \frac{e^{u'}}{(1+e^{u'})^2}\rho_\alpha^\ast(a,w,-x^\top w^m). 
\end{aligned}
\end{equation*}
By Lemma~\ref{lem:tau-gauss-tail}, there exist constants $C,C'$ independent of $\tau$ such that
{\small
$$
\left\| \mathbb{I}_{u'<\tau c_{\mt}}\  a^{m,\tau} w^{m,\tau}(w^{m,\tau})^\top \frac{e^{u'}}{(1+e^{u'})^2} \rho_{\hat \alpha}^\ast(a, w, \frac{u'}{\tau}-x^\top w^{m,\tau}) \right\| \leq C \frac{e^{u'}}{(1+e^{u'})^2} \exp(-C'(a^2+\|w||^2)),
$$
}
where the right-hand side is absolutely integrable in $(a,w,u')$. 
Then with the Dominated Convergence Theorem we have that 
\begin{align*}
    &\lim_{\tau\to \infty} \int \mathbb{I}_{u<c_{\mt}} a^{m,\tau} w^{m,\tau}(w^{m,\tau})^\top \frac{\tau e^{\tau u}}{(1+e^{\tau u})^2} \rho^\ast_{\hat \alpha}(\theta) \, d\theta\\
    =& \int a^m w^m (w^m)^\top \frac{e^{u'}}{(1+e^{u'})^2}\rho_\alpha^\ast(a,w,-x^\top w^m) \, da \, dw \, du'\\
    =& \int a^m w^m (w^m)^\top \rho_\alpha^\ast(a,w,-x^\top w^m) \, da \, dw.
\end{align*}
where the second equality comes from the Fubini's theorem.

For the second term of the right-hand side of~\eqref{eq:tau-limit}, 
recall that as $\tau \to+\infty$, 
$\phi''_{\mt}$ uniformly converges to $\phi''_m$ 
on any compact set in $(m^2,+\infty)$.
Then we have that, 
for all $\theta \in \br^{d+2}\setminus \set{\theta\colon x^\top w^m+b=m^2}$, 
and hence for a.e.\ $\theta$, 
{\small
$$
\lim_{\tau\to+\infty}\mathbb{I}_{u\geq c_{m,\tau}} a^{\mt} w^{\mt} (w^{\mt})^\top  \phi''_{\mt}(u) \rho^\ast_{\hat \alpha}(\theta)
=\mathbb{I}_{x^\top w^m +b \geq m^2} a^m w^m(w^m)^\top \phi''_m(x^\top w^m +b) \rho_\alpha^\ast(\theta). 
$$
}
Similarly, with Lemma~\ref{lem:tau-gauss-tail}, there exist $C,C'$ independent of $\tau$ such that
$$
\left\|\mathbb{I}_{u\geq c_{m,\tau}} a^{\mt} w^{\mt} (w^{\mt})^\top  \phi''_{\mt}(u) \rho^\ast_{\hat \alpha}(\theta)\right\| \leq C \exp(-C'\|\theta\|^2). 
$$
Using the Dominated Convergence Theorem again, we have that 
{\small
$$
\lim_{\tau\to+\infty}\int \mathbb{I}_{u\geq c_{m,\tau}} a^{m,\tau} \nabla_x^2\left( \phi_{m,\tau} (u) \right) \rho^\ast_{\hat \alpha}(\theta) d\theta
=
\int \mathbb{I}_{x^\top w^m +b \geq m^2} a^m w^m(w^m)^\top \phi''_m(x^\top w^m +b) \rho_\alpha^\ast(\theta) d\theta. 
$$
}
Therefore, we have that
\begin{equation}\label{eq:Dx0}
\begin{aligned}
\lim_{\tau\to+\infty}\nabla_x^2 h^\ast_{\hat \alpha}(x) =& 
\int a^m w^m (w^m)^\top \rho_\alpha^\ast(a,w,-x^\top w^m) \, da\, dw\\
&+
\int \mathbb{I}_{x^\top w^m +b \geq m^2} a^m w^m(w^m)^\top \phi''_m(x^\top w^m +b) \rho_\alpha^\ast(\theta) d\theta. 
\end{aligned}
\end{equation}

Now we examine the Hessian of $h_\alpha^\ast(\cdot)$. 
Notice that $h_\alpha^\ast(\cdot)$ is continuous and thus its distributional Hessian exists.
Let $D_x^2$ denote the distributional Hessian operator. 
Notice that, for any test function $\eta(x)$, 
\begin{align*}
    \left\langle D_x^2 h_\alpha^\ast(\cdot),  \eta(\cdot)\right\rangle 
    &= \int D_x^2 \eta(x) \left( \int \sigma^m(x,\theta) \rho^\ast_\alpha(\theta) d\theta\right) dx\\
    &=  \int\left( \int D_x^2 \eta(x) \cdot  \sigma^m(x,\theta) dx \right) \rho^\ast_\alpha(\theta) d\theta\\
    &=  \int\left( \int \eta(x) \cdot  D_x^2  \sigma^m(x,\theta) dx \right) \rho^\ast_\alpha(\theta) d\theta,
\end{align*}
where the second equality is by Fubini's Theorem. 
Note that, as distributions in $x$, 
$$
D^2_x( \sigma^m(x,\theta) ) = 
a^m w^m(w^m)^\top \Big( \delta(x^\top w^m + b) + \mathbb{I}_{x^\top w^m + b\geq m^2}\phi''_m(x^\top w^m + b) \Big).
$$
Then we have that
\begin{align*}
\left\langle D_x^2 h_\alpha^\ast(\cdot),  \eta(\cdot)\right\rangle &= 
\int\left( \int a^mw^m(w^m)^\top \rho^\ast_\alpha(a,w,-x^\top w^m) dadw\right) \eta(x) dx    \\
&\quad+ \int \left(\int a^m w^m(w^m)^\top \phi_m''(x^\top w^m+b) \mathbb{I}_{x^\top w^m+b\geq m^2} \rho^\ast_\alpha(\theta) d\theta \right) \eta(x) dx,
\end{align*}
which means that, 
\begin{equation}\label{eq:Dx}
\begin{aligned}
    D_x^2 h^\ast_\alpha(x)
    =&
    \int a^m w^m(w^m)^\top \rho^\ast_\alpha(a,w,-x^\top w^m) \, da \, dw \\ 
    &+ \int a^m w^m(w^m)^\top \mathbb{I}_{x^\top w^m + b\geq m^2}\phi''_m(x^\top w^m + b) \rho^\ast_\alpha(\theta) d\theta. 
\end{aligned}
\end{equation}
By Lemma~\ref{lem:tau-pointwise}, the right-hand side of \eqref{eq:Dx} is an ordinary continuous function in $x$.
Thus, $h_\alpha^\ast\in C^2(\br^d)$. Comparing~\eqref{eq:Dx0} against \eqref{eq:Dx} gives that
$$
\lim_{\tau\to+\infty}\nabla_x^2 h^\ast_{\hat \alpha}(x) = \nabla_x^2 h^\ast_{\alpha}(x). 
$$

Next, we derive the Hessian bound~\eqref{eq:hessian-quantity}.
Recall from Appendix~\ref{app:relu} that, 
$|a^m|\le m$ and $\|\phi''_m\|_{\infty}\le m^{-2}$. 
Thus, for the second term of the right-hand side of \eqref{eq:Dx}, 
\begin{align*}
\left\|\int a^m w^m(w^m)^\top \mathbb{I}_{x^\top w^m + b\geq m^2}\phi''_m(x^\top w^m + b) \rho^\ast_\alpha(\theta) \,d\theta \right\| 
& \leq   \frac{1}{m}  \int \|w\|^2 \rho^\ast_\alpha(\theta) \,d\theta\\
&\leq C (m\lambda)^{-1}+C'm^{-1},
\end{align*}
where the second inequality comes from Lemma~\ref{lem:Mbound} and $C,C'>0$ are independent of $(\alpha,x)$. 
For the first term of the right-hand side of \eqref{eq:Dx}, 
let 
$$
Z^\ast_\alpha=\int \exp\Big\{-\beta \Big( \sum_{j=1}^M R^\alpha_j \sigma^m(x_j,\theta) +\frac{\lambda}{2}\|\theta\|^2  \Big)  \Big\} d\theta. 
$$
We have that
\begin{align*}
    &\int a^m w^m (w^m)^\top \rho^\ast_\alpha(a,w,-x^\top w^m) \,da\, dw \\
    =&(Z_\alpha^\ast)^{-1} \int a^m w^m(w^m)^\top 
    \cdot \exp\Big\{-\beta\Big[ \sum_{j=1}^M R^\alpha_j\cdot a^m(x_j^\top w^m-x^\top w^m)_+^m\\
    &\qquad \qquad + \frac{\lambda}{2}(a^2+\|w\|^2+(x^\top w^m)^2) \Big] \Big\} \,da\,dw. 
\end{align*}
For all $m> \max_j\|x_j-x\|$, we have that
$$
\max_j | (x_j-x)^\top w^m | \leq \max_j\|x_j-x\| \cdot \|w^m\| \leq \max_j\|x_j-x\| \cdot m<m^2. 
$$
Recall that $(z)_+^m=(z)_+$ if $z\leq m^2$, and that, 
$w^m=\gamma_w w$ for some $\gamma_w\in [0,1]$. 
Then we have that 
{\small 
\begin{align*}
    & Z_\alpha^\ast\cdot \Big\|\int a^m w^m (w^m)^\top \rho^\ast_\alpha(a,w,-x^\top w^m) \,da\, dw \Big\|\\
    \le &   \int |a| \|w\|^2 \cdot \exp\curv{-\beta\left[ \sum_{j=1}^M R^\alpha_j\cdot a^m((x_j-x)^\top w^m)_+ + \frac{\lambda}{2}(a^2+\|w\|^2+|x^\top w^m|^2) \right]} \, da\,dw\\
    = &  \sum_{k=1}^{2\cp(\cx)} \int_{\br \times O^w_k} 
    |a| \|w\|^2 \cdot \exp\curv{-\beta\left[ \sum_{j\in J_k} R^\alpha_j\cdot a^m (x_j-x)^\top w^m + \frac{\lambda}{2}(a^2+\|w\|^2+|x^\top w^m|^2) \right]} \,da\,dw. 
\end{align*}
}
For $k=1,\ldots, 2\cp(\cx)$, 
with Tonelli's theorem we have that
\begin{align*}
    &\int_{\br \times O_k^w} |a| \|w\|^2 \exp\curv{-\beta\left[ \sum_{j\in J_k} R^\alpha_j\cdot a^m(x_j-x)^\top w^m + \frac{\lambda}{2}(a^2+\|w\|^2+|x^\top w^m|^2) \right]} dadw\\
    =& 
    \int_{\br \times O_k^w} |a| \|w\|^2 \cdot \exp\curv{-\frac{\lambda\beta}{2}\left[ 2a^m (Q_k^\alpha)^\top w^m + a^2+\|w\|^2+|x^\top w^m|^2 \right]} \,da\,dw\\
    =& \int_{O_k^w}\|w\|^2 \exp(-\frac{\lambda\beta}{2}(\|w\|^2+|x^\top w^m|^2)) \Big(\int_\br |a| \exp(-\frac{\lambda\beta}{2}(a^2+2a^m(Q_k^\alpha)^\top w^m)) \,da \Big) \,dw. 
\end{align*}
For the inner integral in $a$, we have that 
\begin{align*}
    &\int_\br |a| \exp(-\frac{\lambda\beta}{2}(a^2+2a^m(Q_k^\alpha)^\top w^m)) da\\
    \leq& \int_\br |a| \exp(-\frac{\lambda\beta}{2}(a^2-2|a||(Q_k^\alpha)^\top w^m|)) da\\
    =& 2\int_{0}^\infty a \exp(-\frac{\lambda\beta}{2}(a^2-2a|(Q_k^\alpha)^\top w^m|)) da\\
    =& \frac{2}{\lambda \beta} + \sqrt{\frac{2\pi}{\lambda \beta}} \cdot \exp(\frac{\lambda \beta}{2}((Q_k^\alpha)^\top w^m)^2)\cdot |(Q_k^\alpha)^\top w^m|\cdot \Big(1+\operatorname{erf}(|(Q_k^\alpha)^\top w^m|\sqrt{\frac{\lambda \beta}{2}})\Big)\\
    \le & \frac{2}{\lambda \beta} + \sqrt{\frac{2\pi}{\lambda \beta}} \cdot \exp(\frac{\lambda \beta}{2}((Q_k^\alpha)^\top w^m)^2)\cdot 2 |(Q_k^\alpha)^\top w^m|. 
\end{align*}
Then we have that 
{\small
\begin{equation*}
\begin{aligned}
    &\int_{\br \times O_k^w} |a| \|w\|^2  \exp\curv{-\beta\left[ \sum_{j\in J_k} R^\alpha_j a^m(x_j-x)^\top w^m + \frac{\lambda}{2}(a^2+\|w\|^2+|x^\top w^m|^2) \right]} da\, dw\\
    \le & \frac{2}{\lambda \beta} \int_{O_k^w}\|w\|^2 \exp(-\frac{\lambda\beta}{2}(\|w\|^2+|x^\top w^m|^2)) dw \\
        &\quad\quad + 2\sqrt{\frac{2\pi}{\lam\beta}} \int_{O_k^w} \|w\|^2 |(Q_k^\alpha)^\top w^m| 
    \exp\curv{-\frac{\lambda\beta}{2}\Big(\|w\|^2 +|x^\top w^m|^2 - ((Q_k^\alpha)^\top w^m)^2\Big)} dw\\
    \lesssim & 
    (\lambda \beta)^{-\frac{d+4}{2}} + (\lambda \beta)^{-\frac{1}{2}} \int_{O_k^w} \|w\|^2 |(Q_k^\alpha)^\top w^m| 
    \exp\curv{-\frac{\lambda\beta}{2}\Big(\|w\|^2 +|x^\top w^m|^2 - ((Q_k^\alpha)^\top w^m)^2\Big)} dw,
\end{aligned}
\end{equation*}
}
where the last inequality comes from the fact that 
$$
\frac{2}{\lambda \beta} \int_{O_k^w}\|w\|^2 \exp(-\frac{\lambda\beta}{2}(\|w\|^2+|x^\top w^m|^2)) dw\leq 
\frac{2}{\lambda \beta} \int_{\br^d}\|w\|^2 \exp(-\frac{\lambda\beta}{2}\|w\|^2) dw 
\lesssim (\lambda \beta)^{-\frac{d+4}{2}}. 
$$

By the above and with Lemma~\ref{lem:hess-Z}, we have that 
$$
\|\nabla_x^2 h^\ast_{\alpha}(x)\| 
\lesssim 
m^{-1}(1+\lambda^{-1})
+\beta^{-1}\lambda^{-\frac{d+6}{4}}(1+\sqrt{\lambda})^{\frac{d+2}{2}} 
+ 
\beta^{\frac{d+1}{2}}\lambda^{\frac{d}{4}}(1+\sqrt{\lambda})^{\frac{d+2}{2}}\mathcal{I}^\alpha(x),
$$
where 
$$
\mathcal{I}^\alpha(x)=\sum_{k=1}^{2\cp(\cx)} \int_{O_k^w} \|w\|^2 |(Q_k^\alpha)^\top w^m| \exp\curv{-\frac{\lambda\beta}{2}\Big(\|w\|^2 +|x^\top w^m|^2 - ((Q_k^\alpha)^\top w^m)^2\Big)} dw. 
$$
This completes the proof. 
\end{proof}

\subsection{Proof of Proposition~\ref{prop:qpc}}
\label{app:qpc}

In this section, we prove Proposition~\ref{prop:qpc} which bounds the Hessian $\|\nabla_x^2 h^\ast_\alpha(x)\|$ under the Quadratic Positivity Condition.

\begin{proof}\textit{(Proof of Proposition~\ref{prop:qpc})}
    By assumption, for any $k\in \set{1,\cdots,2\cp(\cx)}$,
    we have that
    $$
    f_k^\alpha(w,b) \geq \delta (\|w\|^2+b^2), \quad \forall (w,b)\in G_x \cap O_k. 
    $$
    Therefore, 
    $$
    \|w\|^2+|x^\top w|^2-|(Q_k^\alpha)^\top w|^2 \geq \delta \|w\|^2, \quad \forall w \in O_k^w. 
    $$
    Note that $w^m=\gamma_w w$ for some $\gamma_w\in [0,1]$ and that $\delta \in (0,1/2]$. 
    Thus, for every $w\in O_k^w$, 
    \begin{align*}
    \|w\|^2+|x^\top w^m|^2 -|(Q_k^\alpha)^\top w^m|^2 
    &=(1-\gamma_w^2)\|w\|^2 + \gamma_w^2(\|w\|^2+|x^\top w|^2-|(Q_k^\alpha)^\top w|^2)\\
    & \geq (1-\gamma^2_w(1- \delta))\|w\|^2\\
    & \geq \delta\|w\|^2. 
    \end{align*}
    Notice that $|(Q_k^\alpha)^\top w^m|^2\leq (1-  \delta)\|w\|^2+|x^\top w^m|^2\leq \|w\|^2(1+\|x\|^2)$. 
    Then we have that
    \begin{align*}
    &\int_{O_k^w} \|w\|^2 |(Q^\alpha_k)^\top w^m| \exp\curv{-\frac{\lambda\beta}{2}\Big( \|w\|^2+|x^\top w^m|^2-|(Q^\alpha_k)^\top w^m|^2\Big)} dw    \\
    \leq &\sqrt{1+\|x\|^2} \int_{\br^d} \|w\|^3 \exp\curv{-\frac{\lambda\beta \delta}{2}\|w\|^2} dw\\
    \lesssim &   \sqrt{1+\|x\|^2}  
    (\lambda \beta \delta)^{-\frac{d+3}{2}}. 
    \end{align*}
    It follows that
    $$
    \mathcal{I}^\alpha(x) \lesssim  \sqrt{1+\|x\|^2}  (\lambda \beta \delta)^{-\frac{d+3}{2}}.
    $$
    With Proposition~\ref{prop:tau-limit}, we have that
    $$
    \|\nabla_x^2 h^\ast_{\alpha}(x)\|
    \lesssim  
    m^{-1}(1+\lambda^{-1})
        +\beta^{-1}\lambda^{-\frac{d+6}{4}}(1+\sqrt{\lambda})^{\frac{d+2}{2}}
        \left(1
        + 
        \sqrt{1+\|x\|^2}\cdot \delta^{-\frac{d+3}{2}}
        \right). 
    $$
    This completes the proof. 
\end{proof}

\subsection{Proof of Proposition~\ref{prop:shrink}}
\label{app:shrink}

In this section, we prove Proposition~\ref{prop:shrink} which shows the shrinkage of the cluster set. 

\begin{proof}\textit{(Proof of Proposition~\ref{prop:shrink})}

\textbf{Step 1.} We begin with fixing the topology of the set 
$$
\hat \Omega^\ast = \bigcup_{k} \hat O_k \cap \hat \Omega^\ast_k.
$$ 
We claim that, for every $k$ in $\set{1,\ldots,2\cp(\cx)}$, $\hat O_k \cap \hat \Omega^\ast_k$ has at most two connected components. 
Note that, when $\|P_k^\ast\|<1$, 
$\hat \Omega^\ast_k=\varnothing$ and, when $\|P_k^\ast\|=1$, $\hat \Omega^\ast_k$ consists of two points. 
The claim thus holds in both cases. 
Assume that $\|P_k^\ast\|>1$. 
Under a suitable basis $\tilde z$, $f_k^\ast$ can be written as
$$
f_k^\ast (\tilde{z})=(1-\|P_k^\ast\|^2) \tilde z_1^2 +\sum_{i>1}\tilde z_i^2. 
$$
Then $\hat \Omega^\ast_k$ has two connected components: 
\begin{equation*}
    C_+ = \Big\{\tilde{z}\in \bs^{d} \colon \tilde z_1 \geq \sqrt{\frac{1}{\|P_k^\ast\|^2-1}\sum_{i>1} \tilde z_i^2}\Big\},
    \  
    C_- = \Big\{\tilde{z}\in \bs^{d} \colon \tilde z_1 \leq -\sqrt{\frac{1}{\|P_k^\ast\|^2-1}\sum_{i>1} \tilde z_i^2}\Big\}. 
\end{equation*}
Clearly, $C_+$ is contained by an open hemisphere on $\bs^d$. 
Select any $z\in C_+$ and consider the standard gnomonic projection $\iota$ from $C_+$ to the tangent space $T_z\bs^d$, viewed as an embedded submanifold of $\br^{d+1}$. 
Recall that $O_k$ is a cone and $\hat O_k$ is its spherical cross-section.
We have that
$$
\iota(\hat O_k\cap C_+) = T_z\bs^d\cap O_k \cap \Big\{\tilde{z}\in \br^{d+1} \colon \tilde z_1 \geq \sqrt{\frac{1}{\|P_k^\ast\|^2-1}\sum_{i>1} \tilde z_i^2}\Big\}. 
$$
The set on the right-hand side is either empty or connected, as it is the intersection of three convex sets. 
As $\iota$ is a homeomorphism, $\hat O_k\cap C_+$ is either empty or connected. 
Similar arguments apply to $\hat O_k \cap C_-$. 
Thus, $\hat O_k \cap \hat \Omega^\ast_k$ has at most two connected components.

We next show that the case of two components of $\hat O_k \cap \hat \Omega^\ast_k$ occurs if and only if both components of $\hat{\Omega}_k^\ast$ intersect the relative boundary of $\hat{O}_k$ in $\bs^d$. 
Note that $f_k^\ast(z)=1-(z^\top P_k^\ast)^2$ for $z\in \bs^d$. 
Then the two components of $\hat \Omega_k^\ast$ can be rewritten as
$$
C_+ = \set{z\in \bs^d\colon z^\top P_k^\ast \ge1 },\quad 
C_- = \set{z\in \bs^d\colon z^\top P_k^\ast \le -1 }. 
$$
For contradiction, 
assume that $C_+$ is in the relative interior of $\hat O_k$ in $\bs^d$. 
Then $\|P_k^\ast\|$ has to be no smaller than one. 
Suppose that $\hat O_k$ corresponds to the binary activation pattern $A$ and assume without loss of generality that $A(x_1)=1$. 
Notice that $\hat O_k$ is contained in the closed hemisphere
$E=\{(w,b)\in \bs^d \colon w^\top x_1 +b\geq0\}.$
By assumption, 
$C_+$ lies in the relative interior of $\hat O_k$ and thus also in the relative interior of $E$. 
By the reflection symmetry, we have that $C_-=-C_+$ lies in the relative interior of $-E$. 
In particular, we have that $C_-\cap E=\varnothing$. 
Consequently, $C_-\cap \hat O_k=\varnothing$, which yields a contradiction. 
Therefore, $\hat O_k \cap \hat \Omega^\ast_k$ has two components if and only if both components of $\hat{\Omega}_k^\ast$ intersect the relative boundary of $\hat{O}_k$.

Consider any two adjacent sectors $\hat{O}_k, \hat{O}_l$ with a non-empty intersection $\hat O_{kl}=\hat O_k \cap \hat{O}_l$. 
Recalling the definition of $f_k^{\alpha}$ from~\eqref{eq:f}, we have that 
$$
f_{k}^{\alpha_n}(w,b)= f_l^{\alpha_n}(w,b),\quad \forall (w,b)\in \hat O_{kl}. 
$$ 
As $f_k^{\alpha_n}$ converges pointwise to $f_k^\ast$,
$$
f_{k}^{\ast}(w,b)= f_l^{\ast}(w,b),\quad \forall (w,b)\in \hat O_{kl}. 
$$
Thus, their non-positive sets agree on $\hat O_{kl}$: 
\begin{equation}\label{eq:adjacent}
    \hat \Omega_k^\ast \cap \hat O_{kl}=\hat \Omega_\ell^\ast \cap \hat O_{kl}. 
\end{equation}
This means that whenever $\hat\Omega_k^\ast \cap \hat O_k$ has two components, each component must touch the relative boundary of $\hat O_k$ and will be merged with components in adjacent sectors.

Now we show that $\hat \Omega^\ast$ has at most $2\cp(\cx)-1$ connected components. 
First, consider the specific closed sector 
\begin{equation}\label{eq:one-side-sector}
    \hat O_{k_0} =  \bigcap_{j=1}^M \set{(w,b)\in \bs^{d}\colon x_j^\top w+b \le 0}.
\end{equation}
We have that $J_{k_0}=\varnothing$, $f_{k_0}^{\alpha_n}(z)=\|z\|^2$ for every $n$. 
Thus, $f_{k_0}^\ast(z)=\|z\|^2$ and $\hat \Omega_{k_0}^\ast=\varnothing$. 
Thus, the sector $\hat O_{k_0}$ does not intersect any connected component of $\hat \Omega^\ast$. 
Now suppose that $\hat \Omega^\ast$ has $N$ connected components, denoted by $\set{C_i}_{i=1}^N$. 
For $i=1,\cdots, N$, consider the number of sectors intersecting $C_i$: 
$$
r_i = \# \set{ k\colon \hat O_k \cap C_i \neq \varnothing } \geq 1. 
$$
Suppose that there are $I$ connected components such that $r_i = 1$. 
Because a sector can contain two components only if both touch its relative boundary, there are precisely $I$ sectors that contain one component in their interior and intersect no other components on their relative boundaries.
For the remaining $N-I$ components, we have that
$r_i \geq 2$.
Thus, 
$$
2(N-I) \leq \sum_{i\colon r_i\geq 2} r_i. 
$$
As there are already $I+1$ sectors, including $\hat O_{k_0}$, that do not intersect any component on their boundaries,
these $N-I$ components with $r_i \geq 2$ can only intersect the remaining $2\cp(\cx)-I-1$ sectors.
However, each of these sectors can intersect at most two components. Thus,
$$
\sum_{i\colon r_i\geq 2} r_i \leq 2(2\cp(\cx)-I-1). 
$$
Therefore, we have that $N-I\leq 2\cp(\cx)-I-1$. 
Thus, the total number of components is bounded as follows:
$$
N=I +(N-I)\leq I +(2\cp(\cx)-I-1)=2\cp(\cx)-1. 
$$

\textbf{Step 2.} 
Next, we show that 
$$
s(\hat \Omega^\ast)=0,
$$
where $s(A)$ denotes the spherical measure of a measurable set $A\subset \bs^d$. 
Fix an arbitrary $k$ from $\set{1,\ldots,2\cp(\cx)}$. 
With $z=(w,b)$, define
$$
\Theta^\alpha_k = \set{a\colon |a|<m} \times \Big(O_k\cap \Omega_k^\alpha \cap \set{z\colon \|z\|<m}\Big) \subset \br^{d+2}. 
$$
Assume that $m>1+\max_j\|x_j\|$. 
Then, given that $|b|<m$, 
$$
\max_j |x_j^\top w^m + b|\leq (1+ \max_j \|x_j\|)m \leq m^2.
$$
Then for all $\theta \in \Theta_k^\alpha$, 
$$
a^m(x_j^\top w^m + b)_+^m=a^m(x_j^\top w^m + b)_+=a(x_j^\top w + b)_+, \quad j=1,\ldots,M. 
$$
Let
$$
Z_\alpha^\ast = \int\exp\Big\{ - \beta \Big[ \sum_{j=1}^M R^\alpha_j a^m(x_j^\top w^m+b)_+^m +\frac{\lambda}{2}\|\theta\|^2\Big] \Big\} d\theta. 
$$
We have that
\begin{align*}
    Z_\alpha^\ast & \geq  \int_{\Theta^\alpha_k} \exp\Big\{ - \beta \Big[ \sum_{j\in J_k} R^\alpha_j a (x_j^\top w+b) +\frac{\lambda}{2}\|\theta\|^2\Big] \Big\} d\theta\\
    &= \int_{\Theta^\alpha_k} \exp\Big\{ - \frac{\lambda\beta}{2} \left[ 2a\cdot 
    (P_k^\alpha)^\top z +\|\theta\|^2\right] \Big\} d\theta\\
    &= \int_{\Theta^\alpha_k} \exp\Big\{ - \frac{\lambda\beta}{2} \Big[ \Big(a+(P_k^\alpha)^\top z\Big)^2 + f_k^\alpha(w,b)\Big] \Big\} d\theta\\
    &\ge \int_{\Theta^\alpha_k} \exp\Big\{ - \frac{\lambda\beta}{2}  \Big(a+(P_k^\alpha)^\top z\Big)^2 \Big\} d\theta,
\end{align*}
where we used that $f_k^\alpha(w,b) \leq 0$ for all $(w,b)\in \Omega_k^\alpha$. 
With change of variables 
$\theta \mapsto (a, \hat z=z/\|z\|, r=\|z\|)$, we have that
\begin{align*}
    Z_\alpha^\ast &\geq \int_{\Theta^\alpha_k} \exp\curv{ - \frac{\lambda\beta}{2} \Big( a+(P_k^\alpha)^\top z\Big)^2 } d\theta\\
    &=\int_{0}^m \int_{\hat O_k\cap \hat \Omega_k^\alpha} \int_{-m}^m r^d
    \exp\curv{ - \frac{\lambda\beta}{2}\cdot (a+(P_k^\alpha)^\top \hat z r)^2 }\, da\,d s(\hat z)\, dr\\
    &\geq \int_{0}^{\sqrt{m}} \int_{\hat O_k\cap \hat \Omega_k^\alpha} \int_{-m}^m r^d
    \exp\curv{ - \frac{\lambda\beta}{2}\cdot (a+(P_k^\alpha)^\top \hat z r)^2 }\, da\, d s(\hat z)\, dr,
\end{align*}
where $s(\hat z)$ denotes the spherical measure on $\bs^d$.

By Lemma~\ref{lem:hess-R}, there exists $C_1$ independent of $\alpha$ such that $\sum_{j} (R^\alpha_j)^2 \leq C_1 \lambda$. 
Then with Cauchy-Schwarz inequality, 
\begin{align*}
    \|P_k^\alpha\|^2 = \Big\| \frac{1}{\lambda}\sum_{j\in J_k} R_j^\alpha  x_j \Big\|^2 + \Big( \frac{1}{\lambda}\sum_{j\in J_k} R_j^\alpha \Big)^2 
    \leq \frac{1}{\lambda^2}\cdot C_1\lambda \cdot \Big(\sum_{j=1}^M\|x_j\|^2+M\Big).
\end{align*}
Thus, there exists $C_2>0$ independent of $\alpha$ such that
$$
\|P^\alpha_k\|< \frac{C_2}{\sqrt{\lambda}}. 
$$
Now assume that $m$ satisfies
$$
m > \sqrt{m} \cdot \frac{C_2}{\sqrt{\lambda}} + 1.
$$
Note that this holds if 
$$
m > \frac{1}{4}\left( \frac{C_2}{\sqrt{\lambda}}+ \sqrt{\frac{C_2^2}{\lambda}+4} \right)^2. 
$$
We have that for $r\in [0,\sqrt{m}]$, 
$$
1+|(P^\alpha_k)^\top \hat z r| \leq 1+ \|P^\alpha_k\|r \leq 1+ \frac{C_2}{\sqrt{\lambda}}\sqrt{m}<m. 
$$
Thus,
$$
m+(P^\alpha_k)^\top \hat z r \geq m-|(P^\alpha_k)^\top \hat z r|\geq 1,\quad 
-m+(P^\alpha_k)^\top \hat z r\leq 0.
$$
Then we have that
\begin{align*}
\int_{-m}^m \exp\curv{ - \frac{\lambda\beta}{2}\cdot (a+(P^\alpha_k)^\top \hat z r)^2 }da &= \int_{-m+(P^\alpha_k)^\top \hat z r}^{m+(P^\alpha_k)^\top \hat z r} \exp\curv{ - \frac{\lambda\beta}{2} u^2 }du\\
&\geq \int _0^1 \exp\curv{ - \frac{\lambda\beta}{2} u^2 }du\\
&= \sqrt{\frac{\pi}{2\lambda \beta}}  \operatorname{erf}(\sqrt{\frac{\lambda \beta}{2}})\\
&\geq \frac{1}{2} \sqrt{\frac{\pi}{2\lambda \beta}},
\end{align*}
where the final inequality comes from that $\lambda \beta >1$ and that $\operatorname{erf}(\sqrt{1/2})>1/2$. 
Therefore, with Tonelli's Theorem, 
\begin{align*}
    Z_\alpha^\ast \geq \int_{0}^{\sqrt{m}} \int_{\hat O_k\cap \hat \Omega_k^\alpha}  \frac{1}{2} \sqrt{\frac{\pi}{2\lambda \beta}} \cdot r^d\,  d s(\hat z)\, dr 
    \geq C_3 \cdot s(\hat O_k\cap \hat \Omega_k^\alpha) \cdot \frac{m^{\frac{d+1}{2}}}{\sqrt{\lambda \beta}},
\end{align*}
where $C_3$ is independent of $\alpha$.

By Lemma~\ref{lem:hess-Z}, there exists $C_4$ independent of $\alpha$ such that
$$
Z_\alpha^\ast \leq \exp(C_4\beta+1)\cdot (\frac{8\pi}{\lambda \beta})^{d+2}. 
$$
Then we have that
$$
C_3 \cdot s(\hat O_k\cap \hat \Omega_k^\alpha) \cdot \frac{m^{\frac{d+1}{2}}}{\sqrt{\lambda \beta}} \leq Z_\alpha^\ast \leq \exp(C_4\beta+1)\cdot (\frac{8\pi}{\lambda \beta})^{d+2}.
$$
Since $\lambda \beta >1$, we have that 
$$
s(\hat O_k\cap \hat \Omega_k^\alpha) \leq C_5 \frac{\exp(C_4\beta+1)}{m^{\frac{d+1}{2}}},
$$
where $C_4,C_5$ are constants independent of $\alpha$. 
Assume that $m > \exp(2 C_4\beta +2)$. 
Then, 
\begin{equation}\label{eq:nonasymp}
s(\hat O_k\cap \hat \Omega_k^\alpha) \leq C_5 \frac{1}{m^{\frac{d}{2}}}. 
\end{equation}
As the analysis applies to any $k$, we have that
\begin{equation}\label{eq:shrink}
    \lim_{n\to \infty} \max_k s(\hat O_k\cap \hat \Omega_k^{\alpha_n})=0.
\end{equation}

Notice that, for every $z\in \bs^d$, 
\begin{equation}\label{eq:uniform}
    |f_k^{\alpha_n}(z)-f_k^\ast(z)|=|z^\top(P_k^{\alpha_n}(P_k^{\alpha_n})^\top - P_k^{\ast}(P_k^{\ast})^\top )z| \leq \|P_k^{\alpha_n}(P_k^{\alpha_n})^\top - P_k^{\ast}(P_k^{\ast})^\top\|. 
\end{equation}
Thus, by $\lim_{n}P_k^{\alpha_n}=P_k^\ast$, we have that
$f_k^{\alpha_n}$ converges uniformly to $f_k^\ast$ on $\bs^d$. 
Now assume that $s(\hat \Omega^\ast)>0$. 
Thus, there exists at least one $k$ such that
$s(\hat \Omega^\ast_k \cap \hat O_k)>0$. 
As $\hat \Omega^\ast_k \cap \hat O_k$ is closed, 
$$
\hat \Omega^\ast_k \cap \hat O_k = \operatorname{Int}_{\bs^d}(\hat \Omega^\ast_k \cap \hat O_k)\cup \partial_{\bs^d}( \hat \Omega^\ast_k \cap \hat O_k),
$$
where $\operatorname{Int}_{\bs^d}$ denotes the relative interior on $\bs^d$ and $\partial_{\bs^d}$ denotes the relative boundary on $\bs^d$. 
Notice that
$$
\partial_{\bs^d}(\hat \Omega^\ast_k \cap \hat O_k) \subset \partial_{\bs^d} \hat \Omega^\ast_k \cup \partial_{\bs^d} \hat O_k. 
$$
By the geometry of $\hat O_k$ and $\hat \Omega^\ast_k$, we have that $\partial_{\bs^d}(\hat \Omega^\ast_k \cap \hat O_k)$ has zero spherical measure. 
Therefore, $\hat \Omega^\ast_k \cap \hat O_k$ must have a non-empty relative interior. 
Fix any $z$ inside its relative interior. 
Then, there exists an $\varepsilon>0$ and a spherical neighborhood $U$ of $z$ such that 
$U\subset \hat \Omega^\ast_k \cap \hat O_k$ and $f_k^\ast(z)<-\varepsilon$ for every $z\in U$. 
By the uniform convergence, there exists $N$ such that for all $n>N$ and every $z\in U$, $f_k^{\alpha_n}(z) \leq -\varepsilon/2<0$. 
Thus, $U \subset \hat \Omega_k^{\alpha_n}\cap  \hat O_k$ for all $n>N$, which contradicts~\eqref{eq:shrink}. 
Hence, $s(\hat \Omega^\ast)=0$.

\textbf{Step 3.} 
Now we prove that $\hat \Omega^\ast$ consists of at most $2\cp(\cx)-1$ points. 
For this end, we show that for every $k$, 
$\hat \Omega_k^\ast \cap \hat O_k$ can be either two points, or one point, or empty. 
Note that, $f_k^\ast(z)=1- (z^\top P_k^\ast)^2$ when $z\in \bs^d$.
Let $\hat \Omega_k^\ast \cap \hat O_k=
A_+ \sqcup A_-$ where
$$
A_+=\set{z\in \hat O_k\colon z^\top P_k^\ast\geq 1}, \quad 
A_-=\set{z\in \hat O_k\colon z^\top P_k^\ast\leq -1}. 
$$
Recall that $s(A_+)=S(A_-)=0$. 
We claim that
$$
A_+=\set{z\in \hat O_k\colon z^\top P_k^\ast= 1}, \quad 
A_-=\set{z\in \hat O_k\colon z^\top P_k^\ast= -1}. 
$$
Indeed, if there exists $z\in \hat O_k$ such that $z^\top P_k^\ast>1$, there exists a spherical neighborhood $U$ of $z$ such that $(z')^\top P_k^\ast >1$ for every $z'\in U$.
Consequently, $s(\hat \Omega^\ast_k\cap \hat O_k)\ge s(U\cap \hat O_k)>0$, which yields a contradiction.

We show that $A_+$ must be a singleton. Assume that $A_+$ contains two distinct points $z_1\neq z_2$. Let $z_0=(z_1+z_2)/2$ and $\hat z_0=z_0/\|z_0\|$. 
Clearly, $\|z_0\|<1$ and $ \hat z_0\in \hat O_k$. 
However, 
$$
|(P_k^\ast)^\top \hat z_0|^2 
= \frac{1}{\|z_0\|^2}\left| \frac{(P_k^\ast)^\top z_1+ (P_k^\ast)^\top z_2}{2} \right|^2
=\frac{1}{\|z_0\|^2}>1.
$$
Similarly, there exists a spherical neighborhood $U$ of $\hat z_0$ such that 
$U\subset \hat \Omega_k^\ast$ and $s(\hat \Omega^\ast_k\cap \hat O_k)\ge s(U\cap \hat O_k)>0$, which yields a contradiction. 
Therefore, for every $k$, 
$\hat \Omega_k^\ast \cap \hat O_k$ can be either two points, or one point, or empty. 
As a consequence, 
$\hat \Omega^\ast=\cup_k \hat \Omega_k^\ast \cap \hat O_k$ is a finite set. 
Since $\hat \Omega^\ast$ has at most $2\cp(\cx)-1$ connected components, we have that $\hat \Omega^\ast$ consists of at most $2\cp(\cx)-1$ points.

\textbf{Step 4.} 
Finally, we show that $\hat \Omega^{\alpha_n}$ is contained by a spherical neighborhood of $\hat \Omega^\ast$ for sufficiently large $n$. 
For this end, we show that, for every $k$ and $\varepsilon$, there exists $N$ such that
$$
\hat O_k \cap \hat \Omega_k^{\alpha_n} \subset \set{z\in \bs^d \colon d_\angle (z,\hat O_k \cap \hat \Omega_k^{\ast})<\varepsilon},\quad \forall n>N. 
$$
Assume that this fails to hold. 
Then there exist $k$ and $\varepsilon_0>0$ and subsequence $\{n_l\}$ such that,
for every $l$,
there exists $z_l \in \hat O_k \cap \hat \Omega^{\alpha_{n_l}}_k$ 
with $d_{\angle}(z_l, \hat O_k\cap \hat \Omega^\ast_k)\ge \varepsilon_0$. 
Since $\hat O_k$ is compact, we can assume without loss of generality that $z^\ast =\lim_l z_l\in \hat O_k$ exists. 
As $z_l \in \hat \Omega_k^{\alpha_{n_l}}$, we have that $|z_l^\top P_k^{\alpha_{n_l}} |\geq 1$ for every $l$. 
Taking $n\to+\infty$ gives that $|(z^\ast)^\top P_k^{\ast} |\geq 1$. 
Therefore, 
$z^\ast \in \hat \Omega_k^\ast \cap \hat O_k$, 
which yields a contradiction.
It directly follows that, 
for every $\varepsilon$ there exists $N$ such that
$$
\hat \Omega^{\alpha_n} \subset \set{z\in \bs^d\colon d_\angle(z,\hat \Omega^\ast)<\varepsilon},\quad \forall n>N. 
$$
This completes the proof.
\end{proof}

\subsection{Proof of Theorem~\ref{thm:pwl-detail}}
\label{app:pwl2}

In this section, we prove Theorem~\ref{thm:pwl-detail}, thereby proving Theorem~\ref{thm:pwl}.

\begin{proof}(\textit{Proof of Theorem~\ref{thm:pwl-detail}})
We will show that
$$
\min_k \inf_{x\in \Gamma}\inf_{z\in \hat G_x \cap \hat O_k} f^\ast_k(z)>0. 
$$
Let 
$$
\hat G_\Gamma=\cup_{x\in \Gamma} \hat  G_x=\bigcup_{x\in \Gamma}\set{(w,b)\in \bs^d\colon w^\top x+b=0}.
$$
It suffices to show that
$$
\min_k \inf_{z\in \hat G_{\Gamma} \cap \hat O_k} f^\ast_k(z)>0. 
$$

Notice that for any $k$ with $\|P^\ast_k\|<1$, we have that
$$
\inf_{z\in \bs^d} f_k^\ast(z) \geq  1- \|P^\ast_k\|^2 >0. 
$$
Consider an arbitrary $k$ with $\|P^\ast_k\|\geq 1$ and thus $\hat \Omega^\ast_k \neq \varnothing$. 
By the assumption that $\Gamma \cap (\cup_k\pi_k)=\varnothing$, 
we have that $\hat G_\Gamma \cap \hat \Omega^\ast =\varnothing$.
Thus, 
$$
\hat G_\Gamma \cap \hat O_k \cap \hat \Omega^\ast_k =\varnothing. 
$$
As $\Gamma$ is compact, 
$\hat G_\Gamma$ is closed. 
Since $\hat G_\Gamma, \hat O_k, \hat \Omega^\ast_k$ are closed sets, we have that
\begin{equation}\label{eq:pwldetail1}
    \delta_k = d_\angle(\hat G_\Gamma\cap \hat O_k, \hat \Omega^\ast_k)>0,
\end{equation}
where $d_\angle(S,S')=\inf_{s\in S,s'\in S'} \arccos s^\top s'$ for non-empty subsets $S,S'\subset \bs^d$.

Let $\hat P_k^\ast= P_k^\ast / \|P_k^\ast\|\in \bs^{d}$. 
Define, for all $z\in \bs^d$, that
$$
\gamma_k(z) = \min\big\{\arccos (z^\top \hat P_k^\ast), \pi- \arccos (z^\top \hat P_k^\ast)\big\} \in [0,\frac{\pi}{2}]. 
$$
Then, for any $z\in \bs^{d}$, 
$$
f^\ast_k(z) =1-(( P_k^\ast)^\top z)^2=1-\| P_k^\ast\|^2\cos^2(\gamma_k(z)). 
$$
Thus, for $z\in \bs^{d}$, 
$$
z \in \hat \Omega_k^\ast \Longleftrightarrow \cos \gamma_k(z) \ge \frac{1}{\| P_k^\ast\|} \Longleftrightarrow \gamma_k(z) \le \bar \gamma_k \triangleq \arccos(\frac{1}{\| P_k^\ast\|}). 
$$
As $\|P_k^\ast\|\geq 1$, $\bar \gamma_k \in [0,\pi/2)$. 
By~\eqref{eq:pwldetail1}, 
$$
\frac{\pi}{2}\geq \gamma_k(z) \ge \bar \gamma_k+\delta_k, \quad \forall z\in \hat G_\Gamma \cap \hat O_k.
$$
As cosine is monotonic on $[0,\pi/2]$, we have that 
$$
f^\ast_k(z) \geq 1-\|P_k^\ast\|^2 \cos^2\Big(\bar \gamma_k+ \delta_k \Big),\quad \forall z\in \hat G_\Gamma \cap \hat O_k. 
$$
Since $\bar \gamma_k,\delta_k\in[0,\pi/2]$,
$$
\cos(\bar \gamma_k+\delta_k) = \cos(\bar \gamma_k)\cos(\delta_k)-\sin(\bar\gamma_k)\sin(\delta_k) \leq 
\cos(\bar\gamma_k)\cos(\delta_k) = \frac{1}{\|P_k^\ast\|}\cos \delta_k. 
$$
Thus, 
$$
f^\ast_k(z) \geq 1-\cos^2(\delta_k)=\sin^2(\delta_k)>0,\quad \forall z\in \hat G_\Gamma \cap \hat O_k. 
$$

Now we have that
$$
\delta^\ast=\min_k \inf_{z\in \hat G_{\Gamma} \cap \hat O_k} f^\ast_k(z) \geq 
\min\Big\{\min_{k\colon \|P_k^\ast\|<1} 1-\|P_k^\ast\|^2,\ \min_{k\colon \|P_k^\ast\|\geq 1} \sin^2(\delta_k)
\Big\}>0. 
$$
As shown in~\eqref{eq:uniform}, $f_k^{\alpha_n}$ converges uniformly to $f_k^\ast$ on $\bs^d$. 
Thus, there exists $N>0$ such that 
$$
\inf_{n>N} \min_k \inf_{z\in \hat G_{\Gamma} \cap \hat O_k} f^{\alpha_n}_k(z) \geq \frac{\delta^\ast}{2}. 
$$
That is, there exists $N$ such that for all $n>N$ and any $x\in \Gamma$, 
$$
\min_k \inf_{z\in \hat G_x \cap \hat O_k}f^{\alpha_n}_k(z) \geq \frac{\delta^\ast}{2}. 
$$
Then, with Proposition~\ref{prop:qpc}, we have that
\begin{equation}\label{eq:pwldetail2}
    \lim_{n\to+\infty} \sup_{x\in \Gamma} \|\nabla_x^2 h^\ast_{\alpha_n}(x)\|=0.
\end{equation}

We now show that $h^\ast$ is piecewise affine. 
Fix an open connected component $C$ of $\br^d\setminus \cup_k\pi_k$ and a point $x_0 \in C$. 
Select $x_1,\ldots,x_d\in C$ such that $\set{x_i-x_0}_{i=1}^d$ forms a basis of $\br^d$. 
By~Proposition~\ref{prop:tau-limit}, $h^\ast_{\alpha_n}$ is $C^2$ smooth for every $n$. 
By the Taylor's theorem, for every $i$, 
$$
h^\ast_{\alpha_n}(x_i)=h^\ast_{\alpha_n}(x_0)+\nabla h^\ast_{\alpha_n}(x_0)^\top(x_i-x_0)+ \frac{1}{2}(x_i-x_0)^\top \nabla^2 h^\ast_{\alpha_n}(\xi_{i,n}) (x_i-x_0),
$$
where $\xi_{i,n}$ is some point on the line segment between $x_0$ and $x_i$. 
With~\eqref{eq:pwldetail2}, we have that the limit
$$
\lim_{n\to+\infty} \nabla h^\ast_{\alpha_n}(x_0)^\top (x_i-x_0) = h^\ast(x_i)-h^\ast(x)
$$
exists for every $i$. 
This implies that $\lim_{n\to+\infty} \nabla h^\ast_{\alpha_n}(x_0)$ exists, which we denoted by $v^\ast$. 
Now for any $x'\in C$, we have that 
$$
h^\ast_{\alpha_n}(x')=h^\ast_{\alpha_n}(x_0)+\nabla h^\ast_{\alpha_n}(x_0)^\top(x'-x_0)+\frac{1}{2}(x'-x_0)^\top \nabla^2 h^\ast_{\alpha_n}(\xi_n') (x'-x_0),
$$
where $\xi_n'$ is some point on the line segment between $x'$ and $x_0$. 
% %
Again, with~\eqref{eq:pwldetail2}, we have that
$$
h^\ast(x')=h^\ast(x_0)+(v^\ast)^\top(x'-x_0). 
$$
Thus, $h^\ast$ is affine on $C$.

Finally, we show that $h^\ast$ is Lipschitz continuous and derive the bound on the Lipschitz constant. 
Note that
$\sigma^m(x,\theta) \rho^\ast_{\bl,m}(\theta)$ 
is integrable in $\theta$ for every $x$ by Lemma~\ref{lem:tau-pointwise} and is differentiable in $x$ for almost every $\theta$. 
When it is differentiable, we have that
$$
\nabla_x \sigma^m(x,\theta) \rho^\ast_{\bl,m}(\theta) = a^mw^m \rho^\ast_{\bl,m}(\theta)\left( \mathbb{I}_{x^\top w^m+b \in (0,m^2)} +
\mathbb{I}_{x^\top w^m+b\geq m^2} \phi_m'(x^\top w^m+b)\right). 
$$
Then, 
$$
\|\nabla_x \sigma^m(x,\theta) \rho^\ast_{\bl,m}(\theta)\| \leq 2\|aw\|\rho^\ast_{\bl,m}(\theta)\leq \|\theta\|^2 \rho^\ast_{\bl,m}(\theta). 
$$
The right-hand side is integrable again by Lemma~\ref{lem:tau-pointwise}.
Thus,
$$
\nabla_x h^\ast_{\bl,m}(x)=\int  \nabla_x \sigma^m(x,\theta) \rho^\ast_{\bl,m}(\theta) d\theta. 
$$
By Lemma~\ref{lem:Mbound}, 
\begin{align*}
    \sup_{x\in \br^d}\left\|\nabla_x h^\ast_{\bl,m}(x)\right\|
    &\leq \int \|\theta\|^2 d\rho^\ast_{\bl,m}(\theta)\\
    &\le 2(d+2) + 4(\lambda \beta)^{-1} (1+(d+2)\log(8\pi))+ \frac{2}{\lambda M}\sum_{j=1}^M y_j^2. 
\end{align*}
Since $h^\ast_{\bl,m}$ is $C^2$, we have that for any $x,x'\in \br^{d}$,
$$
|h^\ast_{\bl,m}(x)-h^\ast_{\bl,m}(x')|\leq \Big(2(d+2) + 4(\lambda \beta)^{-1} (1+(d+2)\log(8\pi))+ \frac{2}{\lambda M}\sum_{j=1}^M y_j^2\Big)\|x-x'\|.
$$
Under the sequence $(\beta_n,\lambda_n,m_n)$, taking $n\to+\infty$ yields that
$$
|h^\ast(x)-h^\ast(x')|\leq \Big( \frac{2}{\bar \lambda M}\sum_{j=1}^M y_j^2 + 2(d+2)\Big)\|x-x'\|.
$$
This completes the proof. 
\end{proof}

\subsection{Technical Lemmas}
\label{app:tech-tau}

In this section, we present technical Lemmas. 
Lemma~\ref{lem:tau-gauss-tail} shows that the stationary measure $\rho^\ast_{\hat \alpha}=\rho^\ast_{\bl,\mt}$ has Gaussian tail decay and is locally Lipschitz. It also controls the residuals of $h^\ast_{\hat \alpha}$. 
Lemma~\ref{lem:Mbound} controls the second moments of $\rho^\ast_{\hat \alpha}$ and $\rho^\ast_{\alpha}$. 
Lemma~\ref{lem:tau-pointwise} shows that $\rho^\ast_{\hat \alpha}$ converges pointwise to $\rho^\ast_{\alpha}$ as $\tau\to+\infty$ and that $\rho^\ast_\alpha$ also has Gaussian tail decay. 
Lemma~\ref{lem:hess-R} bounds the residuals of $h^\ast_\alpha$. 
Lemma~\ref{lem:hess-Z} provides estimates on the partition function $Z_{\alpha}^\ast$, whose definition is given below.
All these technical results are adapted and refined from those of \citet{shevchenko2022mean} to accommodate arbitrary dimensions.

For $j=1,\ldots,M$, let
$$
R^{\hat \alpha}_j  = \frac{1}{M}(h^\ast_{\hat \alpha}(x_j) - y_j).
$$
Let 
$$ 
\Psi^\ast_{\alpha}(\theta) = \sum_{j=1}^M R^\alpha_j \cdot \sigma^{m}(x_j,\theta)+\frac{\lambda}{2}\|\theta\|^2, 
\quad
\Psi^\ast_{\hat \alpha}(\theta) = \sum_{j=1}^M R_j^{\hat \alpha} \cdot \sigma^{\mt}(x_j,\theta)+\frac{\lambda}{2}\|\theta\|^2. 
$$
Then 
$$
\rho^\ast_{ \alpha}(\theta) = (Z_{\alpha}^\ast)^{-1}\exp(-\beta\Psi^\ast_{\alpha}(\theta)),\qquad 
\rho_{\hat \alpha}^\ast(\theta) = (Z_{\hat \alpha}^\ast)^{-1}\exp(-\beta\Psi^\ast_{\hat \alpha}(\theta)),
$$
where $Z_{\alpha}^\ast=\int \exp(-\beta\Psi^\ast_{\alpha}(\theta)) d\theta$ and 
$Z_{\hat \alpha}^\ast=\int \exp(-\beta\Psi^\ast_{\hat \alpha}(\theta)) d\theta$.

\begin{lemma}\label{lem:tau-gauss-tail}
    For any $\beta,\lambda,m,\tau>0$, 
    there exists $C>0$ independent of $\tau$ such that
    $$
    \max_j |R^{\hat \alpha}_j|<C, \qquad 
    \rho_{\hat \alpha}^\ast(\theta) \leq C \exp (-\frac{\beta \lambda}{2}\|\theta\|^2),\ \forall \theta \in \br^{d+2}.
    $$
    Additionally, given any compact set $K\subset \br^{d+2}$,
    there exists $C_K>0$ independent of $\tau$ such that
    $\rho_{\hat \alpha}^\ast(\theta)$ is $C_K$-Lipschitz continuous on $K$. 
\end{lemma}

\begin{proof}
    In this proof, we use $(C_i)_{i\in \bn}$ for positive constants that are independent of $\tau$. 
    By the construction of $\sigma^{\mt}$, 
    we have that
    $$
    \max_j |\sigma^{\mt}(x_j,\theta)|\leq C_1, \quad 
    \max_j |\sigma^{\mt}(x_j,\theta)- \sigma^{\mt}(x_j,\theta')|\leq C_2 \|\theta-\theta'\|, \quad \forall \theta,\theta'\in \br^{d+2}. 
    $$
    It follows that
    $$
    \max_j |h^\ast_{\hat \alpha}(x_j)| = \max_j \Big|\bE_{\theta \sim \rho_{\hat \alpha}^\ast}[\sigma^{\mt}(x_j, \theta)]\Big| \leq C_3,
    $$
    and thus that
    $$
    \max_j |R^{\hat \alpha}_j| < C_4.
    $$    
    Then we have that
    $$
    L+\frac{\lambda}{2}\|\theta\|^2 \geq \Psi^\ast_{\hat \alpha}(\theta) \geq R+\frac{\lambda}{2}\|\theta\|^2, \quad \forall \theta \in \br^{d+2},
    $$
    where $L,R\in \br$ are constants independent of $\tau$. 
    Meanwhile, notice that
    $$
    Z_{\hat \alpha}^\ast = \int \exp(-\beta \Psi^\ast_{\hat \alpha}(\theta)) d\theta \geq \int \exp(-\beta L - \frac{\beta \lambda}{2}\|\theta\|^2)d\theta     \geq C_5. 
    $$
    This implies that
    $$
    \rho_{\hat \alpha}^\ast(\theta) =
    (Z_{\hat \alpha}^\ast)^{-1}\cdot \exp(-\beta \Psi^\ast_{\hat \alpha }(\theta)) \leq C_{6} \exp (-\frac{\beta \lambda}{2}\|\theta\|^2).
    $$

    Since $K$ is compact, we have that
    $$
    |\|\theta\|^2-\|\theta'\|^2 |\leq C_7 \|\theta-\theta'\|, \quad \forall \theta,\theta'\in K. 
    $$
    Thus, 
    $$
    |\Psi^\ast_{\hat \alpha}(\theta)|<C_8,\quad 
    |\Psi^\ast_{\hat \alpha}(\theta) - \Psi^\ast_{\hat \alpha}(\theta')|\leq C_9 \|\theta-\theta'\|, \quad \forall \theta,\theta'\in K. 
    $$
    Therefore,  
    $$
    \Big| \exp(-\beta \Psi^\ast_{\hat \alpha}(\theta))- \exp(-\beta \Psi_{\hat \alpha}^\ast(\theta'))\Big| \leq  C_{10} \|\theta-\theta'\|, \quad \forall \theta,\theta'\in K. 
    $$
    Then 
    $$
    |\rho_{\hat \alpha}^\ast(\theta) - \rho_{\hat \alpha}^\ast(\theta')| = |(Z^\ast_{\hat \alpha})^{-1}| \cdot \Big| \exp(-\beta \Psi^\ast_{\hat \alpha}(\theta))- \exp(-\beta \Psi_{\hat \alpha}^\ast(\theta'))\Big|  \leq C_{11} \|\theta-\theta'\|, \quad \forall \theta,\theta'\in K, 
    $$
    which completes the proof. 
\end{proof}

\begin{lemma}\label{lem:Mbound}
    For $\hat \alpha=(\bl, \mt)$ that satisfies Assumption~\ref{assum}, 
    $$
    \max\{M(\rho^\ast_{ \alpha}), M(\rho^\ast_{\hat \alpha})\} \le 2(d+2)+4(\lambda \beta)^{-1} (1+(d+2)\log(8\pi))+ \frac{2}{\lambda M}\sum_{j=1}^M y_j^2. 
    $$
    Therefore, there exists $C,C'>0$ independent of $\hat \alpha$ such that 
    $$
    \max\{M(\rho^\ast_{ \alpha}), M(\rho^\ast_{\hat \alpha})\} \leq C \lambda^{-1}+C'. 
    $$
\end{lemma}

\begin{proof}
    Let $\rho=N(0,I_{d+2})$. 
    Then, 
    $$
    M(\rho)=d+2,\quad H(\rho)=\frac{d+2}{2}\log(2\pi e). 
    $$
    Since $a\mapsto a^{m,\tau}$ is odd, $h_{\sigma^{\mt}}(x,\rho)=0$ for any $x$. 
    As $\beta>1$, 
    $$
    F_{\sigma^{\mt}}^{\bl}(\rho) = \frac{1}{2M}\sum_{j=1}^M y_j^2 + \frac{d+2}{2} \lambda -  \beta^{-1} \frac{d+2}{2}\log(2\pi e). 
    $$
    With \citet[][Lemma 10.2]{mei2018mean}, we have that
    $$
    \frac{1}{2M}\sum_{j=1}^M y_j^2 + \frac{d+2}{2} \lambda\ge F_{\sigma^{\mt}}^{\bl}(\rho) \geq F_{\sigma^{\mt}}^{\bl}(\rho_{\hat \alpha}^\ast) \geq \frac{\lambda}{4}M(\rho_{\hat \alpha}^\ast)-\beta^{-1}(1+(d+2)\log(\frac{8\pi}{\lambda \beta})).
    $$
    Since $\lambda\beta>1$, 
    $$
    M(\rho^\ast_{ \alpha}) \le  2(d+2) + 4(\lambda \beta)^{-1} (1+(d+2)\log(8\pi))+ \frac{2}{\lambda M}\sum_{j=1}^M y_j^2.
    $$
    It follows that
    $$
    M(\rho_{\hat \alpha}^\ast) \leq C'+C\lambda^{-1},
    $$
    where $C,C'>0$ are independent of $\hat \alpha$. 
    Notice that $a\mapsto a^m$ is also odd and that $\sigma^m$ is bounded Lipschitz. 
    Thus, all above arguments apply. 
    This completes the proof. 
\end{proof}

\begin{lemma}\label{lem:tau-pointwise}
    For $\hat \alpha=(\bl,\mt)$ that satisfies Assumption~\ref{assum}, 
    $\rho^\ast_{\hat \alpha}(\theta)$ converges to $\rho^\ast_\alpha(\theta)$ as $\tau\to+\infty$, for every $\theta \in \br^{d+2}$. 
    Consequently, there exists $C>0$ such that $\rho_{ \alpha}^\ast(\theta)\leq C \exp(-\frac{\beta\lambda}{2}\|\theta\|^2)$ for every $\theta \in \br^{d+2}$. 
\end{lemma}

\begin{proof}
The key step of the proof is to show the weak convergence of measures:
\begin{equation*}
    \rho_{\hat\alpha}^\ast d\theta \Rightarrow \rho_\alpha^\ast d\theta,\quad \tau \to +\infty.
\end{equation*}
By Lemma~\ref{lem:tau-gauss-tail}, there exists a constant $C$ independent of $\tau$ such that 
$$
\rho_{\hat \alpha}^\ast(\theta)\leq C \exp(-\frac{\beta\lambda}{2}\|\theta\|^2). 
$$
Without loss of generality, assume $C>1$. 
Then,
\begin{equation}\label{eq:tau-pointwise-1}
    \begin{aligned}
    \int \max\set{\rho_{\hat \alpha}^\ast \log  \rho_{\hat \alpha}^\ast,0  }   d\theta &\leq 
    \int \max\Big\{C \exp(-\frac{\beta\lambda}{2}\|\theta\|^2) \cdot \Big(\log C -\frac{\beta\lam }{2}\|\theta\|^2    \Big)  ,0  \Big\}   d\theta \\
    &\le \int_{\set{\theta\colon \|\theta\|^2 \leq \frac{2\log C}{\beta \lambda} }} C \log C d\theta\\
    &\leq C' , 
    \end{aligned}
\end{equation}
where $C'$ is independent of $\tau$.

Now consider any sequence $\hat\alpha_n=(\bl,\mt_n)$ with $\tau_n\to+\infty$ as $n\to +\infty$. 
By \eqref{eq:tau-pointwise-1} and Lemma~\ref{lem:Mbound}, 
the sequence $\rho_{\hat \alpha_n}^\ast$ has bounded truncated entropy and second moment. 
Thus, by \citet[][Lemma A.1]{shevchenko2022mean}, 
there exists $\mu_\alpha \in \ch$ and a subsequence $\{n_k\}_k$
such that 
$$
\rho^\ast_{\hat \alpha_{n_k}} \rightharpoonup \mu_\alpha \text{ weakly in }L^1,
\quad k\to +\infty. 
$$
To simplify the notation, let $\hat \alpha_{k}=\hat \alpha_{n_k}$ and $\tau_k=\tau_{n_k}$.  
Note that $L^1$-weak convergence of densities implies weak convergence of measures. 
As $\theta \mapsto \|\theta\|^2$ is continuous and bounded from below,
by Portmanteau theorem,
\begin{equation}\label{eq:tau-pointwise-2}
\liminf_{k\to +\infty} M(\rho^\ast_{\hat \alpha_{k}})  \geq M(\mu_\alpha). 
\end{equation}
Notice that the map $s \mapsto s\log s$ is proper, convex and lower semi-continuous. 
By \citet[][Remark 3.16, Lemma 3.17]{ambrosio2007gradient}, we have that 
\begin{equation}\label{eq:tau-pointwise-3}
\liminf_{k\to +\infty} -H(\rho^\ast_{\hat \alpha_{k}})  \geq -H(\mu_\alpha). 
\end{equation}
Note that
\begin{align*}
    |h_{\sigma^{m,\tau_k}}(x, \rho_{\hat \alpha_k}^\ast) - h_{\sigma^{m}}(x,\mu_\alpha)| \leq 
    |h_{\sigma^{m,\tau_k}}(x, \rho_{\hat \alpha_k}^\ast) - h_{\sigma^{m}}(x,\rho_{\hat \alpha_k}^\ast)| + 
    |h_{\sigma^{m}}(x,\rho_{\hat \alpha_k}^\ast) - h_{\sigma^{m}}(x,\mu_\alpha)|. 
\end{align*}
For the first term, we have that
$$
|h_{\sigma^{m,\tau_k}}(x, \rho_{\hat \alpha_k}^\ast) - h_{\sigma^{m}}(x,\rho_{\hat \alpha_k}^\ast)| \leq \int |\sigma^{\tau_k,m}(x,\theta) - \sigma^m(x,\theta)| \rho_{\hat \alpha_k}^\ast(\theta) d\theta. 
$$
Notice that 
(i) $\rho_{\hat \alpha_k}^\ast$ has Gaussian tail decay with constants independent of $\tau$ according to Lemma~\ref{lem:tau-gauss-tail}, 
(ii) $\sigma^{\mt}$ and $\sigma^m$ are bounded uniformly in $\tau$, 
and (iii) $\sigma^{\mt}(x, \theta)$ converges pointwise in $(x,\theta)$ to $\sigma^m(x,\theta)$ as $\tau\to +\infty$. 
By the Dominated Convergence Theorem, the first term vanishes as $k \to +\infty$. 
The second term also vanishes by the $L^1$-weak convergence. 
Then we have that
$$
\lim_{k\to +\infty} h_{\sigma^{m,\tau_k}}(x, \rho_{\hat \alpha_k}^\ast) = h_{\sigma^{m}}(x,\mu_\alpha).
$$
This, together with \eqref{eq:tau-pointwise-2}, \eqref{eq:tau-pointwise-3}, yields that
\begin{equation}\label{eq:tau-pointwise-5}
\liminf_{k\to+\infty} F^{\bl}_{\sigma^{m,\tau_k}}(\rho_{\hat \alpha_k}^\ast) \geq F^{\bl}_{\sigma^m}(\mu_\alpha). 
\end{equation}
On the other side, as $\sigma^{\mt}$ and $\sigma^m$ are bounded uniformly in $\tau$, and $\sigma^{\mt}(x, \theta)$ converges to $\sigma^m(x,\theta)$ as $\tau \to +\infty$, the Dominated Convergence Theorem gives that 
$$
\lim_{k\to +\infty} h_{\sigma^{m,\tau_k}}(x_j, \rho_\alpha^\ast) = h_{\sigma^{m}}(x_j, \rho_\alpha^\ast),\quad \forall j,
$$
which implies that 
\begin{equation}\label{eq:tau-pointwise-6}
\lim_{k\to +\infty} F^{\bl}_{\sigma^{m,\tau_k}}(\rho_\alpha^\ast) = F^{\bl}_{\sigma^m}(\rho_\alpha^\ast).
\end{equation}
Combining \eqref{eq:tau-pointwise-5} and \eqref{eq:tau-pointwise-6}, we have that
$$
F^{\bl}_{\sigma^m} (\mu_\alpha) \leq \liminf_{k\to+\infty} F^{\bl}_{\sigma^{m,\tau_k}}(\rho_{\hat \alpha_k}^\ast) \leq 
\liminf_{k\to+\infty} F^{\bl}_{\sigma^{m,\tau_k}}(\rho_{\alpha}^\ast) = F^{\bl}_{\sigma^m}(\rho_\alpha^\ast). 
$$
Since $\rho_\alpha^\ast$ is the unique minimizer of $F^{\bl}_{\sigma^m}$, we have that 
$\mu_\alpha=\rho_\alpha^\ast$ a.e..

By the above, for any sequence $\rho^\ast_{\hat \alpha_n}$ with $\tau_n\to+\infty$, 
one can find a subsequence that converges weakly to $\rho_\alpha^\ast$ in $L^1$. 
Again, note that $L^1$-weak convergence of densities implies weak convergence of measures. By \citet[][Theorem 2.6]{billingsley2013convergence}, we have that
$$
\rho_{\hat\alpha}^\ast d\theta \Rightarrow \rho_\alpha^\ast d\theta,\quad  
\tau \to +\infty.
$$

Since $\sigma^{\mt}$ and $\sigma^m$ are uniformly bounded in $\tau$, we have that 
\begin{equation*}\label{eq:tau-pointwise-7}
    \lim_{\tau \to +\infty} 
    h^\ast_{\hat\alpha}(x_j)= h^\ast_\alpha(x_j), \quad \forall j. 
\end{equation*}
It follows that 
\begin{align*}
    \lim_{\tau \to +\infty} \Psi^\ast_{\hat \alpha}(\theta) =\Psi_\alpha^\ast(\theta),\quad \forall \theta.
\end{align*}
This implies that 
$$
\lim_{\tau \to +\infty} \exp(-\beta \Psi^\ast_{\hat \alpha}(\theta) )=\exp(-\beta \Psi_\alpha^\ast(\theta)) ,\quad \forall \theta.
$$
By Lemma~\ref{lem:tau-gauss-tail}, $\max_j |R_j^{\hat \alpha}|$ is bounded by a constant independent of $\tau$. Thus,
$$
 \exp(-\beta \Psi^\ast_{\hat \alpha}(\theta) ) 
 =\exp(-\beta(\sum_{j=1}^M R_j^{\hat \alpha}\cdot \sigma^{m}(x_j,\theta)+\frac{\lambda}{2}\|\theta\|^2))
 \leq C'\exp(-\frac{\beta\lambda}{2}\|\theta\|^2),
$$
for some $C'>0$ independent of $\tau$. 
Thus, with the Dominated Convergence Theorem, 
$$
\lim_{\tau\to+\infty} Z_{\hat \alpha}^\ast= 
\lim_{\tau \to +\infty} \int \exp(-\beta \Psi^\ast_{\hat \alpha}(\theta) )d\theta  = \int \exp(-\beta \Psi_\alpha^\ast(\theta)) d\theta= Z_\alpha^\ast. 
$$
Then we have that
$$
\lim_{\tau \to +\infty} \rho_{\hat \alpha}^\ast(\theta)  = \rho_{\alpha}^\ast(\theta).
$$
Finally, with Lemma~\ref{lem:tau-gauss-tail} we have that, 
there exists $C>0$ such that $\rho_{\alpha}^\ast(\theta)\leq C \exp(-\frac{\beta\lambda}{2}\|\theta\|^2)$ for every $\theta \in \br^{d+2}$. 
This completes the proof. 
\end{proof}

\begin{lemma}\label{lem:hess-R}
    For any $\alpha=(\bl,m)$ satisfying Assumption~\ref{assum}, there exists a constant $C>0$ independent of $\alpha$ such that
    $$
    \sum_{j=1}^M |R^\alpha_j| \leq C \sqrt{\lambda}.
    $$ 
\end{lemma}

\begin{proof}
    We first construct a Dirac measure $\nu$ such that the $h_{\sigma^m}(\cdot,\nu)$ interpolates the training data. 
    Recall that $x_i \neq x_j$ whenever $i\neq j$. 
    Thus, there exists $v\in \br^d$ such that 
    $\|v\|=1$ and $v^\top x_i \neq v^\top x_j$ if $i\neq j$. 
    Define
    $$
    \nu_j = \frac{1}{3}\Big( \delta(\frac{3My_j}{\varepsilon}, v, \varepsilon-v^\top x_j)+
    \delta(-\frac{6My_j}{\varepsilon}, v, -v^\top x_j)+
    \delta(\frac{3My_j}{\varepsilon}, v, -\varepsilon-v^\top x_j)
    \Big)\in \mathcal{P}_2(\br^{d+2}). 
    $$
    Then $x\mapsto \int a(w^\top x +b)_+ d\nu_j$ 
    is a tent function along $v$, centered at $x_j$, 
    with height $My_j$, and support width controlled by $\varepsilon$.
    Let $\nu = \frac{1}{M}\sum_{j=1}^M \nu_j\in \mathcal{P}_2(\br^{d+2})$.
    With sufficiently small $\varepsilon$, we have that 
    $$
    \int a(w^\top x_j+b)_+ d\nu =y_j,\quad j=1,\ldots,M. 
    $$
    Fix such an $\varepsilon$. 
    Since $\supp\nu$ is compact, there exists $M_0>0$ independent of $\alpha$ such that for all $m>M_0$, 
    $$
    \sigma^m(x_j,\theta)=a(w^\top x_j+b)_+\quad \forall \theta\in \supp \nu,\ j=1,\ldots,M.
    $$ 
    With $m>M_0$, we have that
    $$
    h_{\sigma^m}(x_j,\nu)=\int \sigma^m(x_j,\theta) d\nu =y_j,\quad j=1,\ldots, M. 
    $$

    Now we construct a measure $\mu\in \ch$ that has a finite entropy and is close to $\nu$. 
    For any finite set $S=\set{\theta_i}_{i=1}^N \subset \br^{d+2}$, let $B_\infty^\zeta(S)=\set{ \theta\in \br^{d+2}\colon \min_i\|\theta-\theta_i\|_\infty\leq\zeta }$. 
    Define 
    $$
    \mu_j = \unif\Big( B_{\infty}^\zeta ( \supp \nu_j) \Big), \quad \mu=\frac{1}{M}\sum_{j=1}^M \mu_j\in \mathcal{P}_2^a(\br^{d+2}). 
    $$
    Since $\supp\mu$ shrinks to $\supp\nu$ as $\zeta \to 0$, there exist $\zeta_0$ such that for all $\zeta<\zeta_0$ and $m>M_0$, 
    \begin{equation}\label{eq:hess-eq1}
        \sigma^m(x_j,\theta)=a(w^\top x_j+b)_+,\quad  \forall \theta\in \supp \mu,\  j=1,\ldots,M.
    \end{equation}
    Note that $\nu_j$ is supported on three points. 
    Without loss of generality, assume that for all $\zeta<\zeta_0$, 
    $B_\infty^\zeta(\supp \nu_j)$ consists of three disjoint 
    $\ell_\infty$ balls.

    For $j=1,\ldots,M$, suppose that $\nu_j = \frac{1}{3} \sum_{k=1}^3\delta(a_{jk},w_{jk},b_{jk})$ with $(a_{jk},w_{jk},b_{jk})\in \br^{d+2}$. 
    Let $\nu'_{jk}=\delta(w_{jk},b_{jk})$, and $\mu_{jk}'=\unif(B_\infty^\zeta(w_{jk},b_{jk}))$. 
    Then with \eqref{eq:hess-eq1} we have that
    \begin{align*}
        \int \sigma^m(x, \theta) d\mu_j - \int \sigma^m(x,\theta) d\nu_j  &= \frac{My_j}{\varepsilon} \Big( \int  (w^\top x +b)_+ d \mu_{j1}'- \int  (w^\top x +b)_+ d \nu_{j1}'\Big) \\ 
        &\quad\quad - \frac{2My_j}{\varepsilon} \Big( \int  (w^\top x +b)_+ d \mu_{j2}'- \int  (w^\top x +b)_+ d \nu_{j2}'\Big) \\
        &\quad\quad + \frac{My_j}{\varepsilon} \Big( \int  (w^\top x +b)_+ d \mu_{j3}'- \int  (w^\top x +b)_+ d \nu_{j3}'\Big). 
    \end{align*}
    Notice that $(w,b)\mapsto (w^\top x+b)_+$ is $L$-Lipschitz for a constant $L$ independent of $\alpha$. 
    Let $W_1$ denote the $1$-Wasserstein distance with respect to $(\br^{d+1},\|\cdot\|_2)$. 
    Then with the Kantorovich-Rubinstein duality, we have that
    \begin{align*}
        \left| \int \sigma^m(x, \theta) d\mu_j - \int \sigma^m(x,\theta) d\nu_j\right| &\leq \frac{2M|y_j|}{\varepsilon} L \sum_{k=1}^3 W_1(\mu_{jk}', \nu_{jk}')\\
        &\le \frac{2M|y_j|}{\varepsilon} L \sum_{k=1}^3 \sup_{p\in B_\infty^\zeta(w_{jk},b_{jk}) } \|p-(w_{jk},b_{jk})\|_2\\
        &\le \frac{2M|y_j|}{\varepsilon} L  \cdot 3 \zeta \sqrt{d+1}. 
    \end{align*}
    Therefore, there exists $K>0$ that is independent of $\alpha$ such that, for every $\zeta<\zeta_0$,  
    $$
    \max_j \left| \int \sigma^m(x_j, \theta) d\mu - \int \sigma^m(x_j,\theta) d\nu\right|= \max_j \left| \int \sigma^m(x_j,\theta) d\mu - y_j \right| \leq K \zeta. 
    $$
    Then, with $\zeta = \min (\zeta_0/2, \sqrt{\lambda})$ we have that
    $$
    \frac{1}{2M}\sum_{j=1}^M \left( h_{\sigma^m}(x_j,\mu)-y_j \right)^2 \leq C_0 \lambda, 
    $$
    for some constant $C_0>0$ independent of $\alpha$.

    Note that $H$ is concave on $\ch$. 
    Therefore, 
    $$
    H(\mu) \geq \frac{1}{M}\sum_{j=1}^M H(\mu_j) = \frac{1}{M}\sum_{j=1}^M  \log (\vol(B_\infty^\zeta(\supp \nu_j))) = \log(3(2\zeta)^{d+2}). 
    $$
    To bound the second moment of $\mu$, recall that for all $\zeta<\zeta_0$, $\supp \mu$ is contained in a compact set that is independent of $\alpha$. 
    Thus, $M(\mu)<2 C_1$ for some constant $C_1>0$ independent of $\alpha$. 
    Then, we have that 
    $$
    F^{\bl}_{\sigma^m}(\mu) \leq C_0\lambda + C_1 \lambda -\frac{1}{\beta} \log(3(2\zeta)^{d+2}). 
    $$
    With \citet[][Lemma 10.2]{mei2018mean}, we have that
    \begin{align*}
        &\frac{1}{2M}\sum_{j=1}^M (h^\ast_{\alpha}(x_j)-y_j )^2 \\
        \leq & F^{\bl}_{\sigma^m}(\rho_\alpha^\ast)  - \frac{\lam}{4} M(\rho_\alpha^\ast)+\frac{1}{\beta}(1+(d+2)\log (\frac{8\pi}{\beta \lambda})) \\
        \leq & F^{\bl}_{\sigma^m}(\mu) +\frac{1}{\beta}(1+(d+2)\log (\frac{8\pi}{\beta \lambda}))\\
        \leq & (C_0+C_1) \lambda +  \left(1+(d+2) \log(\frac{8\pi}{\beta \lambda})-\log 3 \right) \frac{1}{\beta} + (d+2)\frac{1}{\beta} \log \left(\frac{1}{2\min\set{\zeta_0/2,\sqrt{\lambda}}}\right).
    \end{align*}
    As $\beta > \lambda^{-1}$ and $\beta>\lambda^{-1}\log(\lambda^{-1})$, 
    $$
    \left(1+ (d+2) \log(\frac{8\pi}{\beta \lambda})-\log 3 \right) \frac{1}{\beta} \leq (1-\log 3)\frac{1}{\beta} + \frac{d+2}{\beta}\log(\frac{8\pi}{\lambda}) \leq C_2 \lambda,
    $$
    for some constant $C_2>0$ independent of $\alpha$. 
    When $\sqrt{\lambda}<\zeta_0/2$, 
    we have that
    $$
    (d+2)\frac{1}{\beta} \log \left(\frac{1}{2\min\set{\zeta_0/2,\sqrt{\lambda}}}\right) = (d+2)\frac{1}{\beta} \log \left(\frac{1}{2\sqrt{\lambda}}\right)\leq C_3\lambda,
    $$
    for some constant $C_3>0$ independent of $\alpha$. 
    When $\sqrt{\lambda}\geq \zeta_0/2$, 
    $$
    (d+2)\frac{1}{\beta} \log \left(\frac{1}{2\min\set{\zeta_0/2,\sqrt{\lambda}}}\right) = (d+2)\frac{1}{\beta} \log \left(\frac{1}{\zeta_0}\right)\leq C_4\lambda,
    $$
    for some constant $C_4>0$ independent of $\alpha$. 
    Finally, we have that
    $$
    \sum_{j=1}^M (h^\ast_\alpha(x_j)-y_j )^2 \leq C_5 \lambda,
    $$
    for some constant $C_5$ independent of $\alpha$, 
    which implies that
    $$
    \sum_{j=1}^M|R^\alpha_j|=\frac{1}{M}\sum_{j=1}^M |h^\ast_{\alpha}(x_j)-y_j| \leq C_6 \sqrt{\lambda},
    $$
    for some constant $C_6$ independent of $\alpha$.
    This completes the proof. 
\end{proof}

\begin{lemma}\label{lem:hess-Z}
    For any $\alpha=(\bl,m)$ satisfying Assumption~\ref{assum}, 
    $$
    C \beta^{-\frac{d+2}{2}} \lambda^{-\frac{d+2}{4}}(1+\sqrt{\lambda})^{-\frac{d+2}{2}} \le Z_\alpha^\ast \le \exp\left( C' \beta +1+ (d+2)\log \frac{8\pi}{\lambda \beta} \right)
    $$ 
    for some constants $C,C'>0$ independent of $\alpha$. 
\end{lemma}

\begin{proof}
First, we obtain the lower bound. 
Let $C_1=\max\{\max_j\|x_j\|,1\}$. 
Notice that
\begin{align*}
    Z_\alpha^\ast &= \int \exp\curv{ -\beta\left[ \sum_{j=1}^M R_j^\alpha \cdot a^m (x^\top_j w^m  +b)_+^m +\frac{\lambda}{2}\|\theta\|^2  \right] } d\theta\\
    &\ge \int \exp\curv{ -\beta\left[ \sum_{j=1}^M |
    R_j^\alpha| |a|(\|x_j\|\|w\|+|b|) +\frac{\lambda}{2}\|\theta\|^2  \right] } d\theta\\
    &\ge \int \exp\curv{ -\beta\left[ C_1(\|aw\|+|ab|)\sum_{j=1}^M |
    R_j^\alpha|  +\frac{\lambda}{2}\|\theta\|^2  \right] } d\theta.
\end{align*}
By Lemma~\ref{lem:hess-R}, there exists $C_2>0$ independent of $\alpha$ such that
\begin{align*}
    Z_\alpha^\ast 
    &\ge \int \exp\curv{ -\beta\left[ C_1 C_2\sqrt{\lambda} (\|aw\|  +|ab|)  +\frac{\lambda}{2}\|\theta\|^2  \right] } d\theta\\
    &\ge \int \exp\curv{ -\beta\left[  (C_1C_2\sqrt{\lambda}+\frac{\lambda}{2})a^2 + (\frac{C_1C_2\sqrt{\lambda}+\lambda}{2}) (\|w\|^2 +b^2)  \right] } d\theta\\
    & = \frac{(\pi)^{\frac{d+2}{2}}}{(C_1C_2 \sqrt{\lambda}+\frac{\lambda}{2})^{\frac{1}{2}} (\frac{C_1C_2}{2}\sqrt{\lambda}+\frac{\lambda}{2})^{\frac{d+1}{2}} \beta^{\frac{d+2}{2}} }\\
    &\ge \frac{(\pi)^{\frac{d+2}{2}}}{(C_1C_2 \sqrt{\lambda}+\frac{\lambda}{2})^{\frac{d+2}{2}}  \beta^{\frac{d+2}{2}} }\\
    &\ge C_3 \beta^{-\frac{d+2}{2}} \lambda^{-\frac{d+2}{4}}(1+\sqrt{\lambda})^{-\frac{d+2}{2}},
\end{align*}
for some $C_3>0$ independent of $\alpha$.

For the upper bound, note that 
\begin{align*}
    H(\rho_\alpha^\ast) &= -\int\rho_\alpha^\ast(\theta) \log \rho_\alpha^\ast(\theta) d\theta\\        
    & = - \int \rho_\alpha^\ast(\theta)\left( -\beta\sum_{j=1}^M R_j^\alpha\sigma^m(x_j,\theta)-\frac{\lambda \beta}{2}\|\theta\|^2 - \log Z_\alpha^\ast     \right) d\theta\\
    & =  \beta \sum_{j=1}^M R_j^\alpha \int \rho_\alpha^\ast(\theta ) \sigma^m(x_j,\theta) d\theta + \frac{\lambda \beta}{2}M(\rho_\alpha^\ast) + \log Z_\alpha^\ast\\
    & =  \beta  \sum_{j=1}^M R_j^\alpha 
    h^\ast_{\alpha}(x_j)  + \frac{\lambda \beta}{2}M(\rho_\alpha^\ast) + \log Z_\alpha^\ast. 
\end{align*}
Then we have that 
\begin{align*}
    F^{\bl}_{\sigma^m}(\rho_\alpha^\ast) 
    &= \frac{1}{2M} \sum_{j=1}^M (h^\ast_{\alpha}(x_j)-y_j)^2  + \frac{\lam}{2}M(\rho_\alpha^\ast) - \frac{1}{\beta}H(\rho_\alpha^\ast)\\
    &= \frac{1}{2M} \sum_{j=1}^M (h^\ast_{\alpha}(x_j)-y_j)^2 - \frac{1}{M}\sum_{j=1}^M (h^\ast_{\alpha}(x_j)-y_j) h^\ast_{\alpha}(x_j) -  \frac{1}{\beta} \log Z_\alpha^\ast\\
    &= -\frac{1}{2M}\sum_{j=1}^M(h^\ast_{\alpha}(x_j))^2 + \frac{1}{2M} \sum_{j=1}^My_j^2 -  \frac{1}{\beta} \log Z_\alpha^\ast\\
    &\leq \frac{1}{2M} \sum_{j=1}^My_j^2 -  \frac{1}{\beta} \log Z_\alpha^\ast\\
    &\leq C_4 - \frac{1}{\beta} \log Z_\alpha^\ast,
\end{align*}
for some $C_4>0$ independent of $\alpha$.
With \citet[][Lemma 10.2]{mei2018mean}, we have that 
$$
C_4 - \frac{1}{\beta} \log Z_\alpha^\ast\ge F^{\bl}_{\sigma^m}(\rho_\alpha^\ast) \geq -\frac{1}{\beta}(1+(d+2)\log(\frac{8\pi}{\beta \lambda})). 
$$
Then we have that
$$
Z_\alpha^\ast \le \exp\left( C_4 \beta +1+ (d+2)\log \frac{8\pi}{\lambda \beta} \right). 
$$
This completes the proof. 
\end{proof}

\section{Proof of Theorem~\ref{thm:feature}}
\label{app:feature}

In this section, we present the proof of Theorem~\ref{thm:feature}, which shows that the trained input weights and biases align along a few directions.

\begin{proof}\textit{(Proof of Theorem~\ref{thm:feature})}
Let $\alpha_n=(\beta_n,\lambda_n,m_n)$. 
Let $z=(w,b)$. 
Consider any open set $B\subset \br^{d+1}$
such that 
\begin{equation}\label{eq:B-nbhd}
0\notin \cl(B),\quad  \widehat{\cl B}\cap \hat \Omega^\ast=\varnothing \text{ where } \widehat{\cl B} \triangleq \set{\frac{z}{\|z\|}\colon z\in \cl(B)} . 
\end{equation}
Recall the definition of $\hat \Omega^\ast$ from Proposition~\ref{prop:shrink}.
We have that
$$
\delta = \min \Big\{ \min_{k} \inf_{z\in \widehat{\cl B} \cap \hat O_k} f_k^\ast(z),\ \frac{1}{2} \Big\}\in (0,\frac{1}{2}]. 
$$

Fix any $\varepsilon>0$. 
Since $\rho^\ast_{\alpha_n}$ converges to $\rho^\ast$ in $2$-Wasserstein distance,
$\rho^\ast_{\alpha_n}$ has uniformly integrable second moments \citep[see, e.g.,][Proposition 7.1.5]{ambrosio2005gradient}. 
Thus, there exists $L>0$ such that 
$$
\int_{\|\theta\|>L} \|z\|^2 \rho^\ast_{\alpha_n}(\theta) d\theta<\varepsilon,\quad \forall n\geq 1. 
$$
Thus, 
\begin{align*}
    \int_{\br \times B} \|z\|^2 \rho^\ast_{\alpha_n}(\theta) d\theta 
    \le \int_{(\br \times B)\cap \set{\|\theta\|\leq L}} \|z\|^2 \rho^\ast_{\alpha_n}(\theta) d\theta + \varepsilon. 
\end{align*}
For the first term of the right-hand side, we have that 
$$
\int_{(\br \times B)\cap \set{\|\theta\|\leq L}} \|z\|^2 \rho^\ast_{\alpha_n}(\theta) d\theta = \sum_{k=1}^{2\cp(\cx)} \int_{(\br \times (B\cap O_k))\cap \set{\|\theta\|\leq L}} \|z\|^2 \rho^\ast_{\alpha_n}(\theta) d\theta. 
$$
Consider any fixed $k$. Since $L$ is fixed, for sufficiently large $m$ we have that
$$
\sigma^m(x_j,\theta)=a^m(x_j^\top w^m+b)_+^m=a(x_j^\top w+b)_+,\quad \forall \theta \in\set{\|\theta\|<L} , \forall j.
$$
Let 
$$
Z^\ast_\alpha=\int \exp\Big\{-\beta \Big( \sum_{j=1}^M R^\alpha_j \sigma^m(x_j,\theta) +\frac{\lambda}{2}\|\theta\|^2  \Big)  \Big\} d\theta. 
$$
Then we have that
{\small 
\begin{align*}
    &Z_{\alpha_n}^\ast\cdot \int_{(\br \times (B\cap O_k))\cap \set{\|\theta\|\leq L}} \|z\|^2 \rho^\ast_{\alpha_n}(\theta) d\theta \\
    =&\int_{(\br \times (B\cap O_k))\cap \set{\|\theta\|\leq L}} \|z\|^2 \exp\curv{-\frac{\beta_n \lambda_n}{2}\left[\sum_{j\in J_k}\frac{R_j^{\alpha_n}}{\lambda_n} 2a (w^\top x_j +b) +\|\theta\|^2\right]} d\theta\\
    =& 
    \int_{(\br \times (B\cap O_k))\cap \set{\|\theta\|\leq L}} \|z\|^2 \exp\curv{-\frac{\beta_n \lambda_n}{2}\left[\sum_{j\in J_k}\frac{R_j^{\alpha_n}}{\lambda_n}2a (w^\top x_j +b) +a^2\right]} \exp\curv{-\frac{\beta_n \lambda_n}{2}(\|z\|^2)} d\theta\\
    \le & 
    \int_\br \int_{B\cap O_k \cap \set{\|z\|\leq L}} \|z\|^2 \exp\curv{-\frac{\beta_n \lambda_n}{2}\left[\sum_{j\in J_k}\frac{R_j^{\alpha_n}}{\lambda_n}2a (w^\top x_j +b) +a^2\right]} \exp\curv{-\frac{\beta_n \lambda_n}{2}(\|z\|^2)} dz da.
\end{align*}
}
Note that for any $u\in \br$ and $\beta,\lambda>0$, 
\begin{align*}
    \int_\br \exp\curv{-\frac{\beta\lambda}{2}\left[ 2a u +a^2 \right]} da 
    &\leq 
    \int_\br \exp\curv{-\frac{\beta\lambda}{2}\left[ a^2 -2|a| |u|\right]}da \\
    &= 
    2\int_0^\infty 
    \exp\curv{-\frac{\beta\lambda}{2}\left[ a^2 -2a |u|\right]} da\\
    &\le 2\sqrt{\frac{2\pi}{\lambda\beta}} \exp(\frac{\lambda\beta}{2}u^2). 
\end{align*}
Thus, by Tonelli's theorem, 
\begin{align*}
    &\int_{(\br \times (B\cap O_k))\cap \set{\|\theta\|\leq L}} \|z\|^2 \rho^\ast_{\alpha_n}(\theta) d\theta \\
    \le & 2\sqrt{\frac{2\pi}{\lambda_n\beta_n}} 
    \int_{B\cap O_k \cap \set{\|z\|\leq L}} \|z\|^2 (Z_{\alpha_n}^\ast)^{-1}
    \exp\curv{-\frac{\beta_n \lambda_n}{2}\Big(\|z\|^2-\Big|\sum_{j\in J_k}\frac{R^{\alpha_n}_j}{\lambda_n}(w^\top x_j+b)\Big|^2\Big)} dz\\
    = & 2\sqrt{\frac{2\pi}{\lambda_n\beta_n}} 
    \int_{B\cap O_k \cap \set{\|z\|\leq L}} \|z\|^2 (Z_{\alpha_n}^\ast)^{-1}
    \exp\curv{-\frac{\beta_n \lambda_n}{2}f_k^{\alpha_n}(z)} dz.
\end{align*}
As shown in~\eqref{eq:uniform}, $f_k^{\alpha_n}$ converges uniformly to $f_k^\ast$ on $\bs^d$.
Thus, we have that for sufficiently large $n$: 
$$
\min_k \inf_{z\in \widehat{\cl B} \cap \hat O_k} f_k^{\alpha_n}(z) \geq \frac{\delta}{2}. 
$$
Thus,
\begin{align*}
    &\int_{(\br \times (B\cap O_k))\cap \set{\|\theta\|\leq L}} \|z\|^2 \rho^\ast_{\alpha_n}(\theta) d\theta \\
    \le& 2(Z_{\alpha_n}^\ast)^{-1}\sqrt{\frac{2\pi}{\lambda_n\beta_n}} 
    \int \|z\|^2 \exp\curv{-\frac{\beta_n\lambda_n\delta\|z\|^2}{4}} dz\\
    \le & C (Z_{\alpha_n}^\ast)^{-1} (\lambda \beta)^{-\frac{d+4}{2}} \delta^{-\frac{d+3}{2}}. 
\end{align*}
for some constant $C>0$ independent of $\alpha_n$. 
With Lemma~\ref{lem:hess-Z}, we have that
$$
\int_{(\br \times (B\cap O_k))\cap \set{\|\theta\|\leq L}} \|z\|^2 \rho^\ast_{\alpha_n}(\theta) d\theta \leq C' \beta_n^{-1}\lambda_n^{-(d+6)/4} (1+\sqrt{\lambda_n})^{(d+2)/2},
$$
for some constant $C'>0$ independent of $\alpha_n$. 
Thus,
$$
\lim_{n\to+\infty} \int_{(\br \times (B\cap O_k))\cap \set{\|\theta\|\leq L}} \|z\|^2 \rho^\ast_{\alpha_n}(\theta) d\theta = 0. 
$$

With this, we can conclude that,
$$
\limsup_{n\to+\infty} \int_{\br \times B} \|z\|^2 \rho^\ast_{\alpha_n}(\theta) d\theta\le \varepsilon.
$$
Since $\varepsilon$ is arbitrary, 
we have that
$$
\lim_{n\to+\infty} \int_{\br \times B} \|z\|^2 \rho^\ast_{\alpha_n}(\theta) d\theta =0.
$$
Since $B$ is open, we have that $\mathbb{I}_{z\in B}(\theta)\|z\|^2$ is lower semicontinuous and bounded from below. 
By Portmanteau theorem, 
$$
0\leq \int_{\br \times B} \|z\|^2 d\rho^\ast \leq \liminf_{n\to+\infty} \int_{\br\times B}\|z\|^2 d\rho^\ast_{\alpha_n}=0. 
$$
Thus, $\int_{\br \times B} \|z\|^2 d\rho^\ast=0$. 
Let 
$$
\Omega^\ast = \set{\gamma (w,b)\colon \gamma \geq0, (w,b)\in \hat \Omega^\ast}.
$$
By Proposition~\ref{prop:shrink}, $\Omega^\ast$ is the union of at most $2\cp(\cx)-1$ rays. 
For any $z\notin \Omega^\ast$, we can find a neighborhood $B$ of $z$ that satisfies~\eqref{eq:B-nbhd} and that $\|z\|>c$ for some $c>0$ and for every $z\in B$. 
Then, 
$$
0=\int_{\br \times B} \|z\|^2 d\rho^\ast \geq c^2 \int_{\br \times B} d\rho^\ast\geq0. 
$$
Thus, $(a,z) \notin \supp(\rho^\ast)$ for any $a\in \br, z\notin \Omega^\ast$. 
Hence, we have that
$$
\supp(\rho^\ast) \subset \br \times \Omega^\ast.
$$
It follows that
$$
\supp(q^\ast ) = \supp(\omega_\#(\rho^\ast) ) \subset \Omega^\ast,
$$
which completes the proof. 
\end{proof}

\section{Proof of Theorem~\ref{thm:nonredun}}
\label{app:nonredun}

In this section, we present the proof of Theorem~\ref{thm:nonredun}. 
We first prove Lemma~\ref{lem:strat}, which characterizes the geometry of the strata.
We then prove Proposition~\ref{prop:strat}, which provides a geometric proof of Theorem~\ref{thm:nonredun}.

\begin{proof}(\textit{Proof of Lemma~\ref{lem:strat}})
    It is clear that $\bs^d=\cup_{V\in \mathcal{V}} V$.
    If $V_A\neq V_{A'}$, then there exists $j\in\set{1,\ldots,M}$ such that $A(x_j)\neq A'(x_j)$. 
    As 
    $V_A \subset \set{(w,b)\in \bs^d\colon \sgn(w^\top x_j+b)=A(x_j)}$ 
    and 
    $V_{A'} \subset \set{(w,b)\in \bs^d\colon \sgn(w^\top x_j+b)={A'}(x_j)}$, 
    we have that $V_A\cap V_{A'}=\varnothing$.

    Consider any $V\neq V_0$ and $V\neq \varnothing$. 
    Let $V=V_A$ for a ternary sign pattern $A$. 
    Then there exists $j$ such that $A(x_j)\neq 0$. 
    Thus, $V\subset  \set{(w,b)\in \bs^d\colon \sgn(w^\top x_j+b)=A(x_j)}$, which is an open hemisphere on $\bs^d$. 
    Then we can select a point $z\in V$ and consider the standard gnomonic projection $\iota$ from $V$ to the tangent space $T_z\bs^d$, viewed as an embedded submanifold of $\br^{d+1}$. 
    Notice that
    $$
    \iota(V) = T_z\bs^d \cap \Big(\bigcap_{j=1}^M \set{(w,b)\in \br^{d+1}\colon \sgn(w^\top x_j+b)=A(x_j)}\Big). 
    $$
    The set on the right-hand side is simply connected as it is the intersection of convex sets. 
    As $\iota$ is a homeomorphism, $V$ is simply connected. 
    This completes the proof. 
\end{proof}

\begin{proof}(\textit{Proof of Proposition~\ref{prop:strat}})
    Let $V_0=\{(w,b)\in\bs^d\colon \sgn(w^\top x_j +b)=0\,\forall j\}$. 
    We first prove item $(i)$ for strata $V\neq V_0$. 
    Select a $k\in\set{1,\ldots,M}$ such that $V\subset \hat O_k$ (and thus $\cl(V)\subset \hat O_k$). 
    Let $T_k=\hat \Omega_k^\ast \cap \hat O_k$. 
    Implied by~\eqref{eq:adjacent}, we have that
    $$
    \hat \Omega^\ast \cap V=\hat\Omega_k^\ast \cap V=T_k\cap V,\quad
    \hat \Omega^\ast \cap \cl(V)=\hat\Omega_k^\ast \cap \cl(V)=T_k\cap \cl(V). 
    $$ 
    As shown in the proof of Proposition~\ref{prop:shrink}, $T_k$ consists of at most two points. 
    If $T_k$ is empty or a singleton, items $(i)$ and $(ii)$ hold trivially. 
    Assume that $T_k$ has two points. 
    Using similar arguments as in the proof of Proposition~\ref{prop:shrink}, 
    $T_k$ is the union of the two singleton sets: 
    \begin{equation}\label{eq:strat}
    \begin{aligned}
    \{z_+\} &= \set{z\in \hat O_k\colon z^\top P_k^\ast\ge 1}= \set{z\in \hat O_k\colon z^\top P_k^\ast=1},    \\
    \{z_-\} & = \set{z\in \hat O_k\colon z^\top P_k^\ast\le -1}=  \set{z\in \hat O_k\colon z^\top P_k^\ast=-1}.
    \end{aligned}
    \end{equation}
    Note that this implies that $v^\top P_k^\ast<1$ for every $z\in \hat O_k\setminus\set{z_+}$ and  $v^\top P_k^\ast>-1$ for every $z\in \hat O_k\setminus\set{z_-}$.

    Assume that $z_+\in V$. We aim to show that $z_- \notin \cl(V)$. 
    Let $A$ be the ternary sign pattern $A\colon \set{x_j}_{j=1}^M\to\set{-1,0,1}$ such that $V=V_A$. 
    Let 
    $$
    e_A=\set{j\colon A(x_j)=0},\quad S_A=\set{(w,b)\in \bs^d\colon w^\top x_j+b=0 \, \forall j\in e_A}. 
    $$
    When $e_A=\varnothing$, let $S_A=\bs^d$. 
    Clearly, $\cl(V)\subset S_A$. 
    Note that $S_A$ is the intersection of $\bs^d$ and the linear space $L_A=\{(w,b)\in\br^{d+1}\colon w^\top x_j +b =0 \, \forall j \in e_A\}$, whose dimension is denoted by $r_A$. 
    Thus, $S_A$ is a sphere of dimension $r_A-1\ge0$, i.e., $S_A\cong \bs^{r_A-1}$. 
    When $r_A=1$, then $S_A$ is a pair of antipodal points. 
    As $V$ is connected by Lemma~\ref{lem:strat}, $V$ must be a singleton and the claim holds trivially.

    Now assume that $r_A>1$. 
    Notice that 
    $$
    V= \set{(w,b)\in S_A\colon \sgn(w^\top x_j +b)=A(x_j), A(x_j)\neq 0}.
    $$
    Thus, $V$ is relatively open in $S_A$. 
    Let $\bar P_k^\ast$ denote the projection of $P_k^\ast$ onto the linear space $L_A$. 
    Define the linear function
    $$
    P\colon S_A \to \br, \quad P(z)=z^\top\bar P_k^\ast. 
    $$
    According to~\eqref{eq:strat}, we have that $P(z_+)=1$ and $P(z)<1$ for every $z\in \cl(V)\setminus\set{z_+}$. 
    Since $V$ is relatively open in $S_A$, $P$ attains a local maximum at $z_+$. 
    It is clear that $P$ has only one local maximum over $S_A$, which is the global maximum. 
    It follows that $z_+= \bar P_k^\ast$. 
    Since $P(z_-)=z_-^\top z_+=-1$, we have that $z_-=-z_+$. 
    On the other hand, select any $j'\notin e_A$ and assume without loss of generality that $A(x_{j'})=1$. 
    Consider the closed hemisphere on $S_A$:
    $$
    S_A^+ = \set{(w,b)\in S_A \colon w^\top x_{j'}+b\geq 0}. 
    $$
    Clearly, $\cl(V)\subset S_A^+$. 
    Note that $z^+\in V$ and thus lies in the relative interior of $S_A^+$. 
    By the reflection symmetry, we have that $z^-$ lies in $S_A \setminus S_A^+$,
    which means that $z_- \notin \cl(V)$. 
    Therefore, we conclude that $V\cap \hat \Omega^\ast$ consists of at most one point and $V \cap \hat \Omega^\ast\neq\varnothing$ implies that 
    $(\partial_{\cl(V)}V)\cap \hat \Omega^\ast=\varnothing$.

    Now we prove item $(ii)$, and prove item $(i)$ for $V_0$.
    Notice that $\cl(V_{-1})$ is a spherical sector. 
    Let $k\in \set{1,\cdots,2\cp(\cx)}$ be the index such that $\cl(V_{-1})=\hat O_k$.
    It is clear that $J_k=\varnothing$. 
    Then, $f_k^\ast(z)=\|z\|^2$ and $\hat \Omega_k^\ast=\varnothing$. 
    Thus, $\hat \Omega^\ast \cap \cl(V_{-1})=\varnothing$. 
    In particular, as $\cl(V_0) \subset \cl(V_{-1})$, we have that $\cl(V_0) \cap \hat \Omega^\ast=\varnothing$.

    Item $(iii)$ is a direct consequence of the proof of Proposition~\ref{prop:shrink}, in particular the statement that, for every $k$, $\hat O_k \cap \hat \Omega_k^\ast$ has at most two connected components and the case of two components occurs if and only if both components of $\hat \Omega_k^\ast$ intersect the relative boundary of $\hat O_k$ in $\bs^d$. 
    This completes the proof. 
\end{proof}

\begin{proof}(\textit{Proof of Theorem~\ref{thm:nonredun}})
    Notice that all the realizable linear ternary activation patterns on $\cx$ are in one-to-one correspondence with the strata in $\mathcal{V}$ by Lemma~\ref{lem:strat}. 
    Recall from Theorem~\ref{thm:pwl-detail} that $\{(w_k,b_k)\}_{k=1}^K=\hat \Omega^\ast$. 
    Then, we have that items $(i),(ii),(iii)$ of Proposition~\ref{prop:strat} imply items $(i),(ii),(iii)$ of Theorem~\ref{thm:nonredun}, respectively.
    This completes the proof. 
\end{proof}

\section{Proof of Proposition~\ref{prop:toy}}
\label{app:toy}

In this section, we prove Proposition~\ref{prop:toy}, which identifies the kink hyperplanes for a simple class of data sets.

\begin{proof}\textit{(Proof of Proposition~\ref{prop:toy})}
    Since the data set is symmetric, we can assume without loss of generality that $x_1>0$. 
    Recall that for any cone $S\subset \br^{d+1}$, we use $\hat S$ to denote its spherical cross-section, i.e., $\hat S=S\cap \bs^d$. 
    For any $z\in \br^{d+1}\setminus\{ 0\}$, let $\ell_z$ denote the straight line passing through $z$: 
    $$
    \ell_z=\set{\gamma z\colon \gamma \in \br}.
    $$

    \textbf{Even labels $y_1=y_2$:} 
    We first show that $h^\ast$ is an even function. 
    Consider $T(a,w,b)=(a,-w,b)$. 
    Notice that
    $$
    \sigma^m(x,T(\theta))= a^m\Big( x^\top (-w)^m +b \Big)_+^m = a^m (-x^\top w^m +b)_+^m = \sigma^m(-x,\theta).
    $$
    For any $\rho\in \ch$, let $T_\#\rho$ denote the push-forward measure under $T$.  
    Clearly, $T_\# \rho \in \ch$. 
    We have that
    $$
    h_{\sigma^m}(x, T_\# \rho) = \int \sigma^m(x,T(\theta)) d\rho(\theta) 
    = \int \sigma^m(-x,\theta) d\rho(\theta) = h_{\sigma^m}(-x,\rho). 
    $$
    Let $R$ be the empirical risk 
    $$
    R(\rho)=\frac{1}{4}\Big( (h_{\sigma^m}(x_1,\rho)-y_1)^2+(h_{\sigma^m}(x_2,\rho)-y_2)^2  \Big).
    $$
    By above, $R(\rho)=R(T_\# \rho)$ for any $\rho\in \ch$. 
    Notice that $M(\rho)=M(T_\# \rho)$ as $T$ preserves the norm, and that $H(\rho)=H(T_\# \rho)$ as $T$ is a diffeomorphism that preserves the volume. 
    Therefore, for any $\bl,m>0$, we have that
    $$
    F^{\bl}_{\sigma^m}(\rho) = F^{\bl}_{\sigma^m}(T_\#\rho). 
    $$
    This implies that 
    $\rho^\ast_\alpha=T_\# \rho^\ast_\alpha$ a.e., since $\rho^\ast_\alpha$ is the unique minimizer. 
    Therefore, under the sequence $\alpha_n=(\beta_n,\lambda_n,m_n)$ in Theorem~\ref{thm:pwl}, 
    $$
    h^\ast_{\alpha_n}(x)= h_{\sigma^{m_n}}(x, \rho^\ast_{\alpha_n})=
    h_{\sigma^{m_n}}(x, T_\# \rho^\ast_{\alpha_n})
    =
    h_{\sigma^{m_n}}(-x, \rho^\ast_{\alpha_n})=h^\ast_{\alpha_n}(-x),\quad \forall n.
    $$
    Taking $n\to+\infty$ yields that
    $$
    h^\ast(x)=h^\ast(-x). 
    $$
    Thus, $h^\ast$ is an even function. 

    Let $J_1=\set{1,2},J_2=\set{1},J_3=\set{2},J_4=\varnothing$, and 
    $\kappa=(h^\ast(x_1)-y_1)/(2\bar \lambda)$. 
    We have that
    $$
    \begin{dcases}
        P_1^\ast = \kappa \begin{pmatrix}
            x_1\\1
        \end{pmatrix}+  \kappa\begin{pmatrix}
            x_2\\1
        \end{pmatrix} =
        \kappa \begin{pmatrix}
            \mathbf{0}\\2
        \end{pmatrix},\\
        P_2^\ast = \kappa\begin{pmatrix}
            x_1\\1
        \end{pmatrix},\\
        P_3^\ast = \kappa\begin{pmatrix}
            -x_1\\1
        \end{pmatrix},\\
        P_4^\ast = 0. 
    \end{dcases}
    $$
    For $k=1,\ldots,4$, let $\Omega_k^\ast$ denote the non-positive set of $f_k^\ast$.
    Note that $\Omega_k^\ast$ is either the singleton $\set{0}$ or a double cone with axis $\ell_{P_k^\ast}$. 
    Let
    $$
    \begin{dcases}
        O_1 = \set{(w,b)\in \br^{d+1}\colon w^\top x_1 +b \geq0, -w^\top x_1 +b \geq0},\\
        O_2 = \set{(w,b)\in \br^{d+1}\colon w^\top x_1 +b \geq0, -w^\top x_1 +b \leq0},\\
        O_3 = \set{(w,b)\in \br^{d+1}\colon w^\top x_1 +b \leq0, -w^\top x_1 +b \geq0},\\
        O_4 = \set{(w,b)\in \br^{d+1}\colon w^\top x_1 +b \leq0, -w^\top x_1 +b \leq0}. 
    \end{dcases}
    $$
    Clearly, $f_4^\ast(z)=\|z\|^2$ and $\Omega_4^\ast=\set{0}$.
    We now examine $\Omega_k^\ast \cap O_k$ for $k=1,2,3$. 
    If $\kappa=0$, we have that $P_k^\ast=0$ for every $k$. 
    Then, $f_k^\ast(z)=\|z\|^2$ and $\Omega^\ast_k=\set{0}$ for every $k$. 
    Therefore, $\Omega^\ast=\set{0}$. 
    Clearly, in this case $\hat \Omega^\ast=\varnothing$ and, 
    by Theorem~\ref{thm:pwl-detail}, $h^\ast$ is globally affine. 
    
    Now consider $\kappa\neq0$. 
    Let $O_{jk}=O_j\cap O_k $ for any $i\neq j$. 
    Notice that
    $$
    \begin{dcases}
    \partial O_1 = O_{12} \cup O_{13},\\
    \partial O_2 = O_{12} \cup O_{24},\\
    \partial O_3 = O_{13} \cup O_{34}. 
    \end{dcases}
    $$
    For $k=1$, notice that
    $$
    (P_1^\ast)^\top  (x_1,1) = 2\kappa ,\quad (P_1^\ast)^\top  (-x_1,1) = 2\kappa. 
    $$
    This means that $\ell_{P_1^\ast}$ intersects the interior of $O_1$. 
    With Proposition~\ref{prop:shrink}, we have that 
    \begin{equation}\label{eq:toy0}
        \Omega_1^\ast \cap  O_1\subset \ell_{P_1^\ast},\quad 1\geq \|P_1^\ast\|=|\kappa|\sqrt{2}.
    \end{equation}
    For $k=2$, notice that 
    $$
    (P_2^\ast)^\top  (x_1,1) = \kappa(\|x_1\|^2+1) ,\quad (P_2^\ast)^\top  (-x_1,1) = \kappa(-\|x_1\|^2+1). 
    $$
    If $\|x_1\|<1$, we have that $\ell_{P_2^\ast}$ intersects $O_1$ and $O_4$. 
    Note that
    $$
    \sin (d_\angle(\ell_{P_2^\ast}, O_{12}))=\frac{\frac{(P_2^\ast)^\top P_3^\ast}{\|P_3^\ast\|}}{\|P_2^\ast\|}=\frac{1-\|x_1\|^2}{\|x_1\|^2+1}.
    $$
    Thus,
    $$
    d_\angle(\ell_{P_2^\ast}, O_{12})= \arcsin(\frac{1-\|x_1\|^2}{\|x_1\|^2+1}) < \pi/2 = d_\angle(\ell_{P_2^\ast}, O_{24}).
    $$
    By the double-cone geometry of $\Omega_2^\ast$, 
    we have that
    $\Omega_2^\ast \cap  O_2 \neq \varnothing$ if and only if 
    $$
    \arcsin(\frac{1-\|x_1\|^2}{\|x_1\|^2+1}) = \arccos(\frac{1}{\|P_2^\ast\|})=\arccos(\frac{1}{|\kappa|\sqrt{\|x_1\|^2+1}}).
    $$
    As $\|x_1\|<1$, this leads to that
    $$
    |\kappa|=\sqrt{\frac{\|x_1\|^2+1}{4\|x_1\|^2}}>\frac{1}{\sqrt{2}},
    $$
    which is impossible by~\eqref{eq:toy0}. 
    Thus, $\Omega_2^\ast \cap  O_2 = \varnothing$.
    Similarly, we have that $\Omega^\ast_3 \cap  O_3=\varnothing$. 
    Therefore, $  \Omega^\ast \subset \ell_{P_1^\ast}$. 
    Notice that $\ell_{P_1^\ast}$ corresponds to the empty set $\set{x\colon \mathbf{0}^\top x+1=0}=\varnothing$. 
    Then with Theorem~\ref{thm:pwl-detail}, we have that $h^\ast$ is globally affine.

    If $\|x_1\|\ge 1$, we have that $\ell_{P_2^\ast}$ intersects $O_2$. 
    It follows that $\Omega_2^\ast \cap    O_2 \subset \ell_{P_2^\ast}$. 
    Similarly, we have that 
    $ \Omega_3^\ast \cap    O_3 \subset \ell_{P_3^\ast}$. 
    Thus, we have that
    $$
      \Omega^\ast \subset \ell_{P_1^\ast}\cup\ell_{P_2^\ast}\cup \ell_{P_3^\ast}.
    $$
    Notice that $\ell_{P_2^\ast}$ corresponds to the hyperplane $\set{x\colon x_1^\top x +1=0}$ and $\ell_{P_3^\ast}$ corresponds to 
    $\set{x\colon -x_1^\top x +1=0}$.

    \textbf{Odd labels $y_1=-y_2$:} 
    With similar arguments, we have that $h^\ast$ is an odd function. 
    Then we have that
    $$
    \begin{dcases}
        P_1^\ast = \kappa \begin{pmatrix}
            x_1\\1
        \end{pmatrix} -\kappa \begin{pmatrix}
            x_2\\1
        \end{pmatrix} =
        \kappa \begin{pmatrix}
            2x_1\\0
        \end{pmatrix},\\
        P_2^\ast = \kappa\begin{pmatrix}
            x_1\\1
        \end{pmatrix},\\
        P_3^\ast = \kappa\begin{pmatrix}
            x_1\\-1
        \end{pmatrix},\\
        P_4^\ast = 0. 
    \end{dcases}
    $$
    Similarly, 
    it suffices to examine $\Omega_k^\ast \cap O_k$ for $k=1,2,3$, with $\kappa \neq 0$. 
    Define:
    $$
    Q_1 = ( -x_1,\|x_1\|^2 ),\quad 
    Q_2 = ( x_1,\|x_1\|^2 ).
    $$
    Recall that $G_x = \set{(w,b)\in \br^{d+1}\colon w^\top x+b=0}$. 
    Note that 
    \begin{align*}
    &\ell_{Q_1}=\operatorname{Proj}_{G_{x_1}}(\ell_{P_1^\ast})=\operatorname{Proj}_{G_{x_1}}(\ell_{P_3^\ast})\\
    &\ell_{Q_2}=\operatorname{Proj}_{G_{x_2}}(\ell_{P_1^\ast})=\operatorname{Proj}_{G_{x_2}}(\ell_{P_2^\ast}),
    \end{align*}
    where $\operatorname{Proj}_A(B)$ denotes the orthogonal projection of $B$ onto $A$.

    For $k=1$, we have that
    $$
    (P_1^\ast)^\top (x_1,1)=2\kappa \|x_1\|^2,\quad 
    (P_1^\ast)^\top (-x_1,1)=-2\kappa \|x_1\|^2.
    $$
    Thus, $\ell_{P_1^\ast}$ intersects $O_2$ and $O_3$.
    Notice that 
    $$
    \sin(d_\angle(\ell_{P_1^\ast},O_{12}))=\frac{\frac{(P_1^\ast)^\top P_3^\ast}{\|P_3^\ast\|}}{\|P_1^\ast\|}=\frac{\frac{(P_1^\ast)^\top P_2^\ast}{\|P_2^\ast\|}}{\|P_1^\ast\|}=\sin(d_\angle(\ell_{P_1^\ast},O_{13})). 
    $$
    Notice also that $O_{13}\subset G_1$ and $O_{12}\subset G_2$. 
    Thus, by the double-cone geometry of $\Omega_1^\ast$ and since $\Omega_1^\ast \cap O_1$ consists of at most two rays, we have that
    \begin{equation*}
        O_1 \cap \Omega_1^\ast 
        \subset 
        \operatorname{Proj}_{G_1}(\ell_{P_1^\ast}) \cup 
        \operatorname{Proj}_{G_2}(\ell_{P_1^\ast})=
        \ell_{Q_1}\cup \ell_{Q_2}.         
    \end{equation*}

    For $k=2$, we have that
    $$
    (P_2^\ast)^\top (x_1,1)= \kappa (\|x_1\|^2+1),\quad 
    (P_2^\ast)^\top (-x_1,1)= \kappa (-\|x_1\|^2+1). 
    $$
    When $\|x_1\|\leq 1$, we have that $\ell_{P_2^\ast}$ intersects $O_1$ and $O_4$.
    Notice that
    $$
    d_\angle(\ell_{P_2^\ast}, O_{12})= \arcsin(\frac{1-\|x_1\|^2}{\|x_1\|^2+1}) < \pi/2 = d_\angle(\ell_{P_2^\ast}, O_{24}).
    $$
    Since $O_{12}\subset G_2$, we have that
    $$
    O_2 \cap   \Omega_2^\ast  \subset \operatorname{Proj}_{G_2}(\ell_{P_2^\ast})  = \ell_{Q_2}. 
    $$
    Similarly, we have that
    $$
    O_3 \cap \Omega_3^\ast \subset \ell_{Q_1}. 
    $$
    Thus, 
    $$
    \Omega^\ast \subset \ell_{Q_1} \cup \ell_{Q_2}. 
    $$
    Note that $\ell_{Q_1}$ corresponds to the hyperplane $\set{x\colon -x_1^\top x+\|x_1\|^2=0}$ and $\ell_{Q_2}$ corresponds to
    $\set{x\colon x_1^\top x+\|x_1\|^2=0}$.

    When $\|x_1\|> 1$, we have that
    $\ell_{P_2^\ast}$ intersects the interior of $O_2$. 
    Then we have that 
    \begin{equation}\label{eq:toy2}
    O_2\cap \Omega_2^\ast \subset \ell_{P_2^\ast},\quad 1\geq \|P_2^\ast\|=|\kappa|\sqrt{\|x_1\|^2+1}. 
    \end{equation}
    Note that $O_1\cap \Omega_1^\ast\neq \varnothing$ if and only if
    $$
    d_\angle(\ell_{P_1^\ast}, O_{12})=\arcsin(\frac{\|x_1\|}{\sqrt{\|x_1\|^2+1}}) = 
    \arccos(\frac{1}{2|\kappa|\|x_1\|})
    =\arccos(\frac{1}{\|P_1^\ast\|}). 
    $$
    Equivalently,
    $$
    |\kappa |=\sqrt{\frac{1+\|x_1\|^2}{4\|x_1\|^2}}.
    $$
    This, together with~\eqref{eq:toy2}, yields that
    $$
    1\geq \frac{\|x_1\|^2+1}{2\|x_1\|},
    $$
    which is impossible, as $\|x_1\|>1$. 
    Therefore, we have that $ O_1 \cap \Omega_1^\ast= \varnothing$. 
    Similar to $k=2$, we have that
    $$ 
    O_3\cap \Omega_3^\ast \subset \ell_{P_3^\ast}. 
    $$ 
    Thus, 
    $$
    \Omega^\ast \subset  \ell_{P_2^\ast} \cup \ell_{P_3^\ast}.
    $$
    Note that $\ell_{P_2^\ast}$ corresponds to the hyperplane $\set{x\colon x_1^\top x +1=0}$ and $\ell_{P_3^\ast}$ corresponds to $\set{x\colon -x_1^\top x +1=0}$. 
    This completes the proof. 
\end{proof}

\section{Comparison of the Bounds on Effective Width}
\label{app:bound}

In this section, we compare our bound on the effective width with the bound obtained in \citet{del2026dual}:
$$
K' \leq K'_{\mathrm{bound}}= \max_{k=0,\ldots,M} \binom{M}{k}\Xi(M-k,d+1-k),  
$$
where $\Xi(n,d)$ is the number of regions in $\br^d$ divided by an arrangement of $n$ hyperplanes that pass through the origin. 
Note that \citet{del2026dual} considered networks without input biases and replace each training input $x_j$ by $x_j' =(x_j,1)\in \br^{d'}=\br^{d+1}$.
Accordingly, the original bound in \citet{del2026dual} is 
$$
\max_{k=0,\ldots,M} \binom{M}{k}\Xi(M-k,d'-k). 
$$

For the training set $\cx=\set{x_j}_{j=1}^M\subset \br^{d}$, 
let $G_{x_j}=\set{(w,b)\colon w^\top x_j+b=0}$. 
Note that in the analysis of \citet{del2026dual}, 
$\Xi(M-k,d+1-k)$ is essentially counting the number of regions of 
$\cap_{j\in \mathcal{J}_k} G_{x_j}$ divided by $\set{G_{x_j}}_{j\notin \mathcal{J}_k}$ for some $\mathcal{J}_k \subset \set{1,\ldots,M}$ with $|\mathcal{J}_k|=k$ (see the proof of Lemma 3.11 in \citealt{del2026dual}). 
Below we consider the general case, i.e., $\cx$ are in general position.
In this case, \citet{cover1965geometrical} gives that
$$
2\cp(\cx)-1=2\sum_{i=0}^{d}\binom{M-1}{i} -1,
$$
and that
$$
K'_{\mathrm{bound}}=\max_{k=0,\ldots,M} \binom{M}{k} \cdot 2\sum_{i=0}^{d-k}\binom{M-k-1}{i}.
$$
Note that,
$$
K'_{\mathrm{bound}}>2\sum_{i=0}^{d}\binom{M-1}{i}.
$$
Thus, $K'_{\mathrm{bound}}$ exceeds our bound by at least one.

Now assume that $2d +1 \leq M$. 
Note that 
$$
\frac{\binom{M-1}{i+1}}{\binom{M-1}{i}}=\frac{M-1-i}{i+1}\geq 1 \Longleftrightarrow 2i\leq M-2,
$$
which is satisfied when $i\leq d-1$. 
Thus, 
$$
2\cp(\cx) -1 <  2\sum_{i=0}^d
\binom{M-1}{i}
\leq 2(d+1)
\binom{M-1}{d}.
$$
Meanwhile, 
\begin{align*}
\max_{k=0,\ldots,M} \binom{M}{k} \cdot 2\sum_{i=0}^{d-k}\binom{M-k-1}{i}
&\geq 2\binom{M}{\lfloor d/2\rfloor}\binom{M-\lfloor d/2\rfloor-1}{d-\lfloor d/2\rfloor}\\
&\geq 2 \frac{M}{M-\lfloor d/2\rfloor} \binom{d}{\lfloor d/2\rfloor}\binom{M-1}{d}\\
&\geq 2  \binom{d}{\lfloor d/2\rfloor}\binom{M-1}{d}\\
&\geq 2  \frac{2^d}{d+1}\binom{M-1}{d}.
\end{align*}
Therefore, 
$$
\frac{K'_{\mathrm{bound}}}{2\cp(\cx) -1} \geq \frac{2^d}{(d+1)^2}. 
$$
That is, in the regime of $2d +1 \leq M$, our bound improves over the bound in~\citet{del2026dual} by a factor of at least $2^d/(d+1)^2$.

\section{Existence of Accumulation Points} 
\label{app:accumu}

In this section, we show that for any sequence $(\beta_n,\lambda_n,m_n)$ satisfying Assumption~\ref{assum}, with $m_n,\beta_n\to+\infty$ and $\lambda_n\to \bar \lambda$, one can find a subsequence $(\beta_{n_k},\lambda_{n_k},m_{n_k})$ along which $h^\ast_{\beta_{n_k},\lambda_{n_k},m_{n_k}}(\cdot)$ converges pointwise.

\begin{proposition}\label{prop:accumu}
    Under the setting of Theorem~\ref{thm:pwl}, 
    for any sequence $(\beta_n,\lambda_n,m_n)$ satisfying Assumption~\ref{assum} for $n\geq 1$, with $m_n,\beta_n\to+\infty$ and $\lambda_n\to \bar \lambda$ for an arbitrary $\bar \lambda>0$ as $n\to+\infty$, 
    there exists a subsequence $(\beta_{n_k}, \lambda_{n_k},m_{n_k})$ such that
    $\lim_{k\to+\infty}h^\ast_{\beta_{n_k}, \lambda_{n_k},m_{n_k}}(x)$ exists for every $x\in \br^d$. 
\end{proposition}

\begin{proof}
    Notice that
    \begin{align*}
        |h^\ast_{\bl,m}(0)| = \left| \int a^m(b)_+^m  d \rho^\ast_{\bl,m}(\theta)\right|
        \leq \int |ab| d\rho^\ast_{\bl,m}(\theta)
        \leq \int \|\theta\|^2 d\rho^\ast_{\bl,m}(\theta). 
    \end{align*}
    By Lemma~\ref{lem:Mbound}, we have that
    $$
    \sup_n|h^\ast_{\beta_n,\lambda_n,m_n}(0)|\leq C
    $$
    for a constant $C>0$. 
    On the other hand, as we shown in the proof of Theorem~\ref{thm:pwl-detail}, 
    we have that, for any $x,x'\in \br^{d}$,
    \begin{align*}
        &\sup_n|h^\ast_{\beta_n,\lambda_n,m_n}(x)-h^\ast_{\beta_n,\lambda_n,m_n}(x')|\\
        \leq & \Big(2(d+2)+4(\lambda_n \beta_n)^{-1} (1+(d+2)\log(8\pi))+ \frac{2}{\lambda_n M}\sum_{j=1}^M y_j^2 \Big)\|x-x'\|\\
        \leq& C' \|x-x'\|,
    \end{align*}
    for a constant $C'>0$. 
    Thus, for any $x\in \br^d$,
    $$
    \sup_n |h^\ast_{\beta_n,\lambda_n,m_n}(x)|\leq \sup_n|h^\ast_{\beta_n,\lambda_n,m_n}(0)| + \sup_n|h^\ast_{\beta_n,\lambda_n,m_n}(x)-h^\ast_{\beta_n,\lambda_n,m_n}(0)|\leq C'',
    $$
    for some $C''>0$.
    Thus, there exists a subsequence such that 
    $$
    \lim_{k\to+\infty}h^\ast_{\beta_{n_k},\lambda_{n_k},m_{n_k}}(x)
    $$
    exists.

    Now take a countable dense subset $\set{x_i}_{i\geq1}\subset \br^d$. 
    By a diagonal argument, we can find a subsequence $\set{n_k}$ such that
    $$
    \lim_{k\to+\infty}h^\ast_{\beta_{n_k},\lambda_{n_k},m_{n_k}}(x_i)
    $$
    exists for every $i$. 
    To simplify the notation, let $h_k=h^\ast_{\beta_{n_k},\lambda_{n_k},m_{n_k}}$. 
    Now for any $x\in \br^d$ and any $\varepsilon>0$, note that
    \begin{align*}
        |h_k(x) - h_{k'}(x)| &\leq |h_k(x) - h_{k}(x_i)|+ |h_k(x_i) - h_{k'}(x_i)|+|h_{k'}(x) - h_{k'}(x_i)|\\
        &\leq 2C'\|x-x_i\|+|h_k(x_i) - h_{k'}(x_i)|. 
    \end{align*}
    Since $\set{x_i}$ is dense in $\br^d$, we can select $i$ such that 
    $$
    2C'\|x-x_i\|\leq \varepsilon/2.
    $$
    Since $\set{h_k(x_i)}$ is convergent, there exists $K$ such that for any $k,k'>K$, 
    $$
    |h_k(x_i) - h_{k'}(x_i)|\leq \varepsilon/2. 
    $$
    Therefore, for $k,k'>K$, 
    $$
    |h_k(x) - h_{k'}(x)|\le \varepsilon.
    $$
    Thus, the limit $\lim_{k\to+\infty}h_k(x)$ exists. 
    This means that $h^\ast_{\beta_{n_k},\lambda_{n_k},m_{n_k}}$ has a pointwise limit as $k\to+\infty$. 
    This completes the proof. 
\end{proof}

\section{Non-redundancy of Kink Hyperplane Arrangement}
\label{app:kink}

In this section, we obtain a description of the location of the kink hyperplanes of the limiting piecewise affine function using Theorem~\ref{thm:nonredun}.

We characterize the location of a hyperplane in the input space through the partition it induces on the training set $\cx=\set{x_j}_{j=1}^M$. 
As defined in Section~\ref{sec:pwl}, a dichotomy is a partition of the training set induced by a hyperplane that does not pass through any data point. 
To accommodate hyperplanes that intersect the training set, we introduce the following notion. 
For a training set $\cx=\set{x_j}_{j=1}^M$, we say $\mathcal{T}=(\set{\cx_+,\cx_-},\cx_0)$ is a linear \emph{trichotomy} of $\cx$ if
$\cx=\cx_+\sqcup \cx_-\sqcup \cx_0$ and there exist  $w\in\br^d,b\in \br$ satisfying 
$w^\top x_j +b>0$ if $x_j \in \cx_+$,
$w^\top x_j +b<0$ if $x_j \in \cx_-$, and 
$w^\top x_j +b=0$ if $x_j \in \cx_0$.
In this case, we say that $\mathcal{T}$ is realized by the hyperplane $\{x\colon w^\top x+b=0\}$. 
Note that a dichotomy is a special trichotomy with $\mathcal{X}_0 = \varnothing$. 
We call a trichotomy a strict trichotomy if it is not a dichotomy. 
We introduce a partial order to compare trichotomies. 
For two trichotomies of $\cx$, $\mathcal{T}=(\set{\cx_+,\cx_-},\cx_0)$ and $\mathcal{T}'=(\set{\cx_+',\cx_-'},\cx_0')$, 
define $\mathcal{T}'\preceq \mathcal{T}$ if, after possibly interchanging 
$\cx_+$ and $\cx_-$, 
\[
\cx_+\subseteq \cx_+',\quad \cx_-\subseteq \cx_-'.
\]
When $\mathcal{T}'\preceq \mathcal{T}$, we say that $\mathcal{T}$ is a degeneration of $\ct'$.

Recall from Section~\ref{sec:location} that, each stratum $V\in \mathcal{V}$ uniquely corresponds to a realizable ternary sign pattern on the training inputs $\cx$. 
Note that each realizable trichotomy of $\mathcal{X}$ corresponds to a pair of ternary sign patterns related by a sign flip, and hence to a pair of antipodal strata $\{V,-V\}$, except that the trichotomy $(\{\varnothing,\varnothing\},\mathcal{X})$ corresponds to the single stratum $V_0$.
Under this correspondence, the following result is a direct consequence of Theorem~\ref{thm:nonredun}.

\begin{corollary}[Kink Hyperplane Arrangement]\label{coro:kink} 
    Under the same notations and assumptions as in Theorem~\ref{thm:feature},
    assume that $(w_k,b_k)\neq (w_{k'},b_{k'})$ whenever $k\neq k'$.
    Let $\Pi$ denote the collection of hyperplanes $\pi_k=\set{x\colon x^\top w_k+b_k=0}$ for $k=1,\ldots,K$. 
    For any dichotomy $\ct$ of $\cx$, let $\cd(\ct)$ denote the set of all trichotomies that are degenerations of $\ct$ and are realized by hyperplanes in $\Pi$. 
    The following holds:
    \begin{enumerate}

        \item[(i)]
        Any trichotomy of the training set $\cx$ is realized by most two hyperplanes in $\Pi$. 

        \item[(ii)] Consider the dichotomy $\ct_0=(\set{\cx, \varnothing},\varnothing)$. 
        Then $\ct_0$ is realized by at most one hyperplane in $\Pi$. 
        Moreover, $\cd(\ct_0)$ contains at most two elements.
        When $\cd(\ct_0)$ contains two elements, both elements are strict trichotomies.

        \item[(iii)] 
        Given an arbitrary dichotomy $\ct$ of $\cx$,  $\cd(\ct)$ has at most four elements. 
        At least $2|\cd(\ct)|-4$ of the elements in $\cd(\ct)$ are strict trichotomies, where $|\cd(\ct)|$ denotes the cardinality of $\cd(\ct)$. 
    \end{enumerate}
\end{corollary}

\citet{shevchenko2022mean} showed that the solution function has at most three knots between any two consecutive training points and, when three knots
occur, two must coincide with the endpoints of the interval, while the third lies strictly in its interior. 
We explain that this is covered as a special case of Corollary~\ref{coro:kink}. 
Consider a set of univariate training inputs $\cx=\set{x_j}_{j=1}^M$ and assume that $-L=x_0<x_1<\cdots<x_M<x_{M+1}=L$ for some sufficiently large $L>0$. 
Fix a $j\in \set{1,\ldots, M-1}$. 
Let $\mathcal{T}$ denote the dichotomy $\cx=\set{x_{j'}}_{j'\leq j}\sqcup \set{x_{j'}}_{j'>j}$. 
Note that $|\cd(\ct)|$ precisely counts the knots in the closed interval $[x_{j},x_{j+1}]$. 
By Corollary~\ref{coro:kink}, 
if $|\cd(\ct)|=3$, then at least two of the corresponding hyperplanes realize strict trichotomies. In the univariate setting, this means that two of the associated knots must lie exactly at $x_{j}$ and $x_{j+1}$, while the remaining knot must lie in $(x_{j},x_{j+1})$. 
On the other hand, if $|\cd(\ct)|=4$, then Corollary~\ref{coro:kink} implies that all four corresponding hyperplanes must realize strict trichotomies. This is impossible on the interval $[x_{j},x_{j+1}]$, since only the end points correspond to strict trichotomies. 
Regarding the dichotomy $\ct_0=(\set{\cx, \varnothing},\varnothing)$, 
note that $\mathcal{D}(\ct_0)$ counts the knots in the intervals $[x_0, x_1]$ and $[x_M,x_{M+1}]$. 
By Corollary~\ref{coro:kink}, there are at most two knots on these two closed intervals.
Moreover, if two knots are present, they must occur precisely at the boundary data points $x_1$ and $x_M$.

\newpage 
\vskip 0.2in
\bibliography{refs}

\end{document}